\documentclass[]{article}

\usepackage{microtype}
\usepackage{graphicx}
\usepackage{subfigure}
\usepackage{booktabs} 

\usepackage{hyperref}
\usepackage{multirow}

\usepackage{algorithmic}

\usepackage[accepted]{icml2025}

\usepackage{amsmath}
\usepackage{amssymb}
\usepackage{mathtools}
\usepackage{amsthm}
\usepackage{adjustbox}

\usepackage{thm-restate}
\usepackage{tikz}
\usepackage{pgfplots}
\usetikzlibrary{matrix,fit,calc,arrows.meta,decorations.pathreplacing, pgfplots.groupplots}
\tikzset{
    axis break gap/.initial=0mm
}
\definecolor{gr}{RGB}{60,180,100}
\definecolor{bl}{RGB}{70,70,240}
\definecolor{sky}{RGB}{100,180,240}
\definecolor{yl}{RGB}{250,170,30}
\definecolor{or}{RGB}{200,140,80}
\definecolor{pp}{RGB}{200,150,240}
\definecolor{darkred}{RGB}{200,30,0}

\newcommand{\overbar}[1]{\mkern 1.5mu\overline{\mkern-1.5mu#1\mkern-1.5mu}\mkern 1.5mu}
\newcommand{\bX}{\mathbf{X}}
\newcommand{\bZ}{\mathbf{Z}}
\newcommand{\bz}{\mathbf{z}}
\newcommand{\bw}{\mathbf{w}}
\newcommand{\bW}{\mathbf{W}}
\newcommand{\bF}{\mathbf{F}}
\newcommand{\bH}{\mathbf{H}}
\newcommand{\bL}{\mathbf{L}}
\newcommand{\bP}{\mathbf{P}}
\newcommand{\bc}{\mathbf{c}}

\newcommand{\bU}{\mathbf{U}}
\newcommand{\bB}{\mathbf{B}}

\newcommand{\hbw}{\widehat{\mathbf{w}}}
\newcommand{\hbW}{\widehat{\mathbf{W}}}

\newcommand{\hbZ}{\widehat{\mathbf{Z}}}
\newcommand{\diag}{\mathrm{diag}}
\newcommand{\Diag}{\mathrm{Diag}}

\DeclareMathOperator{\vect}{vec}
\newcommand{\round}{\mathrm{Round}}

\usepackage[capitalize,noabbrev]{cleveref}

\theoremstyle{plain}

\theoremstyle{definition}

\theoremstyle{remark}

\usepackage[textsize=tiny]{todonotes}

\icmltitlerunning{GuidedQuant: Large Language Model Quantization via Exploiting End Loss Guidance}

\pgfplotsset{compat=1.16}
\begin{document}

\twocolumn[
\icmltitle{GuidedQuant: Large Language Model Quantization via \\Exploiting End Loss Guidance}



\icmlsetsymbol{equal}{*}
\icmlsetsymbol{intern}{$\dagger$}

\begin{icmlauthorlist}
\icmlauthor{Jinuk Kim}{snu,nprc,intern}
\icmlauthor{Marwa El Halabi}{sail}
\icmlauthor{Wonpyo Park}{google}
\icmlauthor{Clemens JS Schaefer}{google}
\icmlauthor{Deokjae Lee}{snu,nprc}
\icmlauthor{Yeonhong Park}{snu}
\icmlauthor{Jae W. Lee}{snu}
\icmlauthor{Hyun Oh Song}{snu,nprc}
\end{icmlauthorlist}

\icmlaffiliation{snu}{Department of Computer Science and Engineering, Seoul National University}
\icmlaffiliation{nprc}{Neural Processing Research Center}
\icmlaffiliation{sail}{Samsung AI Lab, Montreal}
\icmlaffiliation{google}{Google}

\icmlcorrespondingauthor{Hyun Oh Song}{hyunoh@snu.ac.kr}

\icmlkeywords{Machine Learning, ICML}

\vskip 0.3in
]



\printAffiliationsAndNotice{$^\dagger$Work partly done during an internship at Google.}  

\begin{abstract}
Post-training quantization is a key technique for reducing the memory and inference latency of large language models by quantizing weights and activations without requiring retraining.
However, existing methods either (1) fail to account for the varying importance of hidden features to the end loss or, when incorporating end loss, (2) neglect the critical interactions between model weights.
To address these limitations, we propose GuidedQuant, a novel quantization approach that integrates gradient information from the end loss into the quantization objective while preserving cross-weight dependencies within output channels. 
GuidedQuant consistently boosts the performance of state-of-the-art quantization methods across weight-only scalar, weight-only vector, and weight-and-activation quantization.
Additionally, we introduce a novel non-uniform scalar quantization algorithm, which is guaranteed to monotonically decrease the quantization objective value, and outperforms existing methods in this category.
We release the code at \url{https://github.com/snu-mllab/GuidedQuant}.
\end{abstract}

\section{Introduction}
\label{introduction}

\begin{table}[t]
\caption{Summary of results of GuidedQuant applied to state-of-the-art PTQ methods on the Llama-2-7B model. Wiki2-4K and Wiki2-2K represent perplexity on WikiText2 dataset with context size of 4096 and 2048, respectively. W4A4KV4 indicates quantization of all weight, activation, and KV cache to 4 bits.}
\vspace{0.5em}
\centering
\begin{adjustbox}{max width=1.0\columnwidth}
\begin{tabular}{c lcc}
\toprule
 &  
Method &  
Bits$\downarrow$ &  
Wiki2-4K$\downarrow$
\\
\cmidrule(lr){2-2}\cmidrule(lr){3-3}\cmidrule(lr){4-4}
Type & Original & 16 & 5.12\\
\cmidrule(lr){1-1}\cmidrule(lr){2-2}\cmidrule(lr){3-3}\cmidrule(lr){4-4}
\multirow{3}{*}{\shortstack[c]{Weight-only\\ Scalar}}&SqueezeLLM
&2.01 & 39.58
\\
&LNQ (Ours) 
& 2.01 & 23.31
\\
&LNQ + GQuant (Ours)
& 2.01 & \textbf{8.83}
\\
\cmidrule(lr){1-1}\cmidrule(lr){2-2}\cmidrule(lr){3-3}\cmidrule(lr){4-4}
\multirow{2}{*}{\shortstack[c]{Weight-only\\ Vector}}&QTIP
& 2.00 & 6.82 \\
& QTIP + GQuant (Ours)
& 2.00  & \textbf{6.11}
\\
\midrule
 &  
Method &  
Bits$\downarrow$ &  
Wiki2-2K$\downarrow$
\\
\cmidrule(lr){2-2}\cmidrule(lr){3-3}\cmidrule(lr){4-4}
Type & Original & 16 & 5.47 \\
\cmidrule(lr){1-1}\cmidrule(lr){2-2}\cmidrule(lr){3-3}\cmidrule(lr){4-4}
\multirow{2}{*}{\shortstack[c]{Weight-and-\\ Activation}}&SpinQuant
& W4A4KV4 & 5.95
\\
& SpinQuant + GQuant (Ours)
& W4A4KV4 & \textbf{5.89}
\\
\bottomrule
\end{tabular}
\end{adjustbox}
\label{tab:summary}
\end{table}

\begin{figure*}
    \centering
    \resizebox{\textwidth}{!}{%
\begin{tikzpicture}[
    2d-arr/.style={matrix of nodes, row sep=-\pgflinewidth, column sep=-\pgflinewidth, 
    nodes in empty cells,
    nodes={draw,
      minimum size=15pt,
      anchor=center,
      align=center,
      inner sep=0pt
    },}
  ]

  \matrix (sal1) [2d-arr, ampersand replacement=\&] {
  |[fill=red!40, opacity=0.9]| \& |[fill=red!40, opacity=0.9]| \& |[fill=red!20, opacity=0.9]| \& |[fill=red!20, opacity=0.9]| \& |[fill=red!40, opacity=0.9]| \& |[fill=red!40, opacity=0.9]| \\
  |[fill=red!20, opacity=0.9]| \& |[fill=red!40, opacity=0.9]| \& |[fill=red!20, opacity=0.9]| \& |[fill=red!60, opacity=0.9]| \& |[fill=red!40, opacity=0.9]| \& |[fill=red!20, opacity=0.9]| \\
  |[fill=red!40, opacity=0.9]| \& |[fill=red!20, opacity=0.9]| \& |[fill=red!40, opacity=0.9]| \& |[fill=red!40, opacity=0.9]| \& |[fill=red!40, opacity=0.9]| \& |[fill=red!20, opacity=0.9]| \\
  |[fill=red!20, opacity=0.9]| \& |[fill=red!20, opacity=0.9]| \& |[fill=red!60, opacity=0.9]| \& |[fill=red!40, opacity=0.9]| \& |[fill=red!60, opacity=0.9]| \& |[fill=red!60, opacity=0.9]| \\
  |[fill=red!20, opacity=0.9]| \& |[fill=red!40, opacity=0.9]| \& |[fill=red!40, opacity=0.9]| \& |[fill=red!20, opacity=0.9]| \& |[fill=red!40, opacity=0.9]| \& |[fill=red!40, opacity=0.9]| \\
  };
    \node[above=-0.5em of sal1] (ssal1) {{$\frac{\partial \ell}{\partial (\mathbf{X} \mathbf{W})} \in \mathbb{R}^{n \times d_\mathrm{out}}$}};  

  \draw[dashed, black] ($(sal1-5-1.south west) + (-4.0em, -1.0em)$) -- ($(sal1-5-1.south west)+ (46.0em, -1.0em)$);
   
     \node[below=1.0em of sal1,xshift=-3.5em] (str) {{$\forall j \in \{1, \ldots, d_\mathrm{out}\}: $}};
   
    \node[right=2.0em of str,xshift=0em] (strh) {{$\bH_j = \bX^\top \mathrm{Diag}(\frac{\partial \ell}{\partial \bz_j})^2 \bX$}};
   
   \matrix (Hlw1) [2d-arr, ampersand replacement=\&, below=5.0em of sal1-5-1.south east, xshift=-4.0em] {
  |[fill=blue!40, opacity=0.9]| \\
  |[fill=blue!10, opacity=0.9]|         \\
  |[fill=blue!10, opacity=0.9]|         \\
  };
  
 \node[left=0.5em of Hlw1,xshift=1.0em, yshift=0.0em] {$\left(\vphantom{\rule{2em}{2.3em}}\right.$};
 \node[left=0.5em of Hlw1,xshift=8.2em, yshift=0.0em] {$\left.\vphantom{\rule{2em}{2.3em}}\right)^{\!\!\!\top}$};

  \node[above=-0.5em of Hlw1] (sHlw1) {{$\widehat{\mathbf{W}}_{:, j}$}};  

   \node[right=-0.2em of Hlw1] {\Large{$-$}};

     \matrix (Hlw2) [2d-arr, ampersand replacement=\&, right=1.5em of Hlw1] {
  |[fill=blue!20, opacity=0.9]| \\
  |[fill=blue!20, opacity=0.9]| \\
  |[fill=blue!20, opacity=0.9]| \\
  };
  
     \node[above=-0.5em of Hlw2] (sHlw2) {{${\mathbf{W}}_{:, j}$}};
   
  \matrix (H1) [2d-arr, ampersand replacement=\&, right=1.5em of Hlw2] {
  |[fill=red!30!, opacity=0.9]| \& |[fill=red!30!, opacity=0.9]| \& |[fill=red!30!, opacity=0.9]| \\
  |[fill=red!30!, opacity=0.9]| \& |[fill=red!30!, opacity=0.9]| \& |[fill=red!30!, opacity=0.9]| \\
  |[fill=red!30!, opacity=0.9]| \& |[fill=red!30!, opacity=0.9]| \& |[fill=red!30!, opacity=0.9]| \\
  };
    
      \node[above=-0.5em of H1] (sH1) {{${\mathbf{H}}_{j}$}};

   \matrix (Hrw1) [2d-arr, ampersand replacement=\&, right=1.0em of H1] {
  |[fill=blue!40, opacity=0.9]| \\
  |[fill=blue!10, opacity=0.9]|         \\
  |[fill=blue!10, opacity=0.9]|         \\
  };

 \node[left=0.5em of Hrw1,xshift=1.0em, yshift=0.0em] {$\left(\vphantom{\rule{2em}{2.3em}}\right.$};
 \node[left=0.5em of Hrw1,xshift=7.8em, yshift=0.0em] {$\left.\vphantom{\rule{2em}{2.3em}}\right)$};
  
   \node[right=-0.2em of Hrw1] {\Large{$-$}};

      \node[above=-0.5em of Hrw1] (sHrw1) {{$\widehat{\mathbf{W}}_{:, j}$}};
  
     \matrix (Hrw2) [2d-arr, ampersand replacement=\&, right=1.5em of Hrw1] {
  |[fill=blue!20, opacity=0.9]| \\
  |[fill=blue!20, opacity=0.9]| \\
  |[fill=blue!20, opacity=0.9]| \\
  };

    \node[above=-0.5em of Hrw2] (sHrw2) {{${\mathbf{W}}_{:, j}$}};

  \draw[dashed, black] ($(Hrw1-1-1.north west) + (8.0em, 4.0em)$) -- ($(Hrw1-1-1.north west)+ (8.0em, -5.0em)$);

     \node[right=14.5em of str] (str2) {{$\forall k \in \{1, \ldots, g \}: \quad (g \ll d_\mathrm{out})$}};

     \node[right=3.0em of str2,xshift=0em] (strh2) {{$\overbar{\bH}_k = \frac{1}{| J_k |} \sum_{j \in J_k} \bH_j $}};

   \matrix (Hglw1) [2d-arr, ampersand replacement=\&, right=4.0em of Hrw2] {
  |[fill=blue!40, opacity=0.9]| \&  |[fill=blue!40, opacity=0.9]| \\
  |[fill=blue!10, opacity=0.9]|          \&  |[fill=blue!40, opacity=0.9]| \\
  |[fill=blue!10, opacity=0.9]|          \&  |[fill=blue!40, opacity=0.9]| \\
  };
  
   \node[right=-0.2em of Hglw1] {\Large{$-$}};

 \node[left=0.5em of Hglw1,xshift=1.0em, yshift=0.0em] {$\left(\vphantom{\rule{2em}{2.3em}}\right.$};
 \node[left=0.5em of Hglw1,xshift=11.2em, yshift=0.0em] {$\left.\vphantom{\rule{2em}{2.3em}}\right)^{\!\!\!\top}$};

  \node[above=-0.5em of Hglw1] (sHglw1) {{$\widehat{\mathbf{W}}_{:, J_k}$}};
  
     \matrix (Hglw2) [2d-arr, ampersand replacement=\&, right=1.5em of Hglw1] {
 |[fill=blue!20, opacity=0.9]| \&  |[fill=blue!20, opacity=0.9]| \\
 |[fill=blue!20, opacity=0.9]| \&  |[fill=blue!20, opacity=0.9]| \\
 |[fill=blue!20, opacity=0.9]| \&  |[fill=blue!20, opacity=0.9]| \\
  };
  
     \node[above=-0.5em of Hglw2] (sHglw2) {{${\mathbf{W}}_{:, J_k}$}};
   
  \matrix (Hg1) [2d-arr, ampersand replacement=\&, right=1.5em of Hglw2] {
  |[fill=green!20!, opacity=0.9]| \& |[fill=green!20!, opacity=0.9]| \& |[fill=green!20!, opacity=0.9]| \\
  |[fill=green!20!, opacity=0.9]| \& |[fill=green!20!, opacity=0.9]| \& |[fill=green!20!, opacity=0.9]| \\
  |[fill=green!20!, opacity=0.9]| \& |[fill=green!20!, opacity=0.9]| \& |[fill=green!20!, opacity=0.9]| \\
  };
    
      \node[above=-0.5em of Hg1] (sHg1) {{${\overbar{\mathbf{H}}}_{k}$}};

   \matrix (Hgrw1) [2d-arr, ampersand replacement=\&, right=1.0em of Hg1] {
  |[fill=blue!40, opacity=0.9]| \&  |[fill=blue!40, opacity=0.9]| \\
  |[fill=blue!10, opacity=0.9]|          \&  |[fill=blue!40, opacity=0.9]| \\
  |[fill=blue!10, opacity=0.9]|          \&  |[fill=blue!40, opacity=0.9]| \\
  };

 \node[left=0.5em of Hgrw1,xshift=1.0em, yshift=0.0em] {$\left(\vphantom{\rule{2em}{2.3em}}\right.$};
 \node[left=0.5em of Hgrw1,xshift=10.7em, yshift=0.0em] {$\left.\vphantom{\rule{2em}{2.3em}}\right)$};

   \node[right=-0.2em of Hgrw1] {\Large{$-$}};

    \node[above=-0.5em of Hgrw1] (sHgrw1) {{$\widehat{\mathbf{W}}_{:, J_k}$}};
  
     \matrix (Hgrw2) [2d-arr, ampersand replacement=\&, right=1.5em of Hgrw1] {
 |[fill=blue!20, opacity=0.9]| \&  |[fill=blue!20, opacity=0.9]| \\
 |[fill=blue!20, opacity=0.9]| \&  |[fill=blue!20, opacity=0.9]| \\
 |[fill=blue!20, opacity=0.9]| \&  |[fill=blue!20, opacity=0.9]| \\
  };
  
    \node[above=-0.5em of Hgrw2] (sHgrw2) {{${\mathbf{W}}_{:, J_k}$}};


  \node[right=0.0em of sal1] (str) {\Large{$\odot$}};

  \matrix (X) [2d-arr, ampersand replacement=\&, right=1.0em of str, nodes={draw, fill=gray!20, opacity=0.9}] {
    \&  \&   \\
    \&  \&   \\
    \&  \&   \\
    \&  \&   \\
    \&  \&   \\
  };
  
      \node[above=-0.5em of X] (sX) {{${\mathbf{X}} \in \mathbb{R}^{n \times d_\mathrm{in}}$}};


  \matrix (lw1) [2d-arr, ampersand replacement=\&, right=1.0em of X] {
  |[fill=blue!40, opacity=0.9]| \&  |[fill=blue!40, opacity=0.9]| \&  |[fill=blue!40, opacity=0.9]| \&  |[fill=blue!40, opacity=0.9]| \& |[fill=blue!40, opacity=0.9]| \& |[fill=blue!10, opacity=0.9]| \\
  |[fill=blue!10, opacity=0.9]|          \&  |[fill=blue!40, opacity=0.9]| \&  |[fill=blue!40, opacity=0.9]| \&  |[fill=blue!10, opacity=0.9]|          \& |[fill=blue!40, opacity=0.9]| \& |[fill=blue!10, opacity=0.9]| \\
  |[fill=blue!10, opacity=0.9]|          \&  |[fill=blue!40, opacity=0.9]| \&  |[fill=blue!10, opacity=0.9]|          \&  |[fill=blue!10, opacity=0.9]|          \& |[fill=blue!40, opacity=0.9]| \& |[fill=blue!40, opacity=0.9]| \\
  };

 \node[right=0.0em of lw1] (eq) {\Large{$-$}};

 \node[left=0.5em of lw1,xshift=1.0em, yshift=0.0em] {$\left(\vphantom{\rule{2em}{2.7em}}\right.$};
 \node[left=0.5em of lw1,xshift=23.0em, yshift=0.0em] {$\left.\vphantom{\rule{2em}{2.7em}}\right)$};

 \node[left=0.5em of lw1,xshift=-5.5em, yshift=0.0em] {$\left(\vphantom{\rule{2em}{3.5em}}\right.$};
 \node[left=0.5em of lw1,xshift=24.0em, yshift=0.0em] {$\left.\vphantom{\rule{2em}{3.5em}}\right)$};

 \node[left=0.5em of lw1,xshift=-18.0em, yshift=0.0em] {$\left\lVert\vphantom{\rule{2em}{4em}}\right.$};
 \node[left=0.5em of lw1,xshift=25.5em, yshift=0.0em] {$\left.\vphantom{\rule{2em}{4em}}\right\rVert_F^2$};

  \matrix (rw1) [2d-arr, ampersand replacement=\&, right=1.8em of lw1] {
  |[fill=blue!20, opacity=0.9]| \&  |[fill=blue!20, opacity=0.9]| \&  |[fill=blue!20, opacity=0.9]| \&  |[fill=blue!20, opacity=0.9]| \& |[fill=blue!20, opacity=0.9]| \& |[fill=blue!20, opacity=0.9]| \\
  |[fill=blue!20, opacity=0.9]| \&  |[fill=blue!20, opacity=0.9]| \&  |[fill=blue!20, opacity=0.9]| \&  |[fill=blue!20, opacity=0.9]| \& |[fill=blue!20, opacity=0.9]| \& |[fill=blue!20, opacity=0.9]| \\
  |[fill=blue!20, opacity=0.9]| \&  |[fill=blue!20, opacity=0.9]| \&  |[fill=blue!20, opacity=0.9]| \&  |[fill=blue!20, opacity=0.9]| \& |[fill=blue!20, opacity=0.9]| \& |[fill=blue!20, opacity=0.9]| \\
  };

  \node[above=-0.3em of rw1] {$\mathbf{W} \in \mathbb{R}^{d_\mathrm{in} \times d_\mathrm{out}}$};

  \node[above=-0.3em of lw1] {$\widehat{\mathbf{W}} \in \mathbb{R}^{d_\mathrm{in} \times d_\mathrm{out}}$};

    \draw [decorate,decoration={brace,amplitude=5pt,mirror,raise=4ex}] 
  ($(rw1-2-1.south west) + (0.2em, 0em)$) -- ($(rw1-2-3.south west) + (-0.2em, 0em)$) node[midway,yshift=-3em]{$J_1$};
  
  \node[below=1.0em of rw1-3-4.south west]{$\ldots$};
  
  \draw [decorate,decoration={brace,amplitude=5pt,mirror,raise=4ex}] 
  ($(rw1-2-5.south west) + (0.2em, 0em)$) -- ($(rw1-2-6.south east) + (-0.2em, 0em)$) node[midway,yshift=-3em]{$J_g$};

    \draw [decorate,decoration={brace,amplitude=5pt,mirror,raise=4ex}] 
  ($(lw1-2-1.south west) + (0.2em, 0em)$) -- ($(lw1-2-3.south west) + (-0.2em, 0em)$) node[midway,yshift=-3em]{$J_1$};
  
  \node[below=1.0em of lw1-3-4.south west]{$\ldots$};
  
  \draw [decorate,decoration={brace,amplitude=5pt,mirror,raise=4ex}] 
  ($(lw1-2-5.south west) + (0.2em, 0em)$) -- ($(lw1-2-6.south east) + (-0.2em, 0em)$) node[midway,yshift=-3em]{$J_g$};

\end{tikzpicture}
}

    \caption{
    Top: The proposed GuidedQuant's layer-wise quantization objective \eqref{eq:ours}. Bottom-left: Its equivalent quadratic form \eqref{eq:ours-quad}. Bottom-right: The approximated objective \eqref{eq:ours2} proposed in \cref{sec:challenge}. We denote the input, weight, and quantized weight matrices as $\bX \in \mathbb{R}^{n \times d_\mathrm{in}}$, $\mathbf{W}\in\mathbb{R}^{d_\mathrm{in} \times d_\mathrm{out}}$, and  $\hat{\mathbf{W}}\in\mathbb{R}^{d_\mathrm{in} \times d_\mathrm{out}}$, respectively. The groups $J_1, \ldots, J_g$ form a partition of the set $\{1, \ldots, d_\mathrm{out}\}$, and $\bz_j \in \mathbb{R}^{d_\mathrm{out}}$ denotes the $j$-th column of $\bZ = \bX\bW$.}
    \label{fig:qual2}
\end{figure*}

\looseness=-1 Large language models (LLMs) have shown remarkable capabilities across a range of tasks, from text generation to complex reasoning.  However, these advancements come at the cost of substantial memory usage and inference latency. Quantization provides an effective solution to these challenges. 
Weight-only quantization methods quantize only the model weights, reducing data transfer and thus accelerating inference in memory-bound scenarios such as small-batch inference \citep{gholami2024ai, kim2023squeezellm, tseng2024qtip}. 
On the other hand, weight-and-activation quantization methods quantize both the model weights and activations. In addition to reducing data transfer, these methods also speed up arithmetic operations, making them particularly beneficial for large-batch scenarios such as pre-filling input tokens or generating batched samples \citep{ashkboos2024quarot, liu2024spinquant}.
Weight-only quantization techniques have used three grid types: \emph{uniform scalar} \cite{frantar2022gptq}, \emph{non-uniform scalar} \cite{kim2023squeezellm}, and \emph{vector} quantization \citep{tseng2024qtip, van2024gptvq}, each with its own advantages (see \cref{sec:exp} for details). In contrast, weight-and-activation methods typically use a uniform scalar grid, as using a non-uniform grid would require dequantization before multiplication, preventing the use of faster arithmetic operations.

Quantization benefits come at the cost of performance degradation. Quantization-Aware Training (QAT) methods rely on retraining the quantized model to mitigate this, which is prohibitively expensive at the scale of modern LLMs.
In constrast, Post-Training Quantization (PTQ) methods quantize the pretrained model using a small calibration dataset or no data,
\textit{without} retraining the entire model. Most existing PTQ methods for LLMs rely on a surrogate objective rather than the end loss to make quantization feasible.

One common PTQ strategy, which we refer to as layer-wise output-based quantization, 
aims to quantize each layer by minimizing the mean squared error between the layer's original output and the quantized one \citep{Nagel2020, frantar2022gptq, egiazarian2024extreme, chee2024quip, tseng2024quip, tseng2024qtip, liu2024spinquant}.
However, this strategy treats all hidden features equally, overlooking their varying impact on the end loss.

Alternatively, methods such as \citet{choi2016towards,kim2023squeezellm} leverage gradient information from the end loss to assess the impact of individual weight errors. This is done by computing the gradient of the end loss with respect to weights via a single backpropagation step on a calibration dataset. Saliency scores are then assigned to weights based on these gradients, and the model is quantized by approximately minimizing the sum of saliency-weighted weight errors.
This objective corresponds to a quadratic approximation of the change in the end loss, based on its second-order Taylor expansion, where the Hessian is 
approximated by the \emph{diagonal} of the empirical Fisher information matrix \cite{hassibi1992second}.
A key limitation of this approach is that it ignores cross-weight interactions, which are crucial for overall performance.

\paragraph{Contributions} In this work, we propose \textit{GuidedQuant}, a novel PTQ approach that integrates gradient information from the end loss while preserving cross-weight dependencies within output channels. In particular, GuidedQuant computes saliency scores for layer outputs using the gradients of the end loss with respect to these outputs. 
Each layer is then quantized independently by approximately minimizing the sum of saliency-weighted output errors.
Unlike previous methods that assume a diagonal Hessian, this objective is equivalent to a refined quadratic approximation assuming a \emph{block-diagonal} Hessian, again approximated by the empirical Fisher information matrix. While cross-layer and cross-output channel interactions are still ignored, dependencies within output channels are preserved, enabling a more accurate estimation of quantization's impact on the end loss.

Computing and storing the diagonal blocks of the Fisher matrix for a given layer is too expensive for modern LLMs.
To address this, we partition the layer's outputs into a small number of groups and average the Fisher matrix's blocks within each group (\cref{fig:qual2}).
Other block-diagonal Fisher matrix approximations of the Hessian have been used for pruning CNNs \citep{singh2020woodfisher} and BERT LLMs \cite{kurtic2022optimal} with arbitrary blocks along the diagonal, and for quantizing CNNs \citep{li2021brecq} with diagonal blocks corresponding to the model's residual blocks (see \cref{app:blk_hess} for more details). However, our work is the first to make this approach computationally and storage-efficient at the scale of modern LLMs.

GuidedQuant can be applied as a direct plug-in 
to any layer-wise output-based 
PTQ method. 
We demonstrate its effectiveness by integrating it into the current state-of-the-art methods for weight-only vector quantization, QTIP \cite{tseng2024qtip}, and weight-and-activation quantization, SpinQuant \cite{liu2024spinquant}, which are both layer-wise output-based PTQ methods. GuidedQuant consistently improves their performance (\cref{tab:summary}).


For weight-only scalar quantization, the current state-of-the-art methods are SqueezeLLM \cite{kim2023squeezellm} and GPTVQ 1D \citep{van2024gptvq}. 
Since GPTVQ 1D is a layer-wise output-based PTQ method, GuidedQuant can be applied to it. However, GPTVQ 1D employs a suboptimal algorithm for minimizing layer-wise output errors. To address this, we introduce a novel Layer-wise Non-uniform Quantization method, \textit{LNQ}, which minimize layer-wise output errors using an alternating minimization algorithm, where the codebook is optimized in closed-form, and assignments are optimized via a coordinate descent (CD) algorithm. 
LNQ outperforms GPTVQ 1D and matches or surpasses SqueezeLLM. Applying GuidedQuant to LNQ further improves its performance, achieving state-of-the-art results (\cref{tab:summary}).

\section{Preliminaries}

\looseness=-1 Consider a neural network with $L$ linear layers, trained with a loss function $\ell$ and a calibration data of size $n$. We denote the loss computed on the $i$-th data point as $\ell_i$.
Let $\bW^{(l)} \in \mathbb{R}^{d_{\mathrm{in}}^{(l)} \times d_{\mathrm{out}}^{(l)}}$ be the weight matrix of the $l$-th linear layer, where each column vector $\bw^{(l)}_j \in \mathbb{R}^{d_{\mathrm{in}}^{(l)}}$ corresponds to an output channel. We denote its quantized approximation as $\hbW^{(l)}$.
The input and output feature maps of this layer are $\bX^{(l)} \in \mathbb{R}^{n \times d_{\mathrm{in}}^{(l)}}$ and $\bZ^{(l)} \in \mathbb{R}^{n \times d_{\mathrm{out}}^{(l)}}$, respectively. The output of the linear layer is computed as $\bZ^{(l)} = \bX^{(l)} \bW^{(l)}$, and the output after quantization as $\hbZ^{(l)} = \bX^{(l)} \hbW^{(l)}$. Let $\bw = [\vect(\bW^{(1)})^\top, \cdots, \vect(\bW^{(L)})^\top]^\top$ and $\hbw =  [\vect(\hbW^{(1)})^\top, \cdots, \vect(\hbW^{(L)})^\top]^\top$ be the vectors of weights in all $L$ layers before and after quantization, where $\vect(\bW^{\ell})$ corresponds to stacking the columns of $\bW^{\ell}$.
 
Most existing PTQ methods for LLMs are layer-wise output-based quantization methods, which quantize each layer by approximately minimizing the objective 
\begin{equation}\label{eq:layerwiseObj}
    \| \bX^{(l)} {\bW}^{(l)} - \bX^{(l)} \widehat{\bW}^{(l)}\|_F^2 = \sum_{i=1}^n \sum_{j=1}^{d_\mathrm{out}^{(l)}} \left(Z_{ij}^{(l)} - \widehat{Z}_{ij}^{(l)} \right)^2
\end{equation}
 ignoring the varying impact of outputs on the end loss $\ell$. 
Existing methods employ various heuristics to minimize this objective, such as AdaRound \cite{Nagel2020}, CD methods \cite{nair2024cdquant, behdin2023quantease, egiazarian2024extreme, chee2024quip}, OBQ \cite{frantar2022optimal}, GPTQ\footnote{Also referred to as OPTQ.} \cite{frantar2022gptq}, GPTVQ \cite{van2024gptvq}, and AQLM \cite{egiazarian2024extreme}.

A more accurate proxy objective, first introduced in early pruning methods \citep{lecun1989optimal, hassibi1992second}, is the following quadratic approximation of the change in the end loss 
\begin{equation}\label{eq:second-order}
    \ell(\hbw) - \ell(\bw) \approx \tfrac{1}{2} (\hbw - \bw)^\top \nabla^2 \ell(\bw) (\hbw - \bw).
\end{equation}
This approximation is derived from the second-order Taylor approximation of $\ell$, assuming that the trained model has converged and thus the gradient is close to zero. Since computing the Hessian is infeasible even for small models, a popular approach first proposed in \citet{hassibi1992second} approximates the Hessian by the empirical Fisher information matrix $\bF = \frac{1}{n}\sum_{i=1}^n \nabla \ell_i(\mathbf{w}) \nabla \ell_i(\mathbf{w})^\top$, which yields the following quadratic approximation, $(\hbw - \bw)^\top \bF (\hbw - \bw).$

SqueezeLLM \cite{kim2023squeezellm} is a weight-only non-uniform scalar PTQ method for LLMs which uses this quadratic approximation, but further approximates the Fisher information matrix by its diagonal $\diag(\bF)$, ignoring off-diagonal entries. The resulting objective is given by
\begin{equation}
\label{eq:kmeans}
(\hbw - \bw)^\top \diag(\bF) (\hbw - \bw) = \sum_k F_{kk} (\widehat{w}_k - w_k)^2.\end{equation} 
\looseness=-1 
For non-uniform scalar quantization, minimizing this objective corresponds to solving a weighted $k$-means problem in 1D, which can be solved exactly 
using a dynamic programming algorithm \citep{Groenlund2017Fast}. SqueezeLLM instead employs Lloyd's algorithm  with $k$-means++ initialization \citep{lloyd1982least,arthur2007k}, 
which is only guaranteed to achieve a $\Theta(\log k)$ approximation in expectation, where $k$ is the number of clusters, but is faster in practice \cite{Hyun2024}.
However, the diagonal approximation is highly inaccurate, as both the Hessian matrix and its Fisher approximation are usually strongly non-diagonal, as observed in prior work for small CNNs \citep{hassibi1992second, singh2020woodfisher}. We also confirm this observation for the Fisher matrix of Llama-2-7B in \cref{fig:fisher1,fig:fisher2}. 

\section{GuidedQuant}

In this section, we introduce our PTQ approach \textit{GuidedQuant}. We first propose a layer-wise quantization objective that more accurately approximates the impact of quantization on the final loss compared to surrogate objectives used in existing PTQ methods. We then present a simplified version of this objective, making it computationally and memory efficient for LLMs with up to 70B parameters.
\looseness=-1

\begin{figure}[t]
    \centering
\begin{tikzpicture}

\begin{groupplot}[
        group style={columns=2, rows=1, horizontal sep=1.3cm, 
        vertical sep=3.5cm},
        grid style = {
            dash pattern=on 3pt off 2pt,
            line width = 0.4pt
        }
        ]

\nextgroupplot[
            width=0.47\textwidth,
            height=5cm,
            every axis plot/.append style={thick},
            xmajorgrids={true},
            xminorgrids={true},
            ymajorgrids={true},
            yminorgrids={true},
            title style={at={(0.5,1.3)},anchor=south},
            title={},
            tick label style={font=\scriptsize},
            tick pos=left,
            xlabel near ticks,
            ylabel near ticks,
            minor y tick num=0,
            xtick={2.0, 3.0, 4.0},
            xticklabels={2.0, 3.0, 4.0},
            ytick={5.0, 5.5, 6.0, 6.5, 7.0, 7.5},
            yticklabels={5.0, 5.5, 6.0, 8, 24, 40},
            xmin=1.9,
            xmax=4.1,
            ymin=4.9,
            ymax=7.7,
            xlabel={Bits $\downarrow$},
            ylabel={WikiText2 Perplexity $\downarrow$},
            xlabel shift=0.0cm,         
            ylabel shift=0.0cm,
            xlabel style={align=center},
            label style={font=\scriptsize},
            legend to name=grouplegend1,
            legend style={legend columns=2, font=\scriptsize},
            legend cell align={left},
            after end axis/.code={
            \draw (rel axis cs:0,0.48) +(-2mm,-1mm) -- +(2mm,1mm)
              ++(0pt,-\pgfkeysvalueof{/tikz/axis break gap})
              +(-2mm,-1mm) -- +(2mm,1mm)
              (rel axis cs:0,0.48) +(0mm,0mm) -- +(0mm,0mm)
              ++(0pt,-\pgfkeysvalueof{/tikz/axis break gap})
              +(-2mm,-1mm) -- +(2mm,1mm);
              \draw (rel axis cs:0,0.48) +(-2mm,0mm) -- +(2mm,2mm)
              ++(0pt,-\pgfkeysvalueof{/tikz/axis break gap})
              +(-2mm,-1mm) -- +(2mm,1mm)
              (rel axis cs:0,0.48) +(0mm,0mm) -- +(0mm,0mm)
              ++(0pt,-\pgfkeysvalueof{/tikz/axis break gap})
              +(-2mm,-1mm) -- +(2mm,1mm);
            }
            ]


\addplot[mark=none, black!75, samples=2, dashed] {5.12};\addlegendentry{Original-FP16}

\addplot[gr, mark=otimes] table [x=bits, y=kmeans, col sep=comma]{data/ppl.csv};\addlegendentry{SqueezeLLM (weighted $k$-means)}

\addplot[bl, mark=diamond] table [x=bits, y=layerwise, col sep=comma]{data/ppl.csv};\addlegendentry{LNQ (layer-wise)}


\addplot[red, mark=triangle] table [x=bits, y=ours, col sep=comma]{data/ppl.csv};\addlegendentry{LNQ + GuidedQuant (our objective)}

\coordinate (c11) at (rel axis cs:0,1);

\coordinate (c13) at (rel axis cs:1,1);

\end{groupplot}

\coordinate (mid) at ($(c11)!.5!(c13)$);
\node[above] at ($(mid)+(-0.00, 0.05)$) {\pgfplotslegendfromname{grouplegend1}};

\end{tikzpicture}
\caption{Non-uniform scalar quantization results on Llama-2-7B with different objectives: \textit{layer-wise} output error objective \eqref{eq:layerwiseObj} used in LNQ (\cref{alg:lnq}),
\textit{weighted $k$-means} objective \eqref{eq:kmeans} used in SqueezeLLM,
and our approximated \textit{GuidedQuant} objective \eqref{eq:ours2} used in LNQ combined with GuidedQuant.
We report perplexity on WikiText2 with a context size of 4096. Results are  from \cref{tab:scalar}.
}
\label{fig:obj}
\end{figure}
\subsection{Objective}

As discussed earlier, most existing PTQ methods treat all output features as equally important, by employing the surrogate objective in Eq. \eqref{eq:layerwiseObj}.
In contrast, we propose to modify this objective to account for the varying impact of each output feature on the final loss.

To that end, we approximate the change in the end loss $\ell$ resulting from the output feature $Z_{ij}^{(l)}$ changing to $\widehat{Z}_{ij}^{(l)}$ after quantization, using a first-order Taylor expansion, assuming independence of output features: $$\ell(\widehat{Z}_{ij}^{(l)}) - \ell(Z_{ij}^{(l)}) \approx \frac{\partial \ell}{\partial Z_{ij}^{(l)}} (\widehat{Z}_{ij}^{(l)} - Z_{ij}^{(l)}).$$ 

Accordingly, we propose to scale each output error by the gradient of the end loss with respect to that output, leading to the following layer-wise objective:
\begin{align} \label{eq:ours}
& \left\lVert \frac{\partial \ell}{\partial \bZ^{(l)}} \odot (\bX^{(l)} {\bW}^{(l)} - \bX^{(l)} \widehat{\bW}^{(l)}) \right\rVert_F^2 \nonumber \\
& = \sum_{i=1}^n \sum_{j=1}^{d_\mathrm{out}^{(l)}} \left(\frac{\partial \ell}{\partial Z_{ij}^{(l)}} (Z_{ij}^{(l)} - \widehat{Z}_{ij}^{(l)} )\right)^2,
\end{align}
where $\odot$ denotes the element-wise multiplication. This criterion was previously proposed in \citet{molchanov2019importance} for pruning neurons and filters in vision models, where pruning the $j$th neuron in layer $l$ corresponds to setting $\widehat{\bw}_j^{(l)} = 0$.


We note that the objective in Eq. \eqref{eq:ours} can be viewed as a simplification of the second-order Taylor approximation of the change in the end loss given in Eq. \eqref{eq:second-order}, where the Hessian is approximated by the empirical Fisher information matrix, and where interactions between weights belonging to different layers or output channels of the same layer are ignored. 
In other words, we adopt a \emph{block-diagonal} approximation of the Fisher matrix $\bF$ where we only keep the $d_\mathrm{in}^{(l)} \times d_\mathrm{in}^{(l)}$ blocks $\bF^{(l)}_j= \tfrac{1}{n} \sum_{i=1}^n (\frac{\partial \ell_i}{\partial \bw^{(l)}_j})(\frac{\partial \ell_i}{\partial \bw^{(l)}_j})^\top$ corresponding to interactions within each output channel $j$ of every layer $l$, and ignore all off-block entries.

\begin{restatable}{remark}{chainrule} \label{prop:chainrule} 
The sum of the layer-wise objective in Eq. \eqref{eq:ours} over all layers is equal to the following quadratic approximation of the change in the end loss 
\begin{equation}\label{eq:ours-sum}
    n \sum_{l=1}^L  \sum_{j=1}^{d_\mathrm{out}^{(l)}} (\bw^{(l)}_j - \widehat{\bw}^{(l)}_j)^\top \bF^{(l)}_j (\bw^{(l)}_j - \widehat{\bw}^{(l)}_j).
\end{equation}

\end{restatable}
The proof of \cref{prop:chainrule} follows from the chain rule, and is given in \Cref{app:proof}. A similar observation was made in \citet{molchanov2019importance}.

Assuming that the quantization grid used is separable over layers, which is typically the case, minimizing the objective in Eq. \eqref{eq:ours-sum} is equivalent to independently minimizing 
\begin{equation}\label{eq:ours-quad}
    \sum_{j=1}^{d_\mathrm{out}^{(l)}} (\bw^{(l)}_j - \widehat{\bw}^{(l)}_j)^\top \bH^{(l)}_j (\bw^{(l)}_j - \widehat{\bw}^{(l)}_j), 
\end{equation}
for every layer, where $\bH^{(l)}_j = n \bF^{(l)}_j$, or equivalently the layer-wise objective in Eq. \eqref{eq:ours}. 

Thus our proposed objective is a more accurate approximation of the change in the end loss than the layer-wise output error objective \eqref{eq:layerwiseObj}, which assumes 
$ \frac{\partial \ell}{\partial \bZ^{(l)}} \propto \mathbf{I}$, as well as the weighted $k$-means objective \eqref{eq:kmeans} used in SqueezeLLM, which ignores all off-diagonal entries in the Fisher matrix including those within the blocks $\bF^{(l)}_j$. As a result, our approach achieves better performance, even with the additional approximation discussed in \cref{sec:challenge}, as highlighted in \cref{fig:obj} for non-uniform scalar quantization, and later across other formats in \cref{sec:exp}.

In \cref{fig:fisher1,fig:fisher2}, we visualize a submatrix of the Fisher information matrix corresponding to the first two output channels in the linear layers of the first Transformer block of Llama-2-7B. The visualization confirms that the Fisher matrix exhibits strong off-diagonal values and a prominent block-diagonal structure, with blocks corresponding to $\bF^{(l)}_j$ for the two output channels $j \in \{1,2\}$.

\subsection{Averaging Approximation}
\label{sec:challenge}

The layer-wise output error objective \eqref{eq:layerwiseObj} can be  written as $$\sum_{j=1}^{d_\mathrm{out}^{(l)}} \left(\bw^{(l)}_j - \widehat{\bw}^{(l)}_j\right)^\top \bH^{(l)} \left(\bw^{(l)}_j - \widehat{\bw}^{(l)}_j\right),$$ where $\bH^{(l)} = \bX^{(l)}{}^\top \bX^{(l)} \in \mathbb{R}^{d_\mathrm{in}^{(l)} \times d_\mathrm{in}^{(l)}}$.
Most existing heuristics for optimizing this objective, such as GPTQ \cite{frantar2022gptq} and CD \cite{nair2024cdquant, behdin2023quantease},  require access  only to $\bH^{(l)}$ and not $\bX^{(l)}$. 
Thus, the Hessian matrix $\bH^{(l)}$ is typically precomputed,  which reduces the peak memory usage during optimization,
since $\bX^{(l)} \in \mathbb{R}^{n\times d_\mathrm{in}^{(l)}}$ is much larger than $\bH^{(l)}$, given that $d_\mathrm{in}^{(l)} \ll n$.
Additionally, the precomputed Hessian can be reused across multiple quantization configurations and bit-widths, amortizing the cost of its computation.

Our proposed objective \eqref{eq:ours-quad} can be  seamlessly integrated into any layer-wise output based quantization method by replacing $\bH^{(l)}$ by $\bH^{(l)}_j = n \bF^{(l)}_j$ for each output channel $j$. 
However, precomputing and storing $\bH^{(l)}_j$ for all $j$ incurs a memory cost of $\Theta ((d_\mathrm{in}^{(l)})^2 d_\mathrm{out}^{(l)})$ and a time complexity of $\Theta (n (d_\mathrm{in}^{(l)})^2 d_\mathrm{out}^{(l)})$ per layer $l$. This is infeasible at the scale of modern LLMs, where both $d_\mathrm{in}^{(l)}$ and $d_\mathrm{out}^{(l)}$ exceed $10^3$, and $n$ is much larger than both.


To address this challenge, we partition the output channels of each layer into $g$ distinct groups ($g \ll d_\mathrm{out}^{(l)}$) and 
replace the individual Hessian matrices $\bH^{(l)}_j$ within each group $k$ by a shared matrix $\overbar{\bH}_k^{(l)}$, obtained by averaging $\bH^{(l)}_j$ within the group. 
Formally, let $J_1^{(l)}, \ldots, J_g^{(l)}$ be a partition of the set $\{1, \ldots, d_\mathrm{out}^{(l)}\}$. 
For each group $k = 1, \ldots, g$, we define 
$
    \overbar{\bH}_k^{(l)} = \tfrac{1}{| J_k^{(l)} |} \sum_{j \in J_k^{(l)}} \bH_j^{(l)}. 
$
The resulting layer-wise objective then becomes
\begin{equation}
\label{eq:ours2}
\sum_{k=1}^g \sum_{j \in J_k^{(l)}} \left(\bw^{(l)}_j - \widehat{\bw}^{(l)}_j\right)^\top   \overbar{\bH}_k^{(l)} \left(\bw^{(l)}_j - \widehat{\bw}^{(l)}_j\right).
\end{equation}

Note that by the chain rule, we can write $$\bH_j^{(l)} = \bX^{(l)}{}^\top \Diag\left(\frac{\partial \ell}{\partial \mathbf{z}^{(l)}_j}\right)^2 \bX^{(l)},$$ where $\Diag(\frac{\partial \ell}{\partial \mathbf{z}^{(l)}_j})^2$ is the diagonal matrix whose diagonal entries 
are the element-wise square of the gradient of $\ell$ with respect to the $j$th column $\bz_j^{(l)}$ of $\bZ^{(l)}$. 
We can thus compute $\overbar{\bH}_k^{(l)}$ by averaging the squared gradients: 
$$\overbar{\bH}_k^{(l)} = \bX^{(l)}{}^\top \Diag\left(\frac{1}{| J_k |} \sum_{j \in J_k}\left(\frac{\partial \ell}{\partial \bz^{(l)}_j}\right)^2\right) \bX^{(l)}.$$

This averaging approximation reduces the number of $d_\mathrm{in}^{(l)} \times d_\mathrm{in}^{(l)}$ Hessian matrices that need to be computed for each layer $l$ from $d_\mathrm{out}^{(l)}$ to $g$ (\cref{fig:qual2}). Computing and storing $\overbar{\bH}_k^{(l)}$  for all $k$ requires a significantly lower memory cost of $\Theta((d_\mathrm{in}^{(l)})^2 g)$ and time complexity of $\Theta(n (d_\mathrm{in}^{(l)})^2 g)$ per layer $l$ 
(assuming the squared gradients averages are already computed), making the method scalable.
To partition the output channels, we use a simple strategy that groups every $d_\mathrm{out}^{(l)}/g$ consecutive channels into a single group.
This simple approach works well in practice, though more sophisticated clustering algorithms may yield additional benefits.
In our implementation, we scale the gradients by a large constant (we used 
$10^3$ in all experiments) while computing the averaged Hessians $\overline{\bH}_k$ to prevent underflow.

GuidedQuant quantizes each layer independently by approximately minimizing the layer-wise objective in Eq. \eqref{eq:ours2}.
A complete overview of GuidedQuant is provided in \Cref{alg:gq}. As discussed, the layer-wise quantization algorithm $\mathcal{Q}$ can be any layer-wise output based quantization method. The gradient computation (Line 2) requires a single backpropagation step on the calibration dataset.
During this step, we only store the averaged squared gradients $\mathbf{s}_k^{(l)}$, which requires $O(n g L)$ storage.

\looseness=-1 The total memory cost of GuidedQuant (without the backpropagation step) is then $O(L g (d_\mathrm{in}^2 + n))$, and its total time complexity is $O\left(L g ( n d_\mathrm{in}^2 + ~T_{\mathcal{Q}}({d_\mathrm{in}, d_\mathrm{out}}/{g}))\right)$, where $d_\mathrm{in}, d_\mathrm{out}$ are the largest input and output channel dimensions across all $L$ layers and $T_{\mathcal{Q}}(d_1, d_2)$ is the time complexity of quantizing a $d_1 \times d_2$-weight matrix using $\mathcal{Q}$.
Each step in the for loop (Lines 3-6) can be done in parallel for all groups and layers. As previously discussed, the Hessian matrices $\overline{\bH}_k$'s only need to be computed once, and can be reused for different quantization configurations and bit-widths.



    \begin{algorithm}[tb]
    \caption{GuidedQuant}
    \label{alg:gq}
    \begin{algorithmic}[1]
    \INPUT Layer-wise quantization algorithm $\mathcal{Q}$, number of groups $g$, number of linear layers $L$\\
 \STATE $J_k^{(l)} \!\!\gets \{\tfrac{d_\mathrm{out}^{(l)}}{g}(k-1) + 1, \ldots, \tfrac{d_\mathrm{out}^{(l)}}{g}k\}, \forall l \in [L], k \in [g]$
\STATE $\mathbf{s}_k^{(l)} \gets \frac{1}{| J_k |} \sum_{j \in J_k}(\frac{\partial \ell}{\partial \bz^{(l)}_j})^2, \forall l \in [L],  k \in [g]$

\FORALL{$l \in [L], k \in [g]$}
        \STATE $\overbar{\bH}^{(l)}_k \leftarrow {\bX^{(l)}} {}^\top \Diag(\mathbf{s}_k^{(l)}) \bX^{(l)}$
        \STATE $\widehat{\bW}^{(l)}\!\left[:, J_k^{(l)}\right] \leftarrow \mathcal{Q}\left(\overbar{\bH}^{(l)}_k, \bW^{(l)}\!\left[:, J_k^{(l)}\right]\right)$ 
    \ENDFOR
    \OUTPUT $\widehat{\bW}^{(1)}, \ldots, \widehat{\bW}^{(L)}$.
    \end{algorithmic}
    \end{algorithm}

\begin{algorithm}[tb]
\caption{LNQ}
\label{alg:lnq}
\begin{algorithmic}[1]
\INPUT Hessian of the objective $\bH \in \mathbb{R}^{d_\mathrm{in}\times d_\mathrm{in}}$, input weight $\bW \in  \mathbb{R}^{d_\mathrm{in} \times d_\mathrm{out}}$, initial 
assignment $\bP^{(j)}\in\mathbb{R}^{d_\mathrm{in} \times m}$ for each output channel $j$.\\

\STATE $\bH = \bL \bL^\top$ \hfill \COMMENT{\textit{Cholesky decomposition}}
\FOR{$j \in \{1, \dots, d_\mathrm{out}\}$}
    \FOR{$t = 1$ {\bf to} $T$} 
    \STATE $\mathbf{c}^{(j)} \leftarrow \left(\bP^{(j)}{}^\top \bL \bL^\top \bP^{(j)}\right)^{-1}\bP^{(j)}{}^\top \bL \bL^\top \bw_j$
    \STATE $\hat{\mathbf{\bw}}_{j} \leftarrow \bP^{(j)} \bc^{(j)}$
        \FOR{$k = 1$ {\bf to} $K$}
        \FOR{$i = 1$ {\bf to} $d_\mathrm{in}$}
        \STATE $c^{(j)}_{q^*} \!\! \leftarrow \!\!\!\!\!\!\!\!\!\underset{\widehat{W}_{ij} \in \{c_1^{(j)}, \ldots, c_m^{(j)}\}}{\mathrm{argmin}} \!\!\!\!\!\! (\hbw_j \!-\! \bw_j)^\top \bH (\hbw_j \!-\! \bw_j)$
        \STATE $\widehat{W}_{ij} \leftarrow c_{q^*}^{(j)}$
        \STATE $\forall q\in \{1, \ldots, m\}: P^{(j)}_{iq} =
                \begin{cases} 
                1 & \text{if } q = q^*, \\
                0 & \text{otherwise.}
                \end{cases}$
        \ENDFOR
        \ENDFOR
    \ENDFOR
    \STATE $\mathbf{c}^{(j)} \leftarrow \left(\bP^{(j)}{}^\top \bL \bL^\top \bP^{(j)}\right)^{-1}\bP^{(j)}{}^\top \bL \bL^\top \bw_j$
\ENDFOR
\OUTPUT $\widehat{\bW} = [\bP^{(1)} \bc^{(1)}, \ldots, \bP^{(d_\mathrm{out})} \bc^{(d_\mathrm{out})}]$.
\end{algorithmic}
\end{algorithm}

\section{Layer-wise Non-uniform Quantization}


The choice of the layer-wise output based quantization method $\mathcal{Q}$ in 
GuidedQuant is critical to its overall performance. For weight-only non-uniform scalar quantization, the current state-of-the-art layer-wise output based quantization method is the 1D variant of GPTVQ \citep{van2024gptvq}, which alternates between optimizing the codebook via gradient descent and the assignments via GPTQ algorithm \citep{frantar2022gptq}. However, both of these steps can be improved. Given fixed assignments, the codebook admits an optimal closed form solution. Also, for optimizing assignments, recent works have demonstrated that coordinate descent (CD) methods outperform GPTQ in uniform weight-only quantization \citep{behdin2023quantease,nair2024cdquant}. In this section, we introduce Layer-wise Non-uniform Quantization (LNQ), an alternating minimization algorithm which leverages the closed form solution for the codebook and employs CD to optimize the assignments. We then discuss its theoretical guarantees, as well as the memory cost and computational complexity under our efficient implementation.

\subsection{Optimization Problem}


We omit the layer index $l$ for notational simplicity throughout this section.
Following prior work, we assign to each output channel a separate codebook, though LNQ can be easily adapted to finer-granularity grouping.
Non-uniform scalar quantization maps each scalar weight in the column $\mathbf{w}_j \in \mathbb{R}^{d_\mathrm{in}}$ to one of $m=2^b$ real values $\{c^{(j)}_1, \ldots, c^{(j)}_m\}$, where $b\in \mathbb{N}$ is the target bit-width. The quantized weights $\widehat{\mathbf{w}}_j$ can then be expressed as $\widehat{\mathbf{w}}_j = \bP^{(j)} \mathbf{c}^{(j)}$, where $\mathbf{c}^{(j)} \in \mathbb{R}^m$ is the vector containing the codebook values $\{c^{(j)}_1, \ldots, c^{(j)}_m\}$, and $\bP^{(j)} \in \{0, 1\}^{d_\mathrm{in} \times m}$ is the assignment matrix such that $P^{(j)}_{iq} = 1$ if $W_{ij}$ is assigned to $c^{(j)}_q$, and $P^{(j)}_{iq} = 0$ otherwise.

The optimization problem for layer-wise output-based non-uniform scalar quantization 
can then be written as follows:
\begin{align}
\label{eq:nuq-master}
\underset{\substack{\bP^{(j)} \in \{0, 1\}^{d_\mathrm{in} \times m} \\ \mathbf{c}^{(j)} \in \mathbb{R}^m}}{\mathrm{minimize}}\;\; 
&\sum_{j=1}^{d_{\mathrm{out}}}\| \bX\mathbf{w}_j - \bX\bP^{(j)}\mathbf{c}^{(j)} \|_2^2 \nonumber \\
\mathrm{subject\;to\;}~~~~ &\bP^{(j)}\mathbf{1}_m = \mathbf{1}_{d_\mathrm{in}}, 
\end{align}
where $\mathbf{1}$ is the vector of all ones.
Note that the optimization for each column $j$ is independent of other columns, and can be done in parallel. 

\subsection{LNQ Algorithm} \label{sec:lnq_alg}



We propose LNQ, an alternating minimization algorithm, which iteratively updates the codebook $\mathbf{c}^{(j)}$ and assignment matrix $\bP^{(j)}$ for each $j$, optimizing one while keeping the other fixed.
Alternating minimization is a common strategy used by most non-uniform quantization methods, including SqueezeLLM and GPTVQ.
LNQ quantizes each layer independently.
We present an overview of LNQ, applied to one layer with weights $\bW \in  \mathbb{R}^{d_\mathrm{in} \times d_\mathrm{out}}$ in \cref{alg:lnq}.



Given fixed assignment matrices $\mathbf{P}^{(j)}$, Problem \eqref{eq:nuq-master} reduces to a standard least-squares problem, which admits a closed-form optimal solution  $\bc^{(j)*} = (\bX \bP^{(j)})^\dagger \bX \bw_j$, where $\dagger$ denotes the Moore–Penrose pseudoinverse.
We assume that the matrix $\bP^{(j)}{}^\top \bH \bP^{(j)}$ is invertible, where recall that $\bH = \bX^\top \bX$.
Under this assumption, the closed-form solution is:
\begin{align}
\label{eq:conti-sol}
\mathbf{c}^{(j)*} = \left(\bP^{(j)}{}^\top \bH \bP^{(j)}\right)^{-1} \bP^{(j)}{}^\top \bH \bw_j.
\end{align}
In practice, $\bP^{(j)}{}^\top \bH \bP^{(j)}$ is not always invertible, even when $\bH$ is invertible (for example if no weight is assigned to a given codebook value $c^{(j)}_q$). To address this, we add a small constant $\lambda = 10^{-7}$ to the diagonal of the matrix, as commonly done in prior work \citep{frantar2022optimal, frantar2022gptq, van2024gptvq}.
In our implementation, we use \texttt{torch.linalg.lstsq} function to compute the least squares solution in Eq. \eqref{eq:conti-sol}, which takes $\bX\bP^{(j)}$ and $\bX\bw_j$ as inputs.
However, since $\bX$ is not explicitly stored, we compute the Cholesky decomposition of $\bH = \bX^\top \bX$, denoted as $\bH = \bL \bL^\top$, and instead provide $\bL^\top\bP^{(j)}$ and $\bL\bw_j$ to the solver.
Because Cholesky decomposition requires $\bH$ to be positive definite, we ensure this by adding a small constant to the diagonal of $\bH$.

For fixed codebooks $\bc^{(j)}$, Problem \eqref{eq:nuq-master} can be equivalently written as 
\begin{equation}\label{eq:nuq-fixedC}
    \underset{\widehat{\bw}_j \in \{c_1^{(j)}, \ldots, c_m^{(j)}\}^{d_\mathrm{in}}}{\mathrm{minimize}}
    \sum_{j=1}^{d_\mathrm{out}} (\widehat{\bw}_j - \bw_j)^\top \bH (\widehat{\bw}_j - \bw_j).
\end{equation}
Even in the special case of uniform codebook, this problem  corresponds to a closest vector problem with box constraints, which is NP-Hard 
to approximate within any constant factor approximation for $m\geq 2$ \citep[Theorem 1]{arora1997hardness}.  
Existing heuristics for solving it include OBQ \cite{frantar2022optimal} which does not scale to LLMs with billions of parameters; its faster variant GPTQ \cite{frantar2022gptq};  LDLQ \cite{chee2024quip}, which is a more efficient implementation of GPTQ; greedy CD  \cite{nair2024cdquant}; and cyclic CD \citep{behdin2023quantease, chee2024quip,egiazarian2024extreme}.
Recent works show that both greedy CD  \cite{nair2024cdquant} and cyclic CD \citep{behdin2023quantease} outperform
GPTQ on this problem when using a uniform grid.
We thus adopt the cyclic CD algorithm, since it performs similarly to the greedy variant while being significantly less expensive \citep[Appendix D]{nair2024cdquant}.
In \Cref{sec:abl_disc}, we present an ablation study that further support this choice, showing that cyclic CD matches or outperforms GPTQ when used within LNQ for non-uniform scalar quantization.


Cyclic CD is an iterative algorithm which iterates over coordinates in a fixed order, minimizing at each iteration the objective with respect to one coordinate, while keeping all others fixed. 
The minimization for each coordinate (Line 8 in \cref{alg:lnq}) has a closed form solution, as shown in  \citet[Lemma 1]{behdin2023quantease} and \citet[Section B.2]{chee2024quip}: \looseness=-1
\begin{equation}\label{eq:cd_closed}
   \hspace{-10pt} \round_j\! \left(\!W_{i,j} - \frac{\bH_{i, [d_\mathrm{in}] \setminus i}}{H_{i,i}} (\widehat{\bW}_{[d_\mathrm{in}] \setminus i, j} - \bW_{[d_\mathrm{in}] \setminus i, j})\!\right),
\end{equation}
where $\round_j(\cdot)$ denotes rounding to the nearest point in the grid $\{c_1^{(j)}, \ldots, c_m^{(j)}\}$.

CD is a descent method when initialized with a feasible solution $\widehat{\bw}_j \in \{c_1^{(j)}, \ldots, c_m^{(j)}\}^{d_\mathrm{in}}$, i.e., it monotonically decreases the objective function value. It can be used as a standalone solver for problem \eqref{eq:nuq-fixedC} initialized with the original weights $\bW$, as in \citet{nair2024cdquant, behdin2023quantease},  or  to refine the output of another quantization method, as done in the uniform quantization method QuIP, which runs CD after LDLQ \citep{chee2024quip}. 

In LNQ,  at each iteration, we initialize CD with the quantized weights corresponding to the current assignment and codebook $\hat{\mathbf{\bw}}_{j} = \bP^{(j)} \bc^{(j)}$ for each $j$. For the first iteration, any feasible assignment matrix can be used. In our experiments, we initialize with the assignments from SqueezeLLM.  
Since the codebooks are updated optimally and CD acts as descent method with feasible initialization, it follows that LNQ itself is a descent method and it converges. Refer to \cref{app:proof} for the proof.

\begin{restatable}{proposition}{lnq} \label{prop:lnq} For any $j \in [d_\mathrm{out}]$, 
let $f_j(\bc, \bP) = \left\lVert  \bX\mathbf{w}_j - \bX\bP\mathbf{c} \right\rVert_2^2$, and
let $\bc_{t}^{(j)}$ and $\bP_{t}^{(j)}$ denote $\bc^{(j)}$ and $\bP^{(j)}$ at the $t$-th iteration of LNQ. Then, $f_j(\bc_{t}^{(j)}, \bP_{t}^{(j)}) \geq f_j(\bc_{{t}+1}^{(j)}, \bP_{t}^{(j)}) \geq f_j(\bc_{{t}+1}^{(j)}, \bP_{{t}+1}^{(j)})$ for all $t$, and the sequence $\{f_j(\bc_{t}^{(j)}, \bP_{t}^{(j)})\}_{t \geq 1}$ converges.
\end{restatable}



Since LNQ is a layer-wise output based method, GuidedQuant can be easily applied to it.
In \cref{sec:exp-wo}, we demonstrate the efficacy of LNQ both as a standalone approach and in combination with the GuidedQuant objective. 

\paragraph{Time Complexity.} 
Computing the Cholesky decomposition of $\bH$ (Line 1) requires $O(d_\mathrm{in}^3)$, optimizing the codebook (Line 4 and 14) requires $O(d_\mathrm{in}^2 m)$, and optimizing the codes (Lines 6-12) requires  $O(d_\mathrm{in}^2 K)$ time complexity.
The total time complexity of LNQ algorithm is $O(d_\mathrm{in}^3 + d_\mathrm{in}^2 d_\mathrm{out} T (m + K))$. Here, $T$ and $K$ denotes the number of iterations for alternating optimization and the number of cycles in coordinate descent, respectively.
We provide a detailed analysis of the time complexity in \cref{app:lnq_detail}. 
We discuss in \cref{sec:CDefficient} how to significantly speedup the implementation of CD on GPU, using precomputation and lazy batch-updates. Precomputation is also used in \citet{behdin2023quantease, chee2024quip}, while lazy batch-updates is only used in \citet{chee2024quip} (though not discussed in the paper). These tricks do no change the theoretical time complexity of CD, but they yield up to $3\times$ speedups in practice.

\section{Experiments}\label{sec:exp}

\begin{table}[t]
\caption{End-to-end inference throughput of Llama-2 models on RTX 4090 GPU. OOM indicates an Out-of-Memory error, meaning the GPU lacks memory to run model inference. See \cref{sec:inferSetup} for experimental setup details.}
\vspace{0.5em}
\centering
\begin{adjustbox}{max width=1.0\columnwidth}
\begin{tabular}{l cc cc cc}
\toprule
& \multicolumn{2}{c}{Llama-2-7B} & \multicolumn{2}{c}{Llama-2-13B} & \multicolumn{2}{c}{Llama-2-70B} \\
\cmidrule(lr){2-3}\cmidrule(lr){4-5}\cmidrule(lr){6-7}
Type &  
Bits$\downarrow$ &  
Tok/s$\uparrow$ & 
Bits$\downarrow$ &  
Tok/s$\uparrow$ &
Bits$\downarrow$ & 
Tok/s$\uparrow$ 
\\
\cmidrule(lr){1-1}\cmidrule(lr){2-3}\cmidrule(lr){4-5}\cmidrule(lr){6-7}
Original 
& 16 & 67
& 16 & OOM
& 16 & OOM
\\
\cmidrule(lr){1-1}\cmidrule(lr){2-3}\cmidrule(lr){4-5}\cmidrule(lr){6-7}
Uniform scalar 
&2.00 & 334
&2.00 & 200
&2.00 & 47
\\
Non-uniform scalar
& 2.01 & 347
& 2.01 & 203
& 2.01 & 47
\\
Vector
& 2.00 & 200
& 2.00 & 121
& 2.00 & 38
\\
\cmidrule(lr){1-1}\cmidrule(lr){2-3}\cmidrule(lr){4-5}\cmidrule(lr){6-7}
Uniform scalar
& 3.00 & 260
& 3.00 & 150
& 3.00 & OOM
\\
Non-uniform scalar
& 3.03 & 264
& 3.02 & 148
& 3.01 & OOM
\\
Vector
& 3.00 & 176
& 3.00 & 103
& 3.00 & OOM
\\
\cmidrule(lr){1-1}\cmidrule(lr){2-3}\cmidrule(lr){4-5}\cmidrule(lr){6-7}
Uniform scalar
& 4.00 & 214
& 4.00 & 121
& 4.00 & OOM
\\
Non-uniform scalar
& 4.05 & 209
& 4.04 & 116
& 4.03 & OOM
\\
Vector
& 4.00 & 151
& 4.00 & 89
& 4.00 & OOM
\\
\bottomrule
\end{tabular}
\end{adjustbox}
\label{tab:speedup}
\end{table}

\begin{table*}[t]
\caption{Weight-only scalar post-training quantization results \textit{without fine-tuning} with end-to-end loss. Wiki2 and C4 denotes perplexity on WikiText2 and C4, respectively. The perplexity is measured with the context size of 4096.}
\vspace{0.5em}
\centering
\begin{adjustbox}{max width=1.69\columnwidth}
\begin{tabular}{l ccc ccc ccc}
\toprule
& \multicolumn{3}{c}{Llama-2-7B} & \multicolumn{3}{c}{Llama-2-13B} & \multicolumn{3}{c}{Llama-2-70B} \\
\cmidrule(lr){2-4}\cmidrule(lr){5-7}\cmidrule(lr){8-10}
Method & 
Bits$\downarrow$ & 
Wiki2$\downarrow$ & 
C4$\downarrow$ & 
Bits$\downarrow$ & 
Wiki2$\downarrow$ & 
C4$\downarrow$ &
Bits$\downarrow$ & 
Wiki2$\downarrow$ &
C4$\downarrow$ 
\\
\cmidrule(lr){1-1}\cmidrule(lr){2-2}\cmidrule{3-4}\cmidrule(lr){5-5}\cmidrule{6-7}\cmidrule(lr){8-8}\cmidrule{9-10}
Original & 
16 & 5.12 & 6.63 &  
16 & 4.57 & 6.05 &  
16 & 3.12 & 4.97 \\
\cmidrule(lr){1-1}\cmidrule(lr){2-2}\cmidrule{3-4}\cmidrule(lr){5-5}\cmidrule{6-7}\cmidrule(lr){8-8}\cmidrule{9-10}
QuIP & 
-- & -- & -- & 
2.00 & 13.48 & 16.16 &
2.01 & 5.90 & 8.17 
\\
SqueezeLLM &  
2.01 & 39.58 & 44.05 & 
2.01 & 16.24 & 19.20 &  
2.01 & 9.17 & 13.03 \\
GPTVQ 1D  & 
2.03 & 51.87 & 47.33 & 
2.03 & 9.53 & 12.62 & 
2.03 & 6.03 & 8.44 \\
LNQ  (Ours) & 
2.01 & 23.31 & 26.71 &  
2.01 & 8.78 & 11.80 &
2.01 & 5.23 & 7.31 \\
LNQ + GuidedQuant (Ours) & 
2.01 & \textbf{8.83} & \textbf{11.15} & 
2.01 & \textbf{7.26} & \textbf{9.17} & 
2.01 & \textbf{5.04} & \textbf{7.04} \\
\cmidrule(lr){1-1}\cmidrule(lr){2-2}\cmidrule{3-4}\cmidrule(lr){5-5}\cmidrule{6-7}\cmidrule(lr){8-8}\cmidrule{9-10}
GPTQ & 
3.00 & 8.06 & 10.61 &  
3.00 & 5.85 &  7.86 &  
3.00 & 4.40 &  6.26 \\
QuIP  & 
-- & -- & -- & 
3.00 & 5.12 &  6.79 &  
3.01 & 3.87 &  5.67 \\
SqueezeLLM  &  
3.03 & 5.74 & 7.44 &  
3.02 & 4.99 & 6.60 &  
3.01 & 3.53 & 5.31 \\
GPTVQ 1D  &  
3.03 & 6.17 & 8.02 &  
3.03 & 5.13 & 6.76 &  
3.03 & 3.55 & 5.35
\\
LNQ  (Ours) & 
3.03 & 5.89 & 7.74 &  
3.02 & 5.02 & 6.68 &  
3.01 & 3.50 & 5.31 \\
LNQ + GuidedQuant (Ours) & 
3.03 & \textbf{5.57} & \textbf{7.22} &  
3.02 & \textbf{4.91} & \textbf{6.49} &  
3.01 & \textbf{3.47} & \textbf{5.27} \\
\cmidrule(lr){1-1}\cmidrule(lr){2-2}\cmidrule{3-4}\cmidrule(lr){5-5}\cmidrule{6-7}\cmidrule(lr){8-8}\cmidrule{9-10}
GPTQ & 
4.00 & 5.49 & 7.20 &  
4.00 & 4.78 & 6.34 &  
4.00 & 3.35 & 5.15 \\
QuIP & 
 -- & -- & -- & 
 4.00 & 4.76 & 6.29 &  
 4.00 & 3.58 &  5.38 \\
SqueezeLLM & 
4.05 & 5.23 & 6.78 &  
4.04 & 4.67 & 6.15 &  
4.03 & \textbf{3.20} & 5.04 \\
GPTVQ 1D &  
4.06 & 5.27 & 6.83 &  
4.06 & 4.67 & 6.17 &  
4.03 & \textbf{3.20} & 5.04
\\
LNQ  (Ours) & 
4.05  & 5.26 & 6.82 &  
4.04  & 4.67 & 6.17 &  
4.03 & \textbf{3.20} & 5.04 \\
LNQ + GuidedQuant (Ours) & 
4.05 & \textbf{5.21} & \textbf{6.75} &  
4.04 & \textbf{4.65} & \textbf{6.14} &  
4.03 & \textbf{3.20} & \textbf{5.03} \\
\bottomrule
\end{tabular}
\end{adjustbox}
\label{tab:scalar}
\end{table*}

\begin{table*}[t]
\caption{Weight-only vector post-training quantization results \textit{without fine-tuning} to the end-to-end loss. Wiki2 and C4 denotes perplexity on WikiText2 and C4, respectively. The perplexity is measured with the context size of 4096.}
\vspace{0.5em}
\centering
\begin{adjustbox}{max width=1.69\columnwidth}
\begin{tabular}{l ccc ccc ccc}
\toprule
& \multicolumn{3}{c}{Llama-2-7B} & \multicolumn{3}{c}{Llama-2-13B} & \multicolumn{3}{c}{Llama-2-70B} \\
\cmidrule(lr){2-4}\cmidrule(lr){5-7}\cmidrule(lr){8-10}
Method & 
Bits$\downarrow$ & 
Wiki2$\downarrow$ & 
C4$\downarrow$ & 
Bits$\downarrow$ & 
Wiki2$\downarrow$ & 
C4$\downarrow$ & 
Bits$\downarrow$ & 
Wiki2$\downarrow$ & 
C4$\downarrow$  \\
\cmidrule(lr){1-1}\cmidrule(lr){2-2}\cmidrule{3-4}\cmidrule(lr){5-5}\cmidrule{6-7}\cmidrule(lr){8-8}\cmidrule{9-10}
Original & 
16 & 5.12 & 6.63 &  
16 & 4.57 & 6.05 &  
16 & 3.12 & 4.97 \\
\cmidrule(lr){1-1}\cmidrule(lr){2-2}\cmidrule{3-4}\cmidrule(lr){5-5}\cmidrule{6-7}\cmidrule(lr){8-8}\cmidrule{9-10}
GPTVQ 2D&
2.13 & 10.66 & 12.81 & 
2.13 & 7.55 & 9.82 & 
2.13 & 5.06 & 7.09 \\
GPTVQ 4D&
2.25 & 7.89 & 10.25 & 
2.25 & 6.36 & 8.43 & 
2.25 & 4.44 & 6.28 \\
QuIP\# &
2.00 & 8.22 & 11.01 &  
2.00 & 6.06 & 8.07 &  
2.00 & 4.16 & 6.01 \\
AQLM & 
2.02 & 6.59 & 8.54 &  
2.19 & 5.37 & 7.16 &  
2.07 & 3.94 & 5.72 \\
QTIP & 
2.00 & 6.82 & 8.96 &  
2.00 & 5.52 & 7.39 &  
2.00 & 3.87 & 5.69 \\
QTIP + GuidedQuant (Ours) & 
2.00 & \textbf{6.11} & \textbf{7.99} &  
2.00 & \textbf{5.33} & \textbf{7.05} &  
2.00 & \textbf{3.80} & \textbf{5.61} \\
\cmidrule(lr){1-1}\cmidrule(lr){2-2}\cmidrule{3-4}\cmidrule(lr){5-5}\cmidrule{6-7}\cmidrule(lr){8-8}\cmidrule{9-10}
GPTVQ 2D&
3.13 & 5.63 & 7.32 & 
3.13 & 4.87 & 6.45 & 
3.13 & 3.38 & 5.18  \\
QuIP\# & 
3.00 & 5.60 & 7.34 &  
3.00 & 4.90 & 6.50 &  
3.00 & 3.41 & 5.20 \\
AQLM & 
3.04 & 5.46 & 7.08 &  
3.03 & 4.82 & 6.37 &  
3.01 & 3.36 & 5.17 \\
QTIP & 
3.00 & 5.38 & 6.99 &  
3.00 & 4.74 & 6.28 &  
3.00 & 3.27 & 5.09 \\
QTIP + GuidedQuant (Ours) & 
3.00 & \textbf{5.28} & \textbf{6.87} &  
3.00 & \textbf{4.71} & \textbf{6.22} &  
3.00 & \textbf{3.25} & \textbf{5.08} \\
\cmidrule(lr){1-1}\cmidrule(lr){2-2}\cmidrule{3-4}\cmidrule(lr){5-5}\cmidrule{6-7}\cmidrule(lr){8-8}\cmidrule{9-10}
GPTVQ 2D&
4.13 & 5.24 & 6.77 & 
4.13 & 4.65 & 6.13 & 
4.13 & 3.18 & 5.01 \\
QuIP\# & 
4.00 & 5.22 & 6.79 &  
4.00 & 4.65 & 6.15 &  
4.00 & 3.18 & 5.02 \\
AQLM & 
4.04 & 5.21 & 6.75 &  
3.94 & 4.65 & 6.14 &  
4.14 & 3.19 & 5.03 \\
QTIP & 
4.00 & 5.17 & 6.71 &  
4.00 & 4.62 & 6.10 &  
4.00 & 3.16 & \textbf{5.00} \\
QTIP + GuidedQuant (Ours) & 
4.00 & \textbf{5.16} & \textbf{6.68} &  
4.00 & \textbf{4.61} & \textbf{6.09} &  
4.00 & \textbf{3.15} & \textbf{5.00} \\
\bottomrule
\end{tabular}
\end{adjustbox}
\label{tab:vector}
\end{table*}


In this section, we demonstrate the versatility and effectiveness of our method across various quantization schemes.
We first explore different quantization scenarios and identify the formats best suited to each setting, ultimately focusing on three main approaches: weight-only scalar, weight-only vector, and weight-and-activation quantization.
By integrating the GuidedQuant objective into existing methods, our results consistently achieve state-of-the-art PTQ performance. 
Refer to \cref{sec:gq_diff_quant} for details on how we incorporate GuidedQuant objective into existing methods.
Additional experiments and details, including the overall cost of our method, the effect of the number of groups $g$, and the end-to-end fine-tuning results, are provided in \cref{sec:add}.

\subsection{Weight-only Quantization}

\textbf{Experimental Setup.}
Weight-only quantization primarily accelerates inference latency in low-batch scenarios, where memory bandwidth constitutes the main bottleneck \citep{gholami2024ai}. Among weight-only techniques, three quantization formats are commonly used: \emph{uniform scalar}, \emph{non-uniform scalar}, and \emph{vector} quantization \citep{frantar2022gptq,kim2023squeezellm,tseng2024qtip}. With fixed bit-width constraints, non-uniform scalar quantization generally outperforms uniform scalar quantization, as its search space encompasses that of uniform scalar quantization. Meanwhile, vector quantization can outperform non-uniform scalar quantization by exploiting additional redundancies across weight dimensions.

Despite this, non-uniform scalar quantization offers advantages in inference latency. \Cref{tab:speedup} compares end-to-end single-batch inference latency across these formats using the state-of-the-art GPU kernels: LUT-GEMM \citep{park2022lut} for uniform scalar, Any-Precision-LLM \citep{park2024any} for non-uniform scalar, and QTIP \citep{tseng2024qtip} for vector quantization. Results show that vector quantization incurs higher latency due to its decoding overhead \citep{tseng2024qtip}, whereas uniform and non-uniform scalar quantization have similar latency with minimal decoding overhead. Consequently, non-uniform scalar and vector quantization remain the primary formats of interest for weight-only quantization. In this context, we apply our GuidedQuant to both formats, achieving state-of-the-art performance in each.

For our experiments, we demonstrate the effectiveness of our method on the Llama-2 model family \citep{touvron2023llama}, evaluating on 7B, 13B and 70B model. We use the RedPajama dataset \citep{together2023redpajama} for calibration, following prior work \citep{egiazarian2024extreme, tseng2024quip, tseng2024qtip}, with 1024 sentences, each containing 4096 tokens. We report perplexity on the WikiText2 \citep{wikitext103} and C4 \citep{c4} validation sets.

\textbf{Scalar Post-training Quantization Results.}
We summarize the results of weight-only scalar quantization in \cref{tab:scalar}, comparing our approach with GPTQ \citep{frantar2022gptq}, SqueezeLLM without mixed precision \citep{kim2023squeezellm}, QuIP \citep{chee2024quip}, and GPTVQ 1D \citep{van2024gptvq}. For GPTQ and QuIP, we report the results from \citet{egiazarian2024extreme}, which used the same or a larger calibration dataset, while for GPTVQ 1D, we reproduce the results with the same calibration data while adjusting the group size to align with the average bit-width for a fair comparison (see \cref{app:gptvq_detail} for details).

We evaluate the performance of LNQ both with and without the GuidedQuant objective. Notably, LNQ combined with GuidedQuant consistently outperforms all baselines across various bit-widths and model sizes. Additionally, LNQ with the layer-wise reconstruction objective surpasses GPTVQ 1D in all settings, demonstrating that our approach improves upon GPTVQ 1D by addressing its suboptimal optimization.

\textbf{Vector Post-training Quantization Results.}
For vector post-training quantization (PTQ), we present the results in \cref{tab:vector}. We apply GuidedQuant to the state-of-the-art vector PTQ baseline, QTIP \citep{tseng2024qtip}. We implement it on both the 1MAD and 3INST variants and report the variant that performs better among these two.
Refer to \cref{sec:qtip_var} for results on different variants.
We compare our approach with the following baselines: GPTVQ \citep{van2024gptvq}, QuIP\# \citep{tseng2024quip}, AQLM \citep{egiazarian2024extreme}, and QTIP \citep{tseng2024qtip}. For QuIP\#, AQLM, and QTIP, we report the results from their respective papers, as they used the same or larger calibration datasets than ours. For GPTVQ, we report the reproduced results using our calibration data.
Our method consistently outperforms all vector quantization baselines across different bit-widths and model sizes as well.

\label{sec:exp-wo}

\subsection{Weight-and-activation Quantization}

Weight-and-activation quantization methods apply uniform quantization on both weights and activations to leverage the faster matrix multiplication units in the hardware \citep{ashkboos2024quarot,liu2024spinquant}.
State-of-the-art methods for weight-and-activation quantization include QuaRot \citep{ashkboos2024quarot} and SpinQuant \citep{liu2024spinquant}, which use rotation matrices to reduce the activation outliers before applying the uniform quantization. We incorporate our GuidedQuant objective into the weight quantization process of these methods, guiding the model to quantize the weights  more accurately. Specifically, we implement GuidedQuant on top of the SpinQuant using GPTQ weight quantizer and present the results in \cref{tab:wa}.
Following prior work, we use the WikiText2 dataset \citep{wikitext103} for calibration, with 128 sentences, each containing 2048 tokens \citep{ashkboos2024quarot, liu2024spinquant}.
Our objective consistently improves the perplexity compared to the baseline methods, demonstrating its effectiveness.

\begin{table}[t]
\caption{Weight-and-activation quantization results on Llama-2 models. L-2-7B, L-2-13B and L-2-70B denote Llama-2-7B, Llama-2-13B, and Llama-2-70B model, respectively. Wiki2 denotes perplexity on Wikitext2 with the context size of 2048.}
\vspace{0.5em}
\centering
\begin{adjustbox}{max width=1.0\columnwidth}
\begin{tabular}{c l c c c}
\toprule
& & \multicolumn{1}{c}{L-2-7B} & \multicolumn{1}{c}{L-2-13B} & \multicolumn{1}{c}{L-2-70B} \\
\cmidrule(lr){3-3}\cmidrule(lr){4-4}\cmidrule(lr){5-5}
Bits & Method &  
Wiki2$\downarrow$ &  
Wiki2$\downarrow$ &
Wiki2$\downarrow$ 
\\
\cmidrule(lr){1-2}\cmidrule(lr){3-3}\cmidrule(lr){4-4}\cmidrule(lr){5-5}
16& Original & 5.47 & 4.88 & 3.32 \\
\cmidrule(lr){1-2}\cmidrule(lr){3-3}\cmidrule(lr){4-4}\cmidrule(lr){5-5}
\multirow{3}{*}{\shortstack[c]{W4A4KV4}}
&QuaRot & 6.08 & 5.39 & 3.80
\\
&SpinQuant & 5.95 & 5.24 & \textbf{3.71}
\\
&SpinQuant + GQuant (Ours) & \textbf{5.89} & \textbf{5.19} & \textbf{3.71} 
\\
\cmidrule(lr){1-2}\cmidrule(lr){3-3}\cmidrule(lr){4-4}\cmidrule(lr){5-5}
\multirow{3}{*}{\shortstack[c]{W4A4KV16}}
&QuaRot & 6.02 & 5.34 & 3.77
\\
&SpinQuant & 5.90 & 5.22 & \textbf{3.68}
\\
&SpinQuant + GQuant (Ours) & \textbf{5.84} & \textbf{5.17} & \textbf{3.68}
\\
\bottomrule
\end{tabular}
\end{adjustbox}
\label{tab:wa}
\end{table}

\section{Additional Related Work}

There’s a large body of work on neural network compression, even when considering only quantization for LLMs, making a complete overview infeasible. Instead, we focus here on the works most related to ours. 
\paragraph{Hessian-based Compression}
Neural networks compression based on the second-order Taylor approximation of the end loss (Eq \eqref{eq:second-order}) dates back to the early works of \citet{lecun1989optimal} and \citet{hassibi1992second}. 
OBD \citep{lecun1989optimal} introduced this approach for pruning, under the assumption that the Hessian matrix is diagonal. 
OBS \citep{hassibi1992second} improved upon this by 
dropping the diagonal assumption and instead approximating the Hessian by the empirical Fisher information matrix. However, applying OBS to large neural networks remains computationally intractable. To address this, various more efficient Hessian approximations have been proposed, including the K-FAC approximation \citep{martens2015optimizing, zeng2018mlprune, wang2019eigendamage,ouderaa2024the}, block-diagonal Fisher approximation \citep{singh2020woodfisher, kurtic2022optimal, li2021brecq}, and diagonal Fisher approximation \citep{choi2016towards, theis2018faster, kim2023squeezellm, bai2024skim}. Other strategies 
directly estimate inverse-Hessian vector products \citep{frantar2021m}. 
The most similar approaches to GuidedQuant are ones that employ block-diagonal Fisher approximation, which achieve a good trade-off between approximation accuracy and computation and storage cost. 
However, these methods remain intractable at the scale of modern LLMs (see \cref{app:blk_hess}). 

\paragraph{Gradient-based Compression} Various compression methods are based on a first-order Taylor approximation of the end loss, with respect to output feature maps or gates applied to them \cite{molchanov2017pruning,molchanov2019importance, you2019gate}, or weights \cite{ding2019global}. The one most similar to GuidedQuant is \cite{molchanov2019importance}, which employs the same criterion in Eq. \eqref{eq:ours} to prune filters and neurons in vision models. However, as explained in \cref{sec:challenge}, adopting this criterion for quantizing modern LLMs is infeasible, without the averaging approximation we propose.  

\paragraph{Non-uniform Scalar PTQ for LLMs}
PTQ encompasses a vast array of work, so we focus on non-uniform scalar PTQ methods for LLMs that use look-up tables (codebooks) for weight decoding, which are closely related to our LNQ algorithm. One approach is zero-shot quantization, which requires no calibration data: Dynamic Tree Quantization \citep{dettmers20218} defines a new data type with dynamic exponential bits and stores decoded values in the codebook; Quantile Quantization \citep{dettmers2023case} saves quantile values of the weight distribution; and QLoRA \citep{dettmers2023qlora} introduces the NF4 data type using quantiles of a standard normal distribution. These methods share a global codebook, with each layer maintaining its own scale parameters. HIGGS \citep{malinovskii2024pushing} further refines this by adopting MSE-optimal grids for the standard normal distribution and applying rotation matrices to approximate Gaussian weight distributions. Another line of work involves one-shot quantization methods that optimize the output quantization error using calibration data. For instance, SqueezeLLM \citep{kim2023squeezellm} optimizes separate channel-wise codebooks via the $k$-means algorithm, while a 1D variant of GPTVQ \citep{van2024gptvq} alternates between optimizing assignments with the GPTQ algorithm and refining codebooks with gradient descent. The GPTVQ 1D shows the strongest performance among this line of research.
Although not a scalar PTQ method, the vector quantization variant of AQLM \citep{egiazarian2024extreme} also follows a similar paradigm, optimizing assignments through CD and codebooks via gradient descent. \looseness=-1

\section{Conclusion}

We introduced GuidedQuant, a novel PTQ approach that integrates gradient information from the end loss while preserving cross-weight dependencies within output channels. GuidedQuant improves state-of-the-art methods  across quantization formats, including weight-only scalar, weight-only vector, and weight-and-activation quantization. Furthermore, we identified inefficiencies in the current state-of-the-art methods for non-uniform scalar quantization and proposed LNQ, a new algorithm that, when combined with GuidedQuant, improves over the state-of-the-art performance. These contributions advance the efficiency and accuracy of quantization for modern LLMs.

\section*{Impact Statement}
This work advances the compression of LLMs, in particular via post-training quantization. As discussed, quantization, and model compression more broadly, reduces the memory and computational requirements of LLMs and speeds up inference, thus reducing their environmental impact and enabling their use on resource-constrained devices and for latency-critical applications. This can also help democratize access to these models for organizations with limited resources and support privacy-preserving, offline applications.
On the other hand, compression methods, including quantization, can adversely affect fairness in language models \citep{ramesh2023comparative}. While there are ongoing efforts to address fairness concerns in pruned LLMs \citep{zayed2024fairness}, extending these mitigation strategies to quantized models remains an important direction for future research. Furthermore, reducing the cost of using LLMs can also lower the barrier to their use by malicious actors. Finally, the energy and resources saved through compression might be reinvested elsewhere, so the net reduction in environmental harm is not guaranteed (Jevons paradox \citep{alcott2005jevons}).

\section*{Acknowledgements}
This work was supported by Samsung Electronics Co., Ltd. (IO250418-12669-01), 
Mobile eXperience (MX) Business, Samsung Electronics Co., Ltd., 
Institute of Information \& Communications Technology Planning \& Evaluation (IITP) grant funded by the Korea government (MSIT) [No. RS-2020-II200882, (SW STAR LAB) Development of deployable learning intelligence via self-sustainable and trustworthy machine learning, No. RS-2021-II211343, Artificial Intelligence Graduate School Program (Seoul National University), and No. 2022-0-00480, RS-2022-II220480, Development of Training and Inference Methods for Goal-Oriented Artificial Intelligence Agents], 
and Basic Science Research Program through the National Research Foundation of Korea (NRF) funded by the Ministry of Education (RS-2023-00274280). 
Hyun Oh Song is the corresponding author.

\bibliography{main}
\bibliographystyle{icml2025}

\newpage
\appendix
\onecolumn
\section{Proofs}
\label{app:proof}
Here, we prove \cref{prop:chainrule} and \cref{prop:lnq}, each restated here for convenience.

\chainrule*
\begin{proof}
Recall that $\bZ^{(l)} = \bX^{(l)} {\bW}^{(l)}$. Then, by chain rule we have that $\frac{\partial \ell_i}{\partial \bw^{(l)}_j} = \frac{\partial \ell_i}{\partial Z_{ij}^{(l)}} (\bX^{(l)}_{i,:})^\top$. Note also that $\frac{\partial \ell}{\partial Z_{ij}^{(l)}} = \frac{\partial \ell_i}{\partial Z_{ij}^{(l)}}$. Hence, 
\begin{align*}
    \left\lVert \frac{\partial \ell}{\partial \bZ^{(l)}} \odot (\bX^{(l)} {\bW}^{(l)} - \bX^{(l)} \widehat{\bW}^{(l)}) \right\rVert_F^2 &=  \sum_{j=1}^{d_\mathrm{out}^{(l)}} \sum_{i=1}^n \left(\frac{\partial \ell_i}{\partial Z_{ij}^{(l)}}\bX^{(l)}_{i,:} (\bw^{(l)}_j - \widehat{\bw}^{(l)}_j)\right)^2 \\
    &=  \sum_{j=1}^{d_\mathrm{out}^{(l)}} \sum_{i=1}^n \left(\left(\frac{\partial \ell_i}{\partial \bw^{(l)}_j}\right)^\top(\bw^{(l)}_j - \widehat{\bw}^{(l)}_j) \right)^2\\
    &= n \sum_{j=1}^{d_\mathrm{out}^{(l)}} (\bw^{(l)}_j - \widehat{\bw}^{(l)}_j)^\top \bF^{(l)}_j (\bw^{(l)}_j - \widehat{\bw}^{(l)}_j),
\end{align*}
where the last equality follows from the definition of the Fisher blocks $\bF^{(l)}_j= \tfrac{1}{n} \sum_{i=1}^n (\frac{\partial \ell_i}{\partial \bw^{(l)}_j})(\frac{\partial \ell_i}{\partial \bw^{(l)}_j})^\top$.
Taking the sum over $l \in [L]$ on both sides yields the claim. 
\end{proof}

\lnq*

\begin{proof} We first show that the objective value is non-increasing in LNQ. For all $t \geq 1$, we have $\bP_{t}^{(j)}\mathbf{1}_m = \mathbf{1}_{d_\mathrm{in}}$ and thus the corresponding quantized weights $\hat{\mathbf{\bw}}_{j} = \bP^{(j)}_{t} \bc^{(j)}_{t}$ are feasible 
Hence, CD is initialized with a feasible solution at each iteration $t$, so it acts as a descent method. Then, 
\begin{align*} 
f_j(\bc_{t}^{(j)}, \bP_{t}^{(j)}) &\geq f_j(\bc_{{t}+1}^{(j)}, \bP_{t}^{(j)}) &&\text{(since $\bc_{{t}+1}^{(j)} = \underset{\bc^{(j)} \in \mathbb{R}^m}{\mathrm{argmin}}\;\;f_j(\bc, \bP_{t})$)} \\ 
&\geq f_j(\bc_{{t}+1}^{(j)}, \bP_{{t}+1}^{(j)}), &&\text{(since CD does not increase the objective value)} 
\end{align*} 
for all $t \geq 1$.
Since $f_j(\bc,\bP)$ is bounded below by $0$, the sequence $\{f_j(\bc_{t}^{(j)}, \bP_{t}^{(j)})\}$ is monotonically non-increasing and bounded below. Hence, it converges to its infimum by the monotone convergence theorem. 
\end{proof}

\section{Hyperparameters and Details}
\label{sec:hyperparam}

In this section, we clarify the hyperparameters and details of the methods discussed in the main paper.

\subsection{GuidedQuant} \label{app:gq_detail}

The proposed GuidedQuant method has a single hyperparameter: the number of group $g$ used to average the Hessian matrices $\overbar{\bH}_j$ (see \cref{sec:challenge}). 
For weight-only quantization experiments, we set $g=4$ for Llama-2-7B and Llama-2-13B, and $g=2$ for Llama-2-70B. For weight-and-activation quantization experiments, we set $g=1$.
For the hyperparameter $g$, we selected the number of groups to be as large as possible within the limits of our computational and memory constraints. Notably, GuidedQuant also maintains strong performance with smaller values of $g$ (see \cref{sec:group}).

Computing the Hessian (Line 4 in \cref{alg:gq}) and running the quantization algorithm $\mathcal{Q}$ (Line 5 in \cref{alg:gq}) for each group and layer can be parallelized.
We parallelize Hessian computation across groups. For quantization, we parallelize across groups in LNQ + GuidedQuant, while in QTIP + GuidedQuant and SpinQuant + GuidedQuant, we run this step in a sequential manner to minimally change the codebase of the original methods.

\subsection{LNQ} \label{app:lnq_detail}

The proposed LNQ method has two hyperparameters: (1) the number of iterations during which we alternate between optimizing $\bc$ and $\bP$ ($T$ in \cref{alg:lnq}), and (2) the number of coordinate descent iterations over the output dimensions ($K$ in \cref{alg:lnq}). For Llama-2-7B and Llama-2-13B, we use $T=2$ and $K=4$, and for Llama-2-70B, we use $T=1$ and $K=4$ in all the experiments.

We further explain a derivation of the time complexity of the proposed LNQ algorithm (\cref{alg:lnq}), discussed in \cref{sec:lnq_alg}. First, the time complexity of the Cholesky decomposition for a matrix $\bH \in \mathbb{R}^{d_\mathrm{in} \times d_\mathrm{in}}$ is $O(d_\mathrm{in}^3)$ (Line 1).  

For optimizing the codebook (Line 4 and 14), we analyze the computational cost within the loop as follows:  
\vspace{-1em}
\begin{itemize} \itemsep=0pt
    \item Computing $\bL^\top \bP^{(j)}$ requires $O(d_\mathrm{in}^2 m)$ time.
    \item Computing $\bL^\top \bw_j$ requires $O(d_\mathrm{in}^2)$ time.
    \item \texttt{torch.linalg.lstsq} function uses QR decomposition of $\bL^\top \bP^{(j)}$ to compute least squares solution, which requires $O(d_\mathrm{in} m^2)$ time.
\end{itemize}
\vspace{-1em}
Since $d_\mathrm{in} \gg m$, the dominant cost is $O(d_\mathrm{in}^2 m)$. 

For computing $\hat{\bw}_j = \bP^{(j)} \bc^{(j)}$ (Line 5), the cost is $O(d_\mathrm{in}m)$.


In CD, the cost of the minimizing the objective for each coordinate $i$ (Line 8, Eq. \eqref{eq:cd_closed}) is $O(d_\mathrm{in} + m)$. Since $d_\mathrm{in} \gg m$, the dominant cost is $O(d_\mathrm{in})$. 
Considering loop iterations, optimizing the code (Lines 6-12) takes $O(d_\mathrm{in}^2 K)$ time complexity.

Therefore, the cost of Lines 4-12 is $O(d_\mathrm{in}^2 (m + K))$, and the cost of Lines 2-15 is $O(d_\mathrm{in}^2 d_\mathrm{out} T (m + K))$.
Including the Cholesky decomposition, the total time complexity of LNQ algorithm is $O(d_\mathrm{in}^3 + d_\mathrm{in}^2 d_\mathrm{out} T (m + K))$.


    \begin{algorithm}[t]
    \caption{Efficient CD algorithm with precomputation}
    \label{alg:cd_v1}
    \begin{algorithmic}[1]
    \INPUT Hessian of the objective $\bH \in \mathbb{R}^{d_\mathrm{in}\times d_\mathrm{in}}$, input weight $\bW \in  \mathbb{R}^{d_\mathrm{in} \times d_\mathrm{out}}$, current codebook $\bc^{(j)}\in \mathbb{R}^m$ and current quantized weight $\hbW \in  \mathbb{R}^{d_\mathrm{in} \times d_\mathrm{out}}$. Initialize $\mathbf{Q} \in \{1, \ldots, m\}^{d_\mathrm{in} \times d_\mathrm{out}}$ (rounded indices).\\
    \vspace{1em}
    \STATE $\mathbf{\widetilde{H}} \gets \diag(\bH)^{-1} \bH$ , $\bU \gets \mathrm{StrictUpper}(\mathbf{\widetilde{H}})$.
    \vspace{0.5em}
    \FOR{$k = 1$ {\bf to} $K$}
    \STATE $\bB \gets \bU (\hbW - \bW)$
    \FOR{$i = 1$ {\bf to} $d_\mathrm{in}$}
    \STATE $\hbW_{i, :} \leftarrow \round(\bW_{i, :} - \bB_{i, :})$, $\mathbf{Q}_{i, :} \gets \mathrm{RoundIdx}(\bW_{i, :} - \bB_{i, :})$
    \STATE $\bB_{(i+1):, :} \gets \bB_{(i+1):, :} + \bU_{(i+1):, i} (\hbW_{i,:} - \bW_{i,:})$
    \ENDFOR
    \ENDFOR
    \STATE $\forall i \in [d_\mathrm{in}], j \in [d_\mathrm{out}], q\in [m]: P^{(j)}_{iq} =
            \begin{cases} 
            1 & \text{if } q = Q_{ij}, \\
            0 & \text{otherwise.}
            \end{cases}$ \hfill \COMMENT{Extracting assignment matrix}
    \OUTPUT $\bP^{(1)}, \ldots, \bP^{(d_\mathrm{out})}$.
    \end{algorithmic}
    \end{algorithm}
    
    \begin{algorithm}[t]
    \caption{Efficient CD algorithm with precomputation and lazy batch-updates}
    \label{alg:cd_v2}
    \begin{algorithmic}[1]
    \INPUT Hessian of the objective $\bH \in \mathbb{R}^{d_\mathrm{in}\times d_\mathrm{in}}$, input weight $\bW \in  \mathbb{R}^{d_\mathrm{in} \times d_\mathrm{out}}$, current codebook $\bc^{(j)}\in \mathbb{R}^m$ and current quantized weight $\hbW \in  \mathbb{R}^{d_\mathrm{in} \times d_\mathrm{out}}$. Initialize $\mathbf{Q} \in \{1, \ldots, m\}^{d_\mathrm{in} \times d_\mathrm{out}}$ (rounded indices).\\
    \vspace{1em}
    \STATE $\mathbf{\widetilde{H}} \gets \diag(\bH)^{-1} \bH$ , $\bU \gets \mathrm{StrictUpper}(\mathbf{\widetilde{H}})$.
    \vspace{0.5em}
    \FOR{$k = 1$ {\bf to} $K$}
    \STATE $\bB \gets \bU (\hbW - \bW)$
    \FOR{$s = 1,~ b+1, ~2b+1, \ldots, ~ d_\mathrm{in} - b + 1$ }
    \FOR{$i = s$ {\bf to} $s + b - 1$}
    \STATE $\hbW_{i, :} \leftarrow \round(\bW_{i, :} - \bB_{i, :})$, $\mathbf{Q}_{i, :} \gets \mathrm{RoundIdx}(\bW_{i, :} - \bB_{i, :})$
    \STATE $\bB_{(i+1):(s+b), :} \gets \bB_{(i+1):(s+b), :} + \bU_{(i+1):(s+b), i} (\hbW_{i,:} - \bW_{i,:})$
    \ENDFOR
    \STATE $\bB_{(s+b):, :} \gets \bB_{(s+b):, :} + \bU_{(s+b):, s:(s+b)} (\hbW_{s:(s+b),:} - \bW_{s:(s+b),:})$
    \ENDFOR
    \ENDFOR
    \STATE $\forall i \in [d_\mathrm{in}], j \in [d_\mathrm{out}], q\in [m]: P^{(j)}_{iq} =
            \begin{cases} 
            1 & \text{if } q = Q_{ij}, \\
            0 & \text{otherwise.}
            \end{cases}$ \hfill \COMMENT{Extracting assignment matrix}
    \OUTPUT $\bP^{(1)}, \ldots, \bP^{(d_\mathrm{out})}$.
    \end{algorithmic}
    \end{algorithm}

\subsection{Efficient Implementation of CD Algorithm in LNQ}\label{sec:CDefficient}
In the LNQ algorithm (\cref{alg:lnq}), computing the solution across all output channels $j \in [d_\mathrm{out}]$ is independent and thus fully parallelizable. Therefore, we perform the coordinate descent (CD) updates for each output channel in parallel.

\paragraph{Coordinate-wise Closed-form Solution.} For a given quantized weight matrix $\hbW \in \mathbb{R}^{d_\mathrm{in} \times d_\mathrm{out}}$, the CD update for the $i$-th input coordinate can be computed in parallel using the coordinate-wise closed-form solution as follows \citep[Lemma 1]{behdin2023quantease}:
\begin{align} \hbW_{i,:} \gets \round\left(\bW_{i,:} - \frac{\bH_{i, [d_\mathrm{in}] \setminus i}}{H_{i,i}} \left(\hbW_{[d_\mathrm{in}] \setminus i, :} - \bW_{[d_\mathrm{in}] \setminus i, :}\right)\right), \label{eq:cd_closed_par} \end{align}
where $\round(\cdot): \mathbb{R}^{1 \times d_\mathrm{out}} \rightarrow \mathbb{R}^{1 \times d_\mathrm{out}}$ rounds $j$-th element to the nearest point in the grid $\{c_1^{(j)}, \ldots, c_m^{(j)}\}$.
We adopt this coordinate-wise closed-form solution within the CD loop.

\paragraph{Precomputation Trick.}

On GPUs, the coordinate-wise CD update in Eq.~\eqref{eq:cd_closed_par} can be accelerated by precomputing parts of the update that remain unchanged during previous coordinate updates.
Specifically, when updating the $i$-th coordinate, the components of $\hbW$ corresponding to coordinates $(i+1)$ to $d_\mathrm{in}$ remain fixed and can therefore be precomputed before entering the CD loop:
\begin{align*}
\bB \coloneqq \frac{\bH_{i, [d_\mathrm{in}] \setminus i}}{H_{i,i}} \left(\hbW_{[d_\mathrm{in}] \setminus i, :} - \bW_{[d_\mathrm{in}] \setminus i, :}\right) = \underbrace{\frac{\bH_{i, 1:i}}{H_{i,i}} \left(\hbW_{1:i, :} - \bW_{1:i, :}\right)}_{\text{Cannot be precomputed before the CD loop}} + \underbrace{\frac{\bH_{i, (i+1):}}{H_{i,i}} \left(\hbW_{(i+1):, :} - \bW_{(i+1):, :}\right)}_{\text{Can be precomputed before the CD loop}}.
\end{align*}

To take advantage of this, we precompute the second term (which corresponds to future coordinates) for all $i \in [d_\mathrm{in}]$ in parallel using matrix operations before entering the CD loop:
\begin{align*} \bB \gets \mathrm{StrictUpper}(\widetilde{\bH}) (\hbW - \bW), \end{align*}
where $\widetilde{\bH}$ is obtained by dividing each row of $\bH$ by the corresponding diagonal entry $H_{i,i}$, and $\mathrm{StrictUpper}(\cdot)$ extracts the strictly upper triangular part of the matrix.

During the CD loop, we use the precomputed $\bB$ to compute the coordinate-wise update in Eq.~\eqref{eq:cd_closed_par}, and after updating the $i$-th coordinate, we incrementally update $\bB$ to reflect the new value of $\hbW_{i,:}$:
\begin{align*}
&\hbW_{i, :} \leftarrow \round(\bW_{i, :} - \bB_{i, :})\\
&\bB_{(i+1):, :} \gets\bB_{(i+1):, :} +  \mathrm{StrictUpper}(\mathbf{\widetilde{H}})_{(i+1):,i} (\hbW_{i,:} - \bW_{i,:}).
\end{align*}
The full CD algorithm incorporating this precomputation strategy is provided in \cref{alg:cd_v1}.

This acceleration trick has been proposed in QuIP \citep[Appendix B.2.] {chee2024quip} and QuantEase \citep{behdin2023quantease}. It is worth noting that this precomputation trick does not change the theoretical time complexity, but improves practical performance by exploiting the GPU parallelization. In particular, the CD update for the $i$-th coordinate in \cref{eq:cd_closed_par} requires $2 d_\mathrm{out}(d_\mathrm{in} - 1)$ FLOPs without precomputation, while with precomputation, the cost is reduced to $2 d_\mathrm{out}(d_\mathrm{in} - i)$ FLOPs.

\paragraph{Lazy Batch-updates.}

After incorporating the precomputation trick, we observe that the update steps within the CD loop (Lines 4–7 in \cref{alg:cd_v1}) resemble the OBQ update scheme used in the GPTQ method \citep{frantar2022optimal,frantar2022gptq}. In OBQ, each iteration involves rounding a single coordinate and adjusts the not-yet-rounded coordinates accordingly. Analogously, our CD update with precomputation rounds $\bW_{i, :} - \bB_{i, :}$ for the $i$-th coordinate and incrementally updates $\bB_{(i+1):, :}$ to reflect the new values of $\hbW_{i, :}$.

Both OBQ and our CD update suffer from a low compute-to-memory ratio: although each iteration involves relatively few FLOPs, it requires frequent reading and writing to large matrices. As a result, these updates tend to be memory-bound and suffer from poor GPU utilization.
To mitigate this, GPTQ introduces \textit{lazy batch-updates}, in which a batch of coordinates (with batch size $b = 128$) is processed together. Within each batch, updates are applied sequentially to each coordinate, while corrections are made only for the remaining unprocessed coordinates within the batch. Once all $b$ coordinates in the batch are updated, a global correction step is performed for the rest of the matrix. This strategy improves memory efficiency by reducing the frequency of global updates.

We adopt this lazy batch-updates approach in our CD implementation with precomputation trick. Specifically, we restrict updates to the relevant portion of $\bB$ within each block of $b$ coordinates, and defer global updates to $\bB$ until the entire block has been processed. This significantly reduces memory-bound operations and enhances GPU utilization. The final efficient CD algorithm incorporating both precomputation trick and lazy batch-updates is given in \cref{alg:cd_v2}.

QuIP \citep{chee2024quip} also supports lazy batch-updates in their open-source code, though it is not mentioned in their paper. QuantEase \citep{behdin2023quantease} does not use this approach in their implementation.
As with the precomputation trick, lazy batch-updates do not change the theoretical time complexity. However, they substantially accelerate the overall algorithm in practice by better utilizing GPU resources.

\paragraph{Speedup Factor}

To demonstrate the speedup achieved by our optimization techniques for the CD algorithm, we report the quantization time for quantizing the Llama-2-7B model into 4-bit precision on a single RTX 6000 Ada GPU.
Without any optimizations, adopting the naive strategy of exhaustively evaluating the objective function for all coordinate choices and selecting the option with the lowest value takes 3.9 hours to quantize the entire model. Applying the coordinate-wise closed-form solution described in Eq. \eqref{eq:cd_closed_par} reduces this time to 2.7 hours. Incorporating the precomputation trick further lowers it to 1.2 hours. Finally, applying lazy batch-updates brings the total quantization time down to just 0.9 hours.
Overall, these optimizations yield more than a 4$\times$ speedup in end-to-end quantization time on GPU.

\subsection{GPTVQ} \label{app:gptvq_detail}
In the original GPTVQ paper \citep{van2024gptvq}, the authors used 128 sentences from the WikiText2 dataset \citep{wikitext103}, each containing 2048 tokens, as a calibration data.
For a fair comparison, we reproduced their method using their open-sourced code but used 1024 sentences of RedPajama dataset \citep{together2023redpajama}, each containing 4096 tokens.
We adopted their default hyperparameters except for the group size and block size, which we adjusted to match the average bit width when comparing with different methods in \cref{tab:scalar}.
We provide a complete list of GPTVQ hyperparameters for each table in \cref{tab:hp_gptvq}.

\begin{table}[t]
\caption{Hyperparameters that we used in reproducing GPTVQ \citep{van2024gptvq} in \cref{tab:scalar} and \cref{tab:vector}.}
\vspace{0.5em}
\centering
\begin{adjustbox}{max width=0.95\columnwidth}
\begin{tabular}{cccccc c}
\toprule
& Weight & VQ & Codebook sharing & Scaling & Codebook
\\
Table & 
bits &  
dim &  
group size & 
block size &
bit-width &
Avg bits
\\
\cmidrule(lr){1-1}\cmidrule(lr){2-6}\cmidrule(lr){7-7}
\multirow{4}{*}{\cref{tab:scalar}}
&2 & 1 & 1024 &  -- & 8 & 2.03
\\
&3 & 1 & 2048 &  -- & 8 & 3.03
\\
&4 & 1 & 8192 & 256 & 8 & 4.03
\\
&4 & 1 & 4096 & 128 & 8 & 4.06
\\
\cmidrule(lr){1-1}\cmidrule(lr){2-6}\cmidrule(lr){7-7}
\multirow{4}{*}{\cref{tab:vector}}
&2 & 2 &  2048 & -- & 8 & 2.13 
\\
&2 & 4 & 32768 & -- & 8 & 2.25
\\
&3 & 2 & 16384 & 64 & 8 & 3.13
\\
&4 & 2 & 65536 & 64 & 8 & 4.13
\\
\bottomrule
\end{tabular}
\end{adjustbox}
\label{tab:hp_gptvq}
\end{table}

\section{Details on Experimental Setup}
\label{sec:details}

This section provides a detailed explanation of the experimental settings used.



\subsection{End-to-end Inference Throughput Experiments (\cref{tab:speedup})}\label{sec:inferSetup}

\begin{table}[t]
\caption{End-to-end inference throughput of Llama-2 models on RTX 4090 GPU, including the vector quantization kernel \textit{after} fusing the query/key/value projection matrices into one linear layer and the up/gate projection matrices into another when measuring the throughput. OOM indicates an Out-of-Memory error, meaning the GPU lacks memory to run model inference.}
\vspace{0.5em}
\centering
\begin{adjustbox}{max width=1.0\columnwidth}
\begin{tabular}{l cc cc cc}
\toprule
& \multicolumn{2}{c}{Llama-2-7B} & \multicolumn{2}{c}{Llama-2-13B} & \multicolumn{2}{c}{Llama-2-70B} \\
\cmidrule(lr){2-3}\cmidrule(lr){4-5}\cmidrule(lr){6-7}
Type &  
Bits$\downarrow$ &  
Tok/s$\uparrow$ & 
Bits$\downarrow$ &  
Tok/s$\uparrow$ &
Bits$\downarrow$ & 
Tok/s$\uparrow$ 
\\
\cmidrule(lr){1-1}\cmidrule(lr){2-3}\cmidrule(lr){4-5}\cmidrule(lr){6-7}
Original 
& 16 & 67
& 16 & OOM
& 16 & OOM
\\
\cmidrule(lr){1-1}\cmidrule(lr){2-3}\cmidrule(lr){4-5}\cmidrule(lr){6-7}
Uniform scalar 
&2.00 & 334
&2.00 & 200
&2.00 & 47
\\
Non-uniform scalar
& 2.01 & 347
& 2.01 & 203
& 2.01 & 47
\\
Vector
& 2.00 & 200
& 2.00 & 121
& 2.00 & 38
\\
Vector \textit{(fused)}
& 2.00 & 248
& 2.00 & 153
& 2.00 & 42
\\
\cmidrule(lr){1-1}\cmidrule(lr){2-3}\cmidrule(lr){4-5}\cmidrule(lr){6-7}
Uniform scalar
& 3.00 & 260
& 3.00 & 150
& 3.00 & OOM
\\
Non-uniform scalar
& 3.03 & 264
& 3.02 & 148
& 3.01 & OOM
\\
Vector
& 3.00 & 176
& 3.00 & 103
& 3.00 & OOM
\\
Vector \textit{(fused)}
& 3.00 & 209
& 3.00 & 123
& 3.00 & OOM
\\
\cmidrule(lr){1-1}\cmidrule(lr){2-3}\cmidrule(lr){4-5}\cmidrule(lr){6-7}
Uniform scalar
& 4.00 & 214
& 4.00 & 121
& 4.00 & OOM
\\
Non-uniform scalar
& 4.05 & 209
& 4.04 & 116
& 4.03 & OOM
\\
Vector
& 4.00 & 151
& 4.00 & 89
& 4.00 & OOM
\\
Vector \textit{(fused)}
& 4.00 & 176
& 4.00 & 103
& 4.00 & OOM
\\
\bottomrule
\end{tabular}
\end{adjustbox}
\label{tab:speedup_full}
\end{table}

In \cref{tab:speedup}, we measure each model’s inference throughput in generating 100 tokens on RTX 4090 GPU, after integrating the kernels with into a PyTorch-based inference pipeline optimized with the \texttt{torch.compile} function \citep{ansel2024pytorch,nvidia2019cuda}. For QTIP, our chosen vector quantization kernel, we adopt the HYB variant of it as its GPU kernel is publicly available, though it is possible to implement fast GPU kernels with other variants as well \citep{tseng2024qtip}.

For the base model and for models quantized using uniform or non-uniform scalar formats, we fuse the query/key/value projection matrices into one linear layer and the up/gate projection matrices into another when measuring the throughput. This fusion trick can be applied to QTIP as well, provided the matrices are fused before quantization and the scale parameters are shared across layers. However, in the main paper, we present QTIP results without fusion to match the original experimental setup (and reported numbers) from their work, in which they quantize the layers independently without fusing them. Meanwhile, scalar quantization methods quantize the layer in an output channel-wise manner, and this allows fusing matrices even when layers are quantized separately.

For completeness, we include \cref{tab:speedup_full}, which also shows QTIP’s fused end-to-end throughput (measured using dummy values) to illustrate the impact of fusion, restating the relevant results from \cref{tab:speedup}. Although the fusion boosts the throughput, it does not change the conclusion that QTIP still runs more slowly than the scalar quantization methods.

\subsection{Implementation Details for Different Quantization Types} \label{sec:gq_diff_quant}

GuidedQuant employs a quantization algorithm $\mathcal{Q}$ as a subroutine (Line 8 in \cref{alg:gq}). In this section, we clarify which specific quantization algorithm $\mathcal{Q}$ each method uses, which GuidedQuant builds upon in \cref{alg:gq}. We integrate GuidedQuant with three different quantization methods: (1) LNQ for weight-only scalar quantization, (2) QTIP for weight-only vector quantization, and (3) SpinQuant for weight-and-activation quantization. LNQ adopts the algorithm shown in \cref{alg:lnq}, QTIP uses the BlockLDLQ algorithm proposed in \citet{tseng2024quip}, and SpinQuant employs the GPTQ algorithm introduced in \citet{frantar2022gptq}.

\begin{table}[t]
\caption{Total GPU cost incurred during the quantization process for LNQ and QTIP, both with and without GuidedQuant, across various group sizes $g$. We specify the number and type of GPU used in the parentheses. R6A denotes the RTX 6000 Ada GPU.}
\vspace{0.5em}
\centering
\begin{adjustbox}{max width=1.0\columnwidth}
\begin{tabular}{l l ccc ccc}
\toprule
 & &  
\multicolumn{3}{c}{LNQ}  &
\multicolumn{3}{c}{QTIP} \\
\cmidrule(lr){3-5}\cmidrule(lr){6-8}
Model & Method &  
GPU Cost - 2 bits  &  
GPU Cost - 3 bits  & 
GPU Cost - 4 bits  &
GPU Cost - 2 bits  &  
GPU Cost - 3 bits  & 
GPU Cost - 4 bits 
\\
\cmidrule(lr){1-1}\cmidrule(lr){2-2}\cmidrule(lr){3-5}\cmidrule(lr){6-8}
Llama-2-7B & Layer-wise (LNQ, QTIP)
& 0.5 h (1$\times$R6A) & 0.6 h (1$\times$R6A) & 0.9 h (1$\times$R6A)
& 1.3 h (1$\times$R6A) & 1.2 h (1$\times$R6A) & 1.2 h (1$\times$R6A)
\\
&Layer-wise + GQuant ($g=1$) 
& 0.5 h (1$\times$R6A) & 0.6 h (1$\times$R6A) & 0.9 h (1$\times$R6A)
& 1.3 h (1$\times$R6A) & 1.2 h (1$\times$R6A) & 1.2 h (1$\times$R6A)
\\
&Layer-wise + GQuant ($g=2$) 
& 0.6 h (1$\times$R6A) & 0.7 h (1$\times$R6A) & 0.9 h (1$\times$R6A) 
& 1.5 h (1$\times$R6A) & 1.4 h (1$\times$R6A) & 1.5 h (1$\times$R6A)
\\
&Layer-wise + GQuant ($g=4$) 
& 0.7 h (1$\times$R6A) & 0.7 h (1$\times$R6A) & 0.9 h (1$\times$R6A)
& 1.9 h (1$\times$R6A) & 1.9 h (1$\times$R6A) & 1.9 h (1$\times$R6A)
\\
\cmidrule(lr){1-1}\cmidrule(lr){2-2}\cmidrule(lr){3-5}\cmidrule(lr){6-8}
Llama-2-13B & Layer-wise (LNQ, QTIP)
& 0.9 h (1$\times$R6A) & 1.1 h (1$\times$R6A) & 1.6 h (1$\times$R6A)
& 3.0 h (1$\times$R6A) & 2.7 h (1$\times$R6A) & 2.7 h (1$\times$R6A)
\\
&Layer-wise + GQuant ($g=1$) 
& 0.9 h (1$\times$R6A) & 1.1 h (1$\times$R6A) & 1.6 h (1$\times$R6A)
& 3.0 h (1$\times$R6A) & 2.7 h (1$\times$R6A) & 2.7 h (1$\times$R6A)
\\
&Layer-wise + GQuant ($g=2$) 
& 1.1 h (1$\times$R6A) & 1.2 h (1$\times$R6A) & 1.6 h (1$\times$R6A)
& 2.4 h (1$\times$R6A) & 2.2 h (1$\times$R6A) & 2.3 h (1$\times$R6A)
\\
&Layer-wise + GQuant ($g=4$) 
& 1.2 h (1$\times$R6A) & 1.3 h (1$\times$R6A) & 1.7 h (1$\times$R6A)
& 3.0 h (1$\times$R6A) & 3.0 h (1$\times$R6A) & 3.0 h (1$\times$R6A)
\\
\cmidrule(lr){1-1}\cmidrule(lr){2-2}\cmidrule(lr){3-5}\cmidrule(lr){6-8}
Llama-2-70B & Layer-wise (LNQ, QTIP)
& 2.6 h (1$\times$R6A) & 3.3 h (1$\times$R6A) & 5.1 h (1$\times$R6A)
& 12.0 h (1$\times$R6A) & 10.8 h (1$\times$R6A) & 11.0 h (1$\times$R6A)
\\
&Layer-wise + GQuant ($g=1$) 
& 2.6 h (1$\times$R6A) & 3.3 h (1$\times$R6A) & 5.1 h (1$\times$R6A)
& 12.0 h (1$\times$R6A) & 10.8 h (1$\times$R6A) & 11.0 h (1$\times$R6A)
\\
&Layer-wise + GQuant ($g=2$) 
& 3.7 h (1$\times$R6A) & 4.7 h (1$\times$R6A) & 6.8 h (1$\times$R6A)
& 13.0 h (1$\times$R6A) & 11.9 h (1$\times$R6A) & 12.0 h (1$\times$R6A)
\\
\bottomrule
\end{tabular}
\end{adjustbox}
\label{tab:cost_quant}
\end{table}

\begin{table}[t]
\caption{Total GPU cost and disk usage incurred during the gradient and Hessian caching processes for each objective—weighted $k$-means (SqueezeLLM), layer-wise (LNQ, QTIP), and GuidedQuant. We specify the number and type of GPU used in the parentheses. R6A and A100 denote the RTX 6000 Ada GPU and the A100 GPU, respectively. The calibration data are 1024 sentences of the RedPajama dataset, each containing 4096 tokens.}
\vspace{0.5em}
\centering
\begin{adjustbox}{max width=0.8\columnwidth}
\begin{tabular}{l l cc cc}
\toprule
 & &  
\multicolumn{2}{c}{Gradient Caching} &  
\multicolumn{2}{c}{Hessian Caching} \\
\cmidrule(lr){3-4}\cmidrule(lr){5-6}
Model & Method &  
GPU Cost &  
Disk Size & 
GPU Cost &  
Disk Size 
\\
\cmidrule(lr){1-1}\cmidrule(lr){2-2}\cmidrule(lr){3-4}\cmidrule(lr){5-6}
Llama-2-7B & Weighted $k$-means (SqueezeLLM)
& 0.3 h (1$\times$A100) & 13 GiB & -- & -- 
\\
& Layer-wise (LNQ, QTIP)
& -- & -- & 0.3 h (4$\times$R6A) & 27 GiB
\\
& Layer-wise + GQuant ($g=1$)
& 0.3 h (1$\times$A100) & 2 GiB & 0.3 h (4$\times$R6A) & 27 GiB
\\
& Layer-wise + GQuant ($g=2$)
& 0.3 h (1$\times$A100) & 4 GiB & 0.4 h (4$\times$R6A) & 53 GiB
\\
& Layer-wise + GQuant ($g=4$)
& 0.3 h (1$\times$A100) & 7 GiB & 0.8 h (4$\times$R6A) & 106 GiB 
\\
\cmidrule(lr){1-1}\cmidrule(lr){2-2}\cmidrule(lr){3-4}\cmidrule(lr){5-6}
Llama-2-13B & Weighted $k$-means (SqueezeLLM)
& 0.6 h (2$\times$A100) & 25 GiB& -- & -- 
\\
& Layer-wise (LNQ, QTIP)
& -- & -- & 0.5 h (4$\times$R6A) & 52 GiB
\\
& Layer-wise + GQuant ($g=1$)
& 0.6 h (2$\times$A100) & 3 GiB & 0.5 h (4$\times$R6A) & 52 GiB
\\
& Layer-wise + GQuant ($g=2$)
& 0.6 h (2$\times$A100) & 7 GiB & 0.9 h (4$\times$R6A) & 104 GiB
\\
& Layer-wise + GQuant ($g=4$)
& 0.6 h (2$\times$A100) & 13 GiB & 1.5 h (4$\times$R6A) & 208 GiB 
\\
\cmidrule(lr){1-1}\cmidrule(lr){2-2}\cmidrule(lr){3-4}\cmidrule(lr){5-6}
Llama-2-70B & Weighted $k$-means (SqueezeLLM)
& 2.7 h (6$\times$A100) & 129 GiB& -- & -- 
\\
& Layer-wise (LNQ, QTIP)
& -- & -- & 3.5 h (4$\times$R6A) & 366 GiB
\\
& Layer-wise + GQuant ($g=1$)
& 2.7 h (6$\times$A100) & 5 GiB & 3.5 h (4$\times$R6A) & 366 GiB
\\
& Layer-wise + GQuant ($g=2$)
& 2.7 h (6$\times$A100) & 9 GiB & 5.8 h (4$\times$R6A) & 731 GiB
\\
\bottomrule
\end{tabular}
\end{adjustbox}
\label{tab:cost_hess}
\end{table}



\section{Additional Results and Discussions}

\label{sec:add}

\subsection{Quantization Cost} In this section, we present a detailed breakdown of the computational costs associated with our method, as summarized in \cref{tab:cost_quant} and \cref{tab:cost_hess}. The layer-wise quantization methods on which we build typically require two phases: (1) caching the Hessian matrices to disk, and (2) loading them to quantize weights based on these cached Hessian matrices. It is worth noting that the cost of the first phase (caching) can be amortized if one needs to quantize the same model multiple times at different bit-widths or configurations, as the Hessian matrices can be reused. We report the weight quantization cost in \cref{tab:cost_quant}, and the Hessian-caching cost in \cref{tab:cost_hess}.

From \cref{tab:cost_quant}, observe that for $g=1$, the quantization is identical to standard layer-wise quantization, since the Hessian size is the same. 
Even for $g=2$ or $g=4$, the quantization cost does not increase by more than 50\%. This is because while more Hessian matrices are employed, each weight block to be quantized becomes correspondingly smaller, leaving the total computation unchanged.
All these steps can be performed in an embarrassingly parallel manner; for example, quantizing Llama-2-70B using our LNQ algorithm takes less than three hours when using 8 RTX 6000 Ada GPUs.

We further report the cost of caching Hessian matrices in \cref{tab:cost_hess}, along with the number of GPU used and the disk size requirements. Note that while we used 4 GPUs for caching, this process is also fully parallelizable; using fewer GPUs will simply take longer (it can run on a single GPU), whereas additional GPUs can shorten the total time. Finally, our method’s disk-space requirement is proportional to the number of groups $g$. However, we highlight that for constrained disk space, choosing a smaller number of groups can still capture most of the performance benefits (\cref{tab:group}).

\subsection{Results on Llama-3 Models}

\begin{table*}[t]
\caption{Weight-only scalar post-training quantization results on Llama-3 models. Wiki2 and C4 denotes perplexity on WikiText2 and C4, respectively. The perplexity is measured with the context size of 8192.}
\vspace{0.5em}
\centering
\begin{adjustbox}{max width=1.0\columnwidth}
\begin{tabular}{l ccc ccc}
\toprule
& \multicolumn{3}{c}{Llama-3-8B} & \multicolumn{3}{c}{Llama-3-70B} \\
\cmidrule(lr){2-4}\cmidrule(lr){5-7}
Method & 
Bits$\downarrow$ & 
Wiki2$\downarrow$ & 
C4$\downarrow$ &
Bits$\downarrow$ & 
Wiki2$\downarrow$ & 
C4$\downarrow$
\\
\cmidrule(lr){1-1}\cmidrule(lr){2-2}\cmidrule{3-4}\cmidrule(lr){5-5}\cmidrule{6-7}
Original & 
16 & 5.54 & 7.10 &  
16 & 2.59 & 5.78 \\
\cmidrule(lr){1-1}\cmidrule(lr){2-2}\cmidrule{3-4}\cmidrule(lr){5-5}\cmidrule{6-7}
SqueezeLLM &  
2.01 & 16322 & 1501 & 
2.01 & 38.53 & 38.15 \\
LNQ  (Ours) & 
2.01 & 133.00 & 72.75 &  
2.01 & 24.22 & 19.71 \\
LNQ + GuidedQuant (Ours) & 
2.01 & \textbf{30.80} & \textbf{20.41}  &  
2.01 & \textbf{10.21} & \textbf{11.06} \\
\cmidrule(lr){1-1}\cmidrule(lr){2-2}\cmidrule{3-4}\cmidrule(lr){5-5}\cmidrule{6-7}
SqueezeLLM  &  
3.03 & 7.39 & 8.84 &  
3.02 & 4.12 & 6.44 \\
LNQ  (Ours) & 
3.03 & 7.28 & 8.46 &  
3.01 & 4.57 & 6.61 \\
LNQ + GuidedQuant (Ours) & 
3.03 & \textbf{6.99} & \textbf{8.10} &  
3.01 & \textbf{3.90} & \textbf{6.27} \\
\cmidrule(lr){1-1}\cmidrule(lr){2-2}\cmidrule{3-4}\cmidrule(lr){5-5}\cmidrule{6-7}
SqueezeLLM & 
4.05 & 5.91 & 7.43 &  
4.03 & 2.91 & 5.91 \\
LNQ  (Ours) & 
4.05 & 5.90 & 7.40 & 
4.03 & 3.05 & 5.94 \\
LNQ + GuidedQuant (Ours) & 
4.05 & \textbf{5.80} & \textbf{7.32} &  
4.03 & \textbf{2.89} & \textbf{5.89} \\
\bottomrule
\end{tabular}
\end{adjustbox}
\label{tab:scalar-llama3}
\vspace{-0.5em}
\end{table*}
In this section, we present the results of evaluating LNQ and LNQ combined with GuidedQuant on Llama-3-8B and Llama-3-70B models, comparing with SqueezeLLM under a weight-only scalar quantization setting. We present the results in \cref{tab:scalar-llama3}.
We use RedPajama dataset \citep{together2023redpajama} for calibration with 1024 sentences, each containing 4096 tokens.
We set the number of groups to be $g=1$ for Llama-3-8B and Llama-3-70B, and set the hyperparameters for LNQ (and LNQ + GuidedQuant) to be $T=2, K=4$ for Llama-3-8B and $T=1, K=4$ for Llama-3-70B model.
LNQ with GuidedQuant consistently outperforms the baselines, demonstrating the robustness and effectiveness of our approach.

\subsection{Additional Inference Throughput Results}

\begin{table*}[t]
\caption{Weight-only scalar post-training quantization results on Llama-2 models, including end-to-end throughput. Wiki2 and C4 denotes perplexity on WikiText2 and C4, respectively. The perplexity is measured with the context size of 4096.  Throughput is evaluated on an RTX 3090 GPU, reported as the average of 5 runs with standard deviation in parentheses. OOM indicates an Out-of-Memory error, meaning the GPU lacks memory to run model inference.}
\vspace{0.5em}
\centering
\begin{adjustbox}{max width=1.0\columnwidth}
\begin{tabular}{l cccc cccc cccc}
\toprule
& \multicolumn{4}{c}{Llama-2-7B} 
& \multicolumn{4}{c}{Llama-2-13B}  
& \multicolumn{4}{c}{Llama-2-70B} \\
\cmidrule(lr){2-5}\cmidrule(lr){6-9}\cmidrule(lr){10-13}
Method & 
Bits$\downarrow$ & 
Wiki2$\downarrow$ & 
C4$\downarrow$ &
Tok/s$\uparrow$ &
Bits$\downarrow$ & 
Wiki2$\downarrow$ & 
C4$\downarrow$ &
Tok/s$\uparrow$ &
Bits$\downarrow$ & 
Wiki2$\downarrow$ & 
C4$\downarrow$ &
Tok/s$\uparrow$
\\
\cmidrule(lr){1-1}\cmidrule(lr){2-2}\cmidrule{3-5}\cmidrule(lr){6-6}\cmidrule{7-9}\cmidrule(lr){10-10}\cmidrule{11-13}
Original & 
16 & 5.12 & 6.63  & 64.8 (0.1) & 
16 & 4.57 & 6.05  & OOM & 
16 & 3.12 & 4.97  & OOM \\
\cmidrule(lr){1-1}\cmidrule(lr){2-2}\cmidrule{3-5}\cmidrule(lr){6-6}\cmidrule{7-9}\cmidrule(lr){10-10}\cmidrule{11-13}
SqueezeLLM &  
2.01 & 39.58  & 44.05 & 245.1 (1.8) &
2.01 & 16.24  & 19.20 & 140.5 (0.5) &
2.01 &  9.17  & 13.03 &  31.5 (0.0) \\
LNQ  (Ours) & 
2.01 &  23.31  & 26.71  & 244.6 (0.6) &
2.01 &  8.78  & 11.80  & 141.1 (0.4) &
2.01 & 5.23   & 7.31 & 31.6 (0.1)\\
LNQ + GuidedQuant (Ours) & 
2.01 & \textbf{8.83} & \textbf{11.15}  & 244.4 (2.9)  &  
2.01 & \textbf{7.26} & \textbf{9.17}  & 141.2 (0.5) &
2.01 & \textbf{5.04} & \textbf{7.04} &  31.6 (0.1)\\
\cmidrule(lr){1-1}\cmidrule(lr){2-2}\cmidrule{3-5}\cmidrule(lr){6-6}\cmidrule{7-9}\cmidrule(lr){10-10}\cmidrule{11-13}
SqueezeLLM  &  
3.03 & 5.74  & 7.44 & 207.3 (1.6) & 
3.02 & 4.99  & 6.60 & 118.0 (0.5) &
3.01 & 3.53  & 5.31 & OOM \\
LNQ  (Ours) & 
3.03 & 5.89  & 7.74 & 207.3 (2.1) & 
3.02 & 5.02  & 6.68 & 118.0 (0.6) &
3.01 & 3.50  & 5.31 & OOM \\
LNQ + GuidedQuant (Ours) & 
3.03 & \textbf{5.57} & \textbf{7.22} & 207.6 (1.7) & 
3.02 & \textbf{4.91} & \textbf{6.49} &  117.9 (0.6)  &
3.01 & \textbf{3.47} & \textbf{5.27} & OOM \\
\cmidrule(lr){1-1}\cmidrule(lr){2-2}\cmidrule{3-5}\cmidrule(lr){6-6}\cmidrule{7-9}\cmidrule(lr){10-10}\cmidrule{11-13}
SqueezeLLM & 
4.05 & 5.23 & 6.78 & 161.8 (1.5) & 
4.04 & 4.67 & 6.15 & 89.8 (0.1) &
4.03 & \textbf{3.20} & 5.04 & OOM \\
LNQ  (Ours) & 
4.05 & 5.26 & 6.82  & 161.7 (1.6) & 
4.04 & 4.67 & 6.17 & 89.7 (0.1) &
4.03 & \textbf{3.20} & 5.04 & OOM \\
LNQ + GuidedQuant (Ours) & 
4.05 & \textbf{5.21} & \textbf{6.75} & 162.0 (1.8) & 
4.04 & \textbf{4.65} & \textbf{6.14} & 89.8 (0.1) &
4.03 & \textbf{3.20} & \textbf{5.03} & OOM \\
\bottomrule
\end{tabular}
\end{adjustbox}
\label{tab:scalar-llama2-speedup}
\vspace{-0.5em}
\end{table*}
GuidedQuant leverages existing CUDA kernels (Any-Precision-LLM kernel \citep{park2024any} for weight-only scalar and QTIP kernel \citep{tseng2024qtip} for weight-only vector quantization) and optimizes assignment and codebook values, thus achieving improved performance without sacrificing inference throughput. To validate this, we compare weight-only scalar PTQ results on Llama-2 models across methods using the same CUDA kernel, as shown in \cref{tab:scalar-llama2-speedup}. Specifically, we report perplexity and end-to-end throughput for SqueezeLLM, LNQ, and LNQ + GuidedQuant, all using the Any-Precision Kernel \citep{park2024any}. Throughput is measured on an RTX 3090 GPU as the average of 5 runs, with standard deviation in parentheses. Results confirm that our methods (LNQ and LNQ + GuidedQuant) achieve better perplexity while maintaining the same throughput as other method using the identical kernel.

\subsection{Evaluations on Zero-shot and Few-shot Downstream Benchmarks}

In this section, we provide the evaluations on zero-shot and few-shot downstream tasks of our methods (LNQ and LNQ + GuidedQuant) alongside baselines (SqueezeLLM and GPTVQ 1D) under the weight-only scalar quantization settings, using Llama-2-7B and Llama-2-13B models, in \cref{tab:wo_downstream}. The evaluation includes eight zero-shot tasks: BoolQ \citep{clark2019boolq}, PIQA \citep{bisk2020piqa}, SIQA \citep{sap2019siqa}, HellaSwag \citep{zellers2019hellaswag}, WinoGrande \citep{sakaguchi2019winogrande}, ARC-easy \citep{clark2018think}, ARC-challenge \citep{clark2018think}, and OBQA \citep{mihaylov2018can}. For a few-shot benchmark, we include results on the MMLU \citep{hendrycks2021measuring} benchmark in a 5-shot setting.
We evaluate on these tasks using version 0.4.3 of the \texttt{lm-evaluation-harness} library \citep{eval-harness}.

\cref{tab:wo_downstream} reports both accuracy and standard error for all methods. We highlight the best-performing results, as well as those whose accuracy falls within the top score $\pm$ standard error, under the same bit width constraint. The results show that LNQ combined with GuidedQuant consistently matches or surpasses baseline performance, with notable improvements in extreme quantization scenarios, such as 2-bit quantization.

\begin{table*}[t]
\caption{Weight-only scalar post-training quantization results, evaluated on zero-shot and few-shot downstream tasks. Zero-shot Avg denotes the average accuracy across eight zero-shot tasks: BoolQ, PIQA, SIQA, HellaSwag, WinoGrande, ARC-easy, ARC-challenge, and OBQA. For the few-shot benchmark, MMLU (5-shot) denotes accuracy on the MMLU benchmark in a 5-shot setting.
We report the standard error in parentheses and bold the best results, as well as those whose accuracy score falls within the top score $\pm$ standard error.
}
\vspace{0.5em}
\centering
\begin{adjustbox}{max width=1.0\columnwidth}
\begin{tabular}{l ccc ccc}
\toprule
& \multicolumn{3}{c}{Llama-2-7B} & \multicolumn{3}{c}{Llama-2-13B} \\
\cmidrule(lr){2-4}\cmidrule(lr){5-7}
Method & 
Bits$\downarrow$ & 
Zero-shot Avg$\uparrow$ & 
MMLU (5-shot)$\uparrow$ & 
Bits$\downarrow$ & 
Zero-shot Avg$\uparrow$ & 
MMLU (5-shot)$\uparrow$ 
\\
\cmidrule(lr){1-1}\cmidrule(lr){2-2}\cmidrule{3-4}\cmidrule(lr){5-5}\cmidrule{6-7}
Original & 
16 & 59.88 (0.43) & 45.97 (0.41) &  
16 & 62.80 (0.43) & 54.93 (0.40) \\
\cmidrule(lr){1-1}\cmidrule(lr){2-2}\cmidrule{3-4}\cmidrule(lr){5-5}\cmidrule{6-7}
SqueezeLLM &  
2.01 & 41.80 (0.41) & 24.75 (0.36) & 
2.01 & 42.44 (0.41) & 24.47 (0.36) \\
GPTVQ 1D  & 
2.03 & 37.35 (0.40) & 26.56 (0.37) & 
2.03 & 46.34 (0.41) & 29.63 (0.38) \\
LNQ  (Ours) & 
2.01 & 40.30 (0.40) & 26.76 (0.37) &  
2.01 & 49.51 (0.42) & 32.51 (0.39) \\
LNQ + GuidedQuant (Ours) & 
2.01 & \textbf{50.39 (0.43)} & \textbf{31.53 (0.39)}  &  
2.01 & \textbf{53.98 (0.43)} & \textbf{40.15 (0.41)} \\
\cmidrule(lr){1-1}\cmidrule(lr){2-2}\cmidrule{3-4}\cmidrule(lr){5-5}\cmidrule{6-7}
SqueezeLLM  &  
3.03 & 57.55 (0.43) & 40.59 (0.41) &  
3.02 & \textbf{61.16 (0.43)} & 49.94 (0.40) \\
GPTVQ 1D  &  
3.03 & 54.92 (0.43) & 41.08 (0.41) &  
3.03 & 60.38 (0.43) & 52.06 (0.40) \\
LNQ  (Ours) & 
3.03 & 56.85 (0.43) & 42.18 (0.41) &  
3.02 & 60.61 (0.43) & 51.62 (0.40) \\
LNQ + GuidedQuant (Ours) & 
3.03 & \textbf{58.16 (0.43)} & \textbf{43.38 (0.41)} &  
3.02 & \textbf{61.00 (0.43)} & \textbf{52.67 (0.40)} \\
\cmidrule(lr){1-1}\cmidrule(lr){2-2}\cmidrule{3-4}\cmidrule(lr){5-5}\cmidrule{6-7}
SqueezeLLM & 
4.05 & \textbf{59.41 (0.43)} & \textbf{44.79 (0.41)} &  
4.04 & \textbf{62.32 (0.43)} & 54.52 (0.40) \\
GPTVQ 1D &  
4.06 &  \textbf{59.23 (0.43)} & \textbf{45.06 (0.41)} &  
4.06 &  \textbf{62.37 (0.43)} & \textbf{54.95 (0.40)} \\
LNQ  (Ours) & 
4.05 & \textbf{59.14 (0.43)} & 44.51 (0.41) &  
4.04 & \textbf{62.40 (0.43)} & \textbf{54.79 (0.40)} \\
LNQ + GuidedQuant (Ours) & 
4.05 & \textbf{59.41 (0.43)} & \textbf{45.16 (0.41)} &  
4.04 & \textbf{62.17 (0.43)} & 54.39 (0.40) \\
\bottomrule
\end{tabular}
\end{adjustbox}
\label{tab:wo_downstream}
\vspace{-1em}
\end{table*}

\subsection{Results on Varying the Number of Groups $g$}
\label{sec:group}
In this section, we present results on how varying the number of groups $g$ (introduced in \cref{sec:challenge}) affects performance, focusing on whether fewer groups preserve accuracy or introduce trade-offs when averaging the Hessian within each group. \cref{tab:group} summarizes the impact of changing $g$ under a non-uniform scalar quantization scheme. 
While increasing $g$ can moderately improve results in extreme cases (e.g., quantizing models into 2 bits), performance differences across the number of groups remain minimal in other scenarios.
Note that for weight-only quantization experiments, we chose $g=4$ for Llama-2-7B and Llama-2-13B, and $g=2$ for Llama-2-70B. Still, smaller number of groups are sufficient for achieving most of the performance gains, making them a practical choice for resource-constrained scenarios.

\begin{table*}[t]
\caption{Results with different number of groups $g$ in weight-only post-training quantization results on non-uniform scalar quantization format, \textit{without fine-tuning} to the end-to-end loss. Wiki2 and C4 denotes perplexity on WikiText2 and C4, respectively, which are measured with the context size of 4096.}
\vspace{0.5em}
\centering
\begin{adjustbox}{max width=1.7\columnwidth}
\begin{tabular}{lc ccc ccc ccc}
\toprule
& \multirow{2}{*}{Number of} & \multicolumn{3}{c}{Llama-2-7B} & \multicolumn{3}{c}{Llama-2-13B} & \multicolumn{3}{c}{Llama-2-70B} \\
\cmidrule(lr){3-5}\cmidrule(lr){6-8}\cmidrule(lr){9-11}
Method &
groups $g$ & 
Bits$\downarrow$ & 
Wiki2$\downarrow$ & 
C4$\downarrow$ & 
Bits$\downarrow$ & 
Wiki2$\downarrow$ & 
C4$\downarrow$ & 
Bits$\downarrow$ & 
Wiki2$\downarrow$ & 
C4$\downarrow$  \\
\cmidrule(lr){1-2}\cmidrule(lr){3-3}\cmidrule(lr){4-5}\cmidrule(lr){6-6}\cmidrule(lr){7-8}\cmidrule(lr){9-9}\cmidrule(lr){10-11}
Original & -- &
16 & 5.12 & 6.63 &  
16 & 4.57 & 6.05 &  
16 & 3.12 & 4.97 \\
\cmidrule(lr){1-2}\cmidrule(lr){3-3}\cmidrule(lr){4-5}\cmidrule(lr){6-6}\cmidrule(lr){7-8}\cmidrule(lr){9-9}\cmidrule(lr){10-11}
LNQ & -- &
2.01 & 23.31 & 26.71 &  
2.01 & 8.78  & 11.80 &  
2.01 & 5.23 & 7.31\\
LNQ + GuidedQuant & 1 & 
2.01 & 9.00 & 11.35 &  
2.01 & 7.32 & 9.29 & 
2.01 & 5.11 & 7.06\\ 
 & 2 & 
2.01 & 8.82 & 11.20 &  
2.01 & 7.18 & 9.22 & 
2.01 & 5.04 & 7.04 \\ 
 & 4 & 
2.01 & 8.83 & 11.15 &  
2.01 & 7.26 &  9.17 & 
-- & -- & -- \\ 
\cmidrule(lr){1-2}\cmidrule(lr){3-3}\cmidrule(lr){4-5}\cmidrule(lr){6-6}\cmidrule(lr){7-8}\cmidrule(lr){9-9}\cmidrule(lr){10-11}
LNQ  & -- & 
3.03 & 5.89 & 7.74 &  
3.02 & 5.02 & 6.68 &  
3.01 & 3.50 & 5.31\\
LNQ + GuidedQuant & 1 & 
3.03 & 5.55 & 7.23 &  
3.02 & 4.92 & 6.49 & 
3.01 & 3.46 & 5.27 \\ 
 & 2 & 
3.03 & 5.57 & 7.22 &  
3.02 & 4.92 & 6.49 & 
3.01 & 3.47 & 5.27 \\ 
 & 4 & 
3.03 & 5.57 & 7.22 &  
3.02 & 4.91 & 6.49 & 
-- & -- & -- \\ 
\cmidrule(lr){1-2}\cmidrule(lr){3-3}\cmidrule(lr){4-5}\cmidrule(lr){6-6}\cmidrule(lr){7-8}\cmidrule(lr){9-9}\cmidrule(lr){10-11}
LNQ  & -- & 
4.05 & 5.26 & 6.82 &  
4.04 & 4.67 & 6.17 &  
4.03 & 3.20 & 5.04 \\
LNQ + GuidedQuant & 1 & 
4.05 & 5.21 & 6.75 &  
4.04 & 4.65 & 6.14 & 
4.03 & 3.20 & 5.03\\ 
 & 2 & 
4.05 & 5.22 & 6.75 &  
4.04 & 4.65 & 6.14 & 
4.03 & 3.20 & 5.03 \\ 
 & 4 & 
4.05 & 5.21 & 6.75 &  
4.04 & 4.65 & 6.14 & 
-- & -- & -- \\ 
\bottomrule
\end{tabular}
\end{adjustbox}
\label{tab:group}
\vspace{-1em}
\end{table*}

\subsection{Ablation Study on Assignments Optimization in LNQ}
\label{sec:abl_disc}

In this section, we evaluate our choice of using cyclic CD algorithm instead of GPTQ to solve Problem \eqref{eq:nuq-master} for a fixed codebook $\bc^{(j)}$ in LNQ. In particular, we compare two variants of LNQ with the GuidedQuant objective:  the variant described in \cref{sec:lnq_alg}, which updates the assignments using cyclic CD, and alternative variant that uses GPTQ for assignments updates. Both variants update the codebook using the closed-form solution in \eqref{eq:conti-sol}. We report the results on Llama-2-7B model, evaluated on  WikiText2 and C4 datasets, in \cref{tab:abl_cd}.
Our experiments show that CD consistently outperforms or matches GPTQ, validating our choice of using CD to  optimize the assignment matrix $\bP^{(j)}$.


\begin{table*}[t]
\caption{Ablation study on optimizing discrete assignment $\bP$ in Problem \eqref{eq:nuq-master}. We compare two algorithms for optimizing discrete assignments; GPTQ and coordinate descent algorithm. Wiki2 and C4 denotes perplexity on WikiText2 and C4, respectively, which are measured with the context size of 4096.}
\vspace{0.5em}
\centering
\begin{adjustbox}{max width=1.7\columnwidth}
\begin{tabular}{ll ccc ccc ccc}
\toprule
& \multirow{2}{*}{Optimization} & \multicolumn{3}{c}{Llama-2-7B} & \multicolumn{3}{c}{Llama-2-13B} & \multicolumn{3}{c}{Llama-2-70B}  \\
\cmidrule(lr){3-5}\cmidrule(lr){6-8}\cmidrule(lr){9-11}
Method &
method for $\bP$ & 
Bits$\downarrow$ & 
Wiki2$\downarrow$ & 
C4$\downarrow$ &
Bits$\downarrow$ & 
Wiki2$\downarrow$ & 
C4$\downarrow$ &
Bits$\downarrow$ & 
Wiki2$\downarrow$ & 
C4$\downarrow$ \\
\cmidrule(lr){1-2}\cmidrule(lr){3-3}\cmidrule(lr){4-5}\cmidrule(lr){6-6}\cmidrule(lr){7-8}\cmidrule(lr){9-9}\cmidrule(lr){10-11}
Original & -- &
16 & 5.12 & 6.63 &  
16 & 4.57 & 6.05 &
16 & 3.12 & 4.97\\
\cmidrule(lr){1-2}\cmidrule(lr){3-3}\cmidrule(lr){4-5}\cmidrule(lr){6-6}\cmidrule(lr){7-8}\cmidrule(lr){9-9}\cmidrule(lr){10-11}
LNQ + GQuant & GPTQ &
2.01 & 9.65 & 11.83 &  
2.01 & 7.96 & 11.65 &
2.01 & \textbf{4.92} & \textbf{6.93}\\
 & Coordinate Descent & 
2.01 &  \textbf{8.83} &  \textbf{11.15} &  
2.01 & \textbf{7.26} & \textbf{9.17} &
2.01 & 5.04 & 7.04\\
\cmidrule(lr){1-2}\cmidrule(lr){3-3}\cmidrule(lr){4-5}\cmidrule(lr){6-6}\cmidrule(lr){7-8}\cmidrule(lr){9-9}\cmidrule(lr){10-11}
LNQ + GQuant  & GPTQ & 
3.03 & 5.58 & 7.25 &  
3.02 & \textbf{4.91} & 6.50 &
3.01 & \textbf{3.47} & \textbf{5.27}\\
 & Coordinate Descent & 
3.03 & \textbf{5.57} & \textbf{7.22}  &  
3.02 & \textbf{4.91} & \textbf{6.49} &
3.01 & \textbf{3.47} & \textbf{5.27}\\
\cmidrule(lr){1-2}\cmidrule(lr){3-3}\cmidrule(lr){4-5}\cmidrule(lr){6-6}\cmidrule(lr){7-8}\cmidrule(lr){9-9}\cmidrule(lr){10-11}
LNQ + GQuant  & GPTQ & 
4.05 & 5.22 & \textbf{6.75} &  
4.04 & \textbf{4.65} & \textbf{6.14} & 
4.03 & \textbf{3.20} & \textbf{5.03} \\
 & Coordinate Descent & 
4.05 & \textbf{5.21} & \textbf{6.75} &  
4.04 & \textbf{4.65} & \textbf{6.14} &
4.03 & \textbf{3.20} & \textbf{5.03}\\
\bottomrule
\end{tabular}
\end{adjustbox}
\label{tab:abl_cd}
\end{table*}

\subsection{End-to-end Fine-tuning Results}
\label{sec:e2e}
Recent weight-only quantization methods have explored fine-tuning quantized models using extensive data and compute to improve performance for low-bit models \citep{tseng2024quip, tseng2024qtip, malinovskii2024pvtuning}. In \cref{tab:scalar_ft}, we summarize the performance of quantized models after further fine-tuning on end loss using more data and compute for scalar weight-only quantization. 
We implement PV-Tuning \citep{malinovskii2024pvtuning} in non-uniform scalar quantization setting and report the performance of both our model and SqueezeLLM after fine-tuning with it.
For SqueezeLLM and LNQ + GuidedQuant, we obtain the results using the official open-source implementation of PV-Tuning.
Our fine-tuning setup uses training data from RedPajama dataset \citep{together2023redpajama}, with a context size of 4096 tokens, a batch size of 128 sentences, and fine-tuning for 128 steps in 2-bit quantization and 32 steps in 3-bit quantization.
For GPTQ (uniform scalar quantization), we report the results from the PV-Tuning paper \citep{malinovskii2024pvtuning}.

The results in \cref{tab:scalar_ft} show that our method remains superior, though the gap narrows at larger bit-widths. We hypothesize that existing PTQ methods, which rely on less accurate surrogate objectives, have smaller gaps at higher bit-widths, allowing fine-tuning to narrow the difference. However, in more extreme compression settings, where the gap is wider, our method maintains its advantage even after fine-tuning.


\begin{table}[t]
\caption{Weight-only quantization results on Llama-2-7B model \textit{after fine-tuning} with end-to-end loss. For scalar quantization methods, we report the performance after fine-tuning with PV-Tuning \citep{malinovskii2024pvtuning}. 
}
\vspace{0.5em}
\centering
\begin{adjustbox}{max width=1.0\columnwidth}
\begin{tabular}{c l cc cc}
\toprule
& Method &  
Bits$\downarrow$ &  
Wiki2$\downarrow$ & 
C4$\downarrow$ 
\\
\cmidrule(lr){2-2}\cmidrule(lr){3-3}\cmidrule(lr){4-5}
Type  & Original & 16 & 5.12 & 6.63 \\
\cmidrule(lr){1-1}\cmidrule(lr){2-2}\cmidrule(lr){3-3}\cmidrule(lr){4-5}
\multirow{5}{*}{\shortstack[c]{Weight-only\\ Scalar}}
&GPTQ 
&2.14 & 8.43 & 10.82
\\
&SqueezeLLM 
& 2.01 & 6.78 & 8.82
\\
&LNQ + GQuant (Ours)
& 2.01 & \textbf{6.53} & \textbf{8.53}
\\
\cmidrule(lr){2-2}\cmidrule(lr){3-3}\cmidrule(lr){4-5}
&SqueezeLLM 
& 3.03 & 5.53 & 7.23
\\
&LNQ + GQuant (Ours)
& 3.03 & \textbf{5.50} & \textbf{7.14}
\\
\bottomrule
\end{tabular}
\end{adjustbox}
\label{tab:scalar_ft}
\vspace{-1em}
\end{table}

\subsection{Results on Smaller Bit-width in Weight-and-activation Quantization}

In weight-and-activation quantization, we further conduct an additional experiments with lower bit-widths for weights, specifically 2-bit and 3-bit, while keeping activations and KV caches at 4-bit precision (denoted as W2A4KV4 and W3A4KV4, respectively), on Llama-2-7B model. The results, shown in \cref{tab:wa_extreme}, demonstrate that GuidedQuant outperforms baseline methods by larger margin in these more extreme scenarios, highlighting the strength of our approach under stricter bit-width constraints.

\begin{table}[t]
\caption{Weight-and-activation quantization results on Llama-2-7B model, while quantizing weights into 2- and 3-bits. Wiki2 denotes perplexity on Wikitext2 with the context size of 2048. W$x$A$y$KV$z$ indicates quantizing weights into $x$-, activations into $y$-, and KV cache to $z$-bits, respectively.}
\vspace{0.5em}
\centering
\begin{adjustbox}{max width=1.0\columnwidth}
\begin{tabular}{l cc}
\toprule
Method & 
Bits$\downarrow$ & 
Wiki2$\downarrow$ \\
\cmidrule(lr){1-1}\cmidrule(lr){2-2}\cmidrule{3-3}
Original & 
16 & 5.12  \\
\cmidrule(lr){1-1}\cmidrule(lr){2-2}\cmidrule{3-3}
SpinQuant  & W2A4KV4 & 100.22 \\
SpinQuant + GQuant (Ours) & W2A4KV4 & \textbf{36.05} \\
\cmidrule(lr){1-1}\cmidrule(lr){2-2}\cmidrule{3-3}
SpinQuant  & W3A4KV4 & 6.61 \\
SpinQuant + GQuant (Ours) & W3A4KV4 & \textbf{6.29} \\
\bottomrule
\end{tabular}
\end{adjustbox}
\label{tab:wa_extreme}
\end{table}

\subsection{Comparison with mixed-precision variant of SqueezeLLM}

The dense-and-sparse variant of SqueezeLLM \citep{kim2023squeezellm}, which preserves a small fraction of weights in 16-bit precision to maintain accuracy, is orthogonal to our method and can be combined with it. 
Accordingly, in \cref{tab:wo_dns}, we report results for SqueezeLLM, LNQ, and LNQ + GuidedQuant methods, with the dense-and-sparse approach applied to all of them, using the identical experimental setting with \cref{tab:scalar}.
Following the original SqueezeLLM paper, we retain 0.45\% of the weights in 16-bit and evaluate with 2-, 3-, and 4-bit quantization on the Llama-2-7B model. 
The results show that LNQ with GuidedQuant consistently outperforms the baselines in the dense-and-sparse setting as well, demonstrating the superiority and robustness of our method.

\begin{table}[t]
\caption{Weight-only scalar post-training quantization results on Llama-2-7B model, evaluated under a dense-and-sparse setting, preserving 0.45\% of the weights in 16 bits. Wiki2 and C4 denotes perplexity on WikiText2 and C4, respectively. The perplexity is measured with the context size of 4096.}
\vspace{0.5em}
\centering
\begin{adjustbox}{max width=1.0\columnwidth}
\begin{tabular}{l ccc}
\toprule
Method & 
Bits$\downarrow$ & 
Wiki2$\downarrow$ & 
C4$\downarrow$\\
\cmidrule(lr){1-1}\cmidrule(lr){2-2}\cmidrule{3-4}
Original & 
16 & 5.12 & 6.63 \\
\cmidrule(lr){1-1}\cmidrule(lr){2-2}\cmidrule{3-4}
SqueezeLLM (0.45\%) & 2.22 & 10.64 & 14.10 \\
LNQ (0.45\%) (Ours) & 2.22 & 8.26 & 10.34 \\
LNQ + GuidedQuant (0.45\%) (Ours) & 2.22 & \textbf{8.00} & \textbf{10.18} \\
\cmidrule(lr){1-1}\cmidrule(lr){2-2}\cmidrule{3-4}
SqueezeLLM (0.45\%) & 3.24 & 5.58 & 7.23 \\
LNQ (0.45\%) (Ours) & 3.24 & 5.49 & 7.15 \\
LNQ + GuidedQuant (0.45\%) (Ours) & 3.24 & \textbf{5.48} & \textbf{7.12} \\
\cmidrule(lr){1-1}\cmidrule(lr){2-2}\cmidrule{3-4}
SqueezeLLM (0.45\%) & 4.27 & 5.22 & 6.75 \\
LNQ (0.45\%) (Ours) & 4.27 & \textbf{5.20} & 6.74 \\
LNQ + GuidedQuant (0.45\%) (Ours) & 4.27 & \textbf{5.20} & \textbf{6.73} \\
\bottomrule
\end{tabular}
\end{adjustbox}
\label{tab:wo_dns}
\end{table}

\subsection{Results on Different QTIP Variants (1MAD, 3INST, HYB)}
\label{sec:qtip_var}
\begin{table*}[t]
\caption{Weight-only post-training quantization results on different QTIP variants (1MAD, 3INST, HYB), \textit{without fine-tuning} to the end-to-end loss. Wiki2 and C4 denotes perplexity on WikiText2 and C4, respectively, which are measured with the context size of 4096.}
\vspace{0.5em}
\centering
\begin{adjustbox}{max width=1.7\columnwidth}
\begin{tabular}{cl c cc cc cc}
\toprule
& & & \multicolumn{2}{c}{Llama-2-7B} & \multicolumn{2}{c}{Llama-2-13B} & \multicolumn{2}{c}{Llama-2-70B} \\
\cmidrule(lr){4-5}\cmidrule(lr){6-7}\cmidrule(lr){8-9}
Variant &
Method & 
Bits$\downarrow$ & 
Wiki2$\downarrow$ & 
C4$\downarrow$ & 
Wiki2$\downarrow$ & 
C4$\downarrow$ & 
Wiki2$\downarrow$ & 
C4$\downarrow$  \\
\cmidrule(lr){1-2}\cmidrule(lr){3-3}\cmidrule(lr){4-5}\cmidrule(lr){6-7}\cmidrule(lr){8-9}
& Original & 
16 & 5.12 & 6.63 &  
 4.57 & 6.05 &  
 3.12 & 4.97 \\
\cmidrule(lr){1-2}\cmidrule(lr){3-3}\cmidrule(lr){4-5}\cmidrule(lr){6-7}\cmidrule(lr){8-9}
1MAD & QTIP & 
2.00 & 7.05 & 9.14 &  
 5.59 & 7.46 &  
 3.87 & 5.70 \\
& QTIP + GQuant (Ours) & 
2.00 & \textbf{6.11} & \textbf{7.99} &  
 \textbf{5.33} & \textbf{7.05} &  
 \textbf{3.80} & \textbf{5.61} \\
 \cmidrule(lr){2-2}\cmidrule(lr){3-3}\cmidrule(lr){4-5}\cmidrule(lr){6-7}\cmidrule(lr){8-9}
 & QTIP & 
3.00 & 5.38 & 6.99 &  
 4.74 & 6.28 &  
3.27  & 5.09 \\
& QTIP + GQuant (Ours) & 
3.00 & \textbf{5.28} & \textbf{6.87} &  
 \textbf{4.71} & \textbf{6.22} &  
 \textbf{3.25} & \textbf{5.08} \\
 \cmidrule(lr){2-2}\cmidrule(lr){3-3}\cmidrule(lr){4-5}\cmidrule(lr){6-7}\cmidrule(lr){8-9}
 & QTIP &
4.00 & 5.17 & 6.71 &  
 4.62 & 6.10 &  
 3.16 & \textbf{5.00} \\
& QTIP + GQuant (Ours) & 
4.00 & \textbf{5.16} & \textbf{6.68} &  
 \textbf{4.61} & \textbf{6.09} &  
 \textbf{3.15} & \textbf{5.00} \\
\cmidrule(lr){1-2}\cmidrule(lr){3-3}\cmidrule(lr){4-5}\cmidrule(lr){6-7}\cmidrule(lr){8-9}
3INST & QTIP & 
2.00 & 6.82 & 8.96 &  
 5.52 & 7.39 &  
 3.90 & 5.69 \\
& QTIP + GQuant (Ours) & 
2.00 & \textbf{6.16} & \textbf{7.99} &  
 \textbf{5.33} & \textbf{7.04} &  
 \textbf{3.82} & \textbf{5.61} \\
  \cmidrule(lr){2-2}\cmidrule(lr){3-3}\cmidrule(lr){4-5}\cmidrule(lr){6-7}\cmidrule(lr){8-9}
 & QTIP & 
3.00 & 5.40 & 7.01 &  
 4.74 & 6.28 &  
 3.27 & 5.09 \\
& QTIP + GQuant (Ours) & 
3.00 & \textbf{5.30} & \textbf{6.87} &  
 \textbf{4.70} & \textbf{6.22} &  
 \textbf{3.26} & \textbf{5.08} \\
 \cmidrule(lr){2-2}\cmidrule(lr){3-3}\cmidrule(lr){4-5}\cmidrule(lr){6-7}\cmidrule(lr){8-9}
 & QTIP &
4.00 & 5.17 & 6.71 & 
 4.62 & 6.10 &  
 3.16 & \textbf{5.00} \\
& QTIP + GQuant (Ours) & 
4.00 & \textbf{5.16} & \textbf{6.68} &  
 \textbf{4.61} & \textbf{6.09} &  
 \textbf{3.15} & \textbf{5.00} \\
\cmidrule(lr){1-2}\cmidrule(lr){3-3}\cmidrule(lr){4-5}\cmidrule(lr){6-7}\cmidrule(lr){8-9}
HYB & QTIP & 
2.00 & 6.84 & 9.03 &  
 5.62 & 7.46 &  
 3.93 & 5.74 \\
& QTIP + GQuant (Ours) & 
2.00 & \textbf{6.19} & \textbf{8.06} &  
\textbf{5.36} & \textbf{7.10} &  
\textbf{3.84} & \textbf{5.64} \\
 \cmidrule(lr){2-2}\cmidrule(lr){3-3}\cmidrule(lr){4-5}\cmidrule(lr){6-7}\cmidrule(lr){8-9}
 & QTIP & 
3.00 & 5.39 & 7.03 &  
 4.76 & 6.31 &  
3.28  & 5.10 \\
& QTIP + GQuant (Ours) & 
3.00 & \textbf{5.32} & \textbf{6.89} &  
 \textbf{4.72} & \textbf{6.24} &  
 \textbf{3.27} & \textbf{5.09} \\
 \cmidrule(lr){2-2}\cmidrule(lr){3-3}\cmidrule(lr){4-5}\cmidrule(lr){6-7}\cmidrule(lr){8-9}
 & QTIP &
4.00 & 5.19 & 6.73 &  
 4.63 & 6.12 &  
 3.17 & 5.01 \\
& QTIP + GQuant (Ours) & 
4.00 & \textbf{5.18} & \textbf{6.70} &  
 \textbf{4.61} & \textbf{6.10} &  
 \textbf{3.16} & \textbf{5.00} \\
\bottomrule
\end{tabular}
\end{adjustbox}
\label{tab:qtip}
\end{table*}

The original QTIP paper introduced three variants of their method: 1MAD, 3INST, and HYB \citep{tseng2024qtip}. Both 1MAD and 3INST are look-up table-free methods, while HYB incorporates a small look-up table that fits within the L1 cache of modern GPUs.
The authors reported post-training quantization results without fine-tuning for the 1MAD and 3INST formats, while quantization with fine-tuning was reported for the HYB format. To maintain consistency, we report the better-performing variant between 1MAD and 3INST in \cref{tab:vector} for both QTIP and our method (QTIP + GuidedQuant). For completeness, the full performance results across 1MAD and 3INST are provided in \cref{tab:qtip}.

It is worth noting that QTIP has only open-sourced the CUDA acceleration kernel for HYB, although it is theoretically possible to implement kernels for 1MAD and 3INST. Therefore, we also include the post-training quantization results (without fine-tuning) for the HYB format as well, summarized in \cref{tab:qtip}. The results show that the variations among QTIP methods have minimal impact on the results, and our method consistently outperforms all others in \cref{tab:vector}, regardless of the QTIP variant chosen.

\subsection{Discussion on Block-diagonal Fisher Approximation}\label{app:blk_hess}



\begin{figure}[t]
    \centering
    \makebox[0.06\textwidth]{} 
    \makebox[0.28\textwidth]{{Fisher (Original)}} \hfill
    \makebox[0.28\textwidth]{{WoodFisher}} \hfill
    \makebox[0.28\textwidth]{{GuidedQuant (Ours)}} 


    \makebox[0.06\textwidth][c]{\raisebox{0.5\height}{\rotatebox{90}{ \shortstack{\texttt{self\_attn}\\\texttt{q\_proj}}}}} 
    \begin{minipage}[t]{0.3\textwidth}
        \centering
        \includegraphics[width=\textwidth]{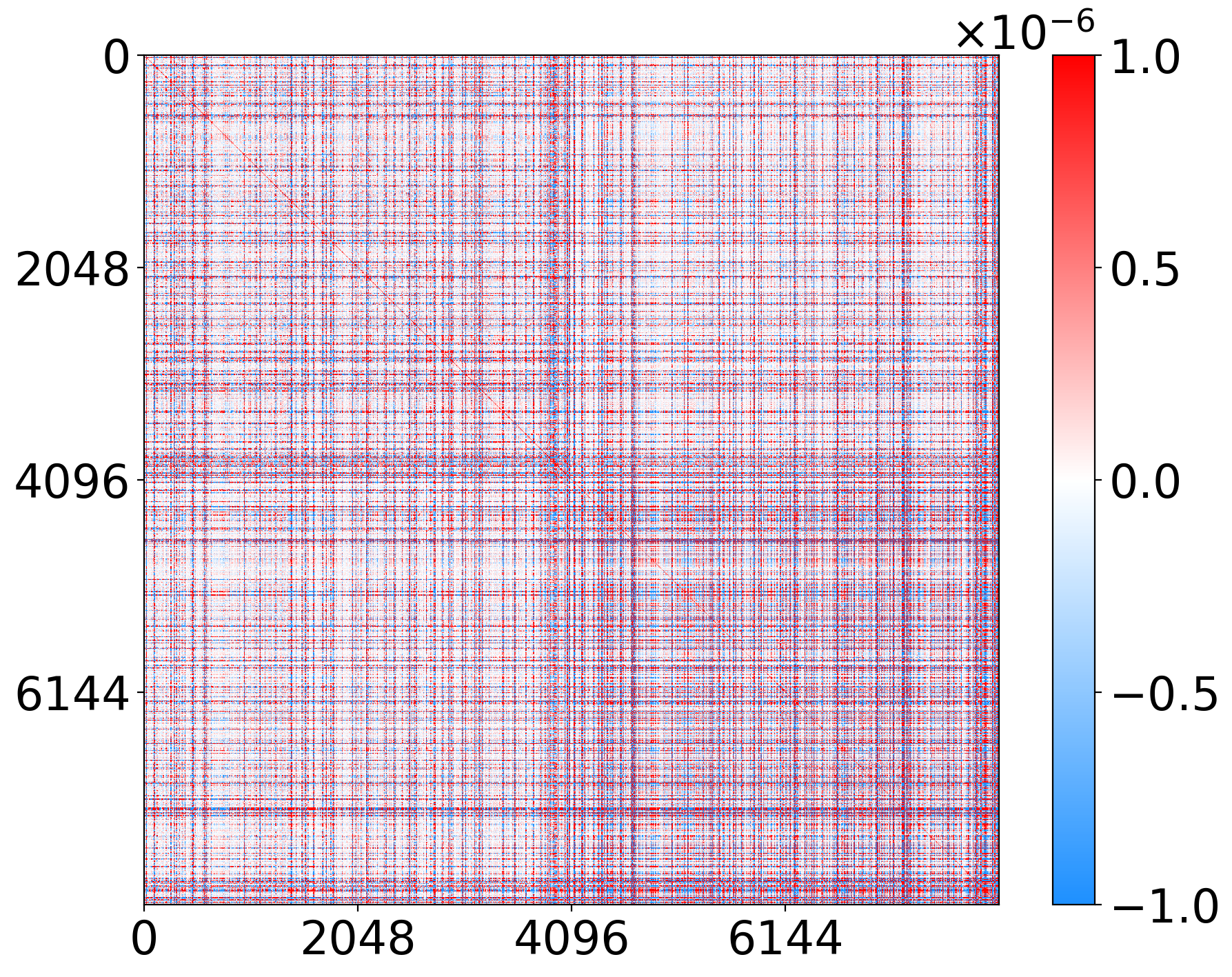}
    \end{minipage}
    \hfill
    \begin{minipage}[t]{0.3\textwidth}
        \centering
        \includegraphics[width=\textwidth]{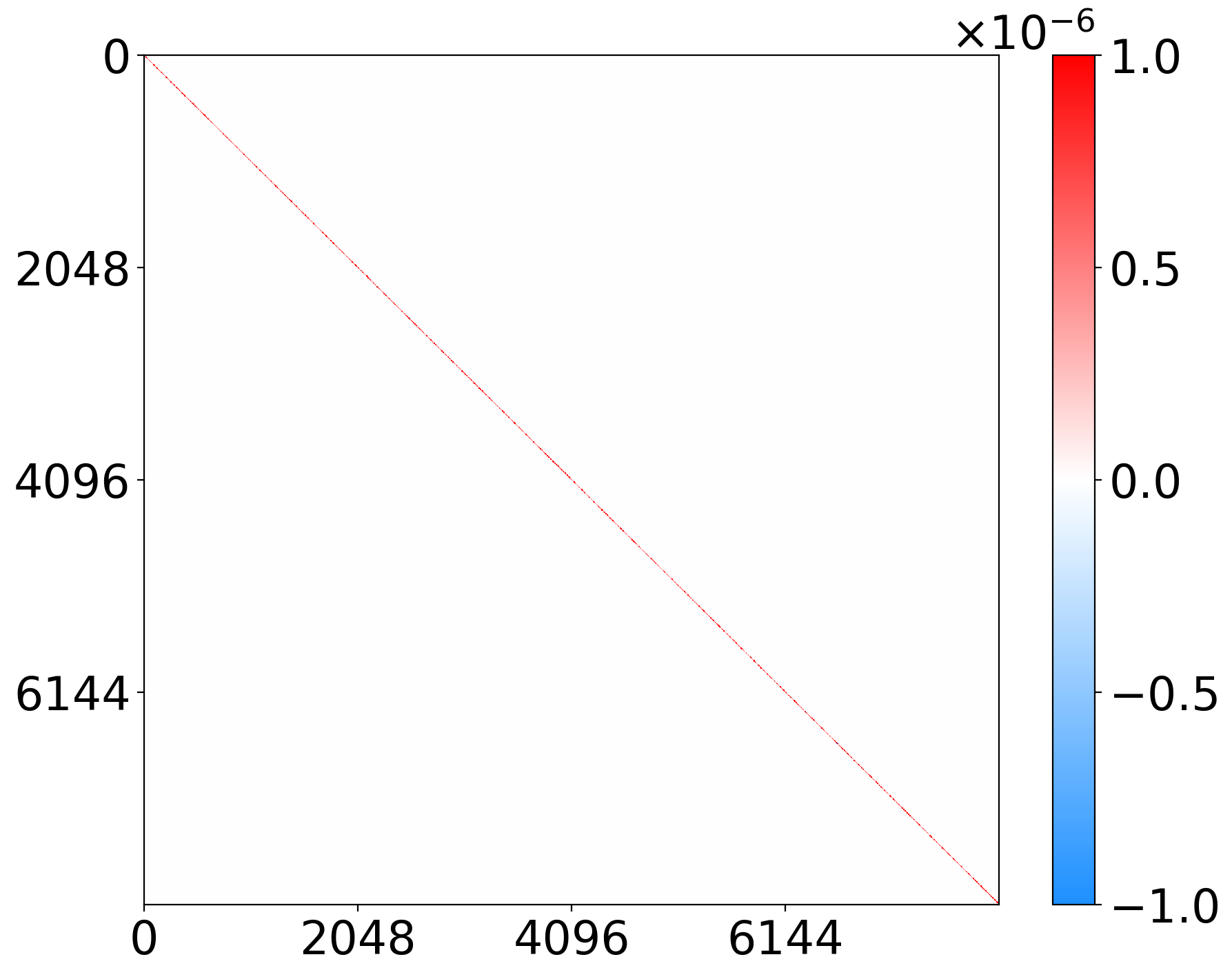}
    \end{minipage}
    \hfill
    \begin{minipage}[t]{0.3\textwidth}
        \centering
        \includegraphics[width=\textwidth]{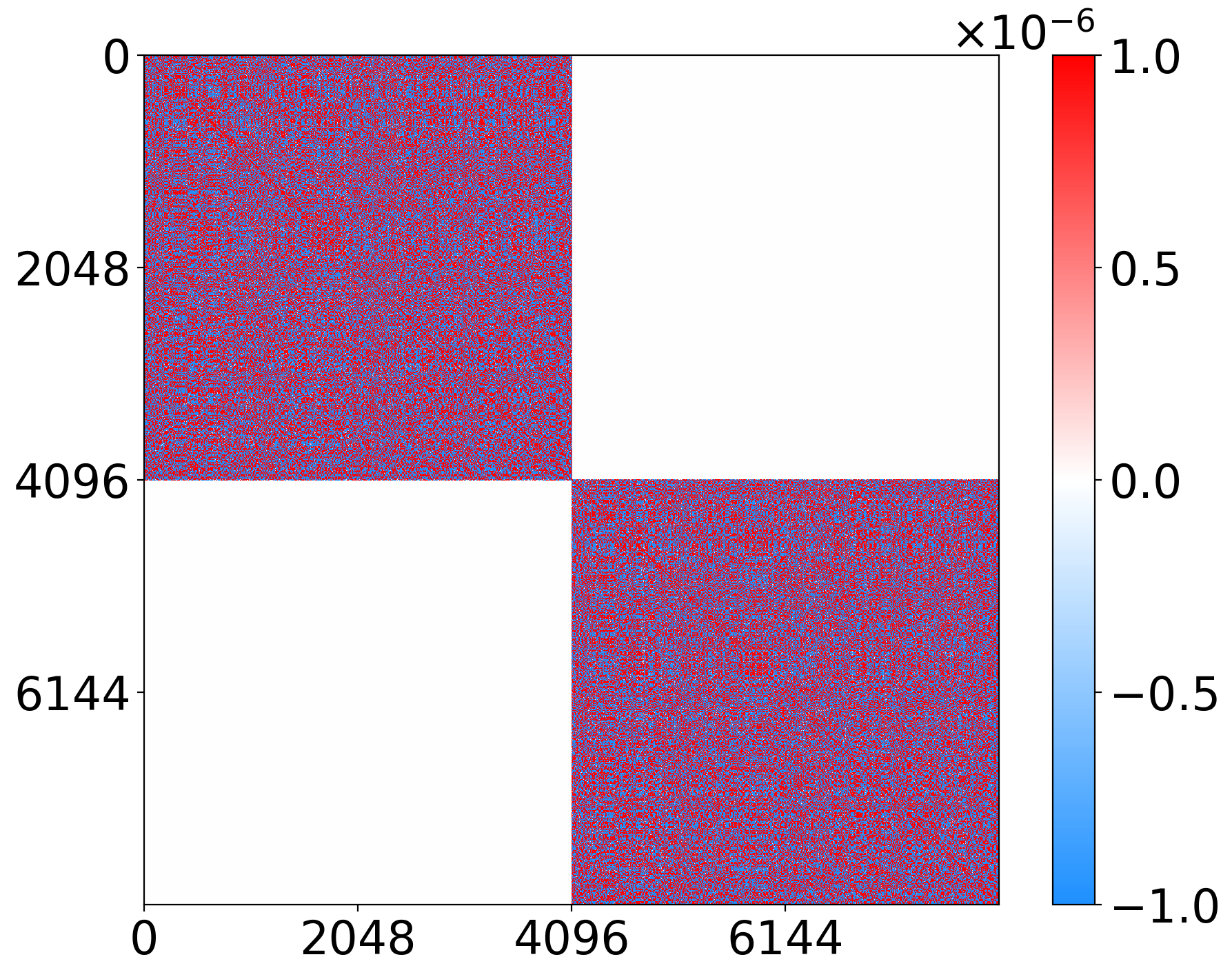}
    \end{minipage}

    \vspace{1em} 

    \makebox[0.06\textwidth][c]{\raisebox{0.5\height}{\rotatebox{90}{\shortstack{\texttt{self\_attn}\\\texttt{k\_proj}}}}} 
    \begin{minipage}[t]{0.3\textwidth}
        \centering
        \includegraphics[width=\textwidth]{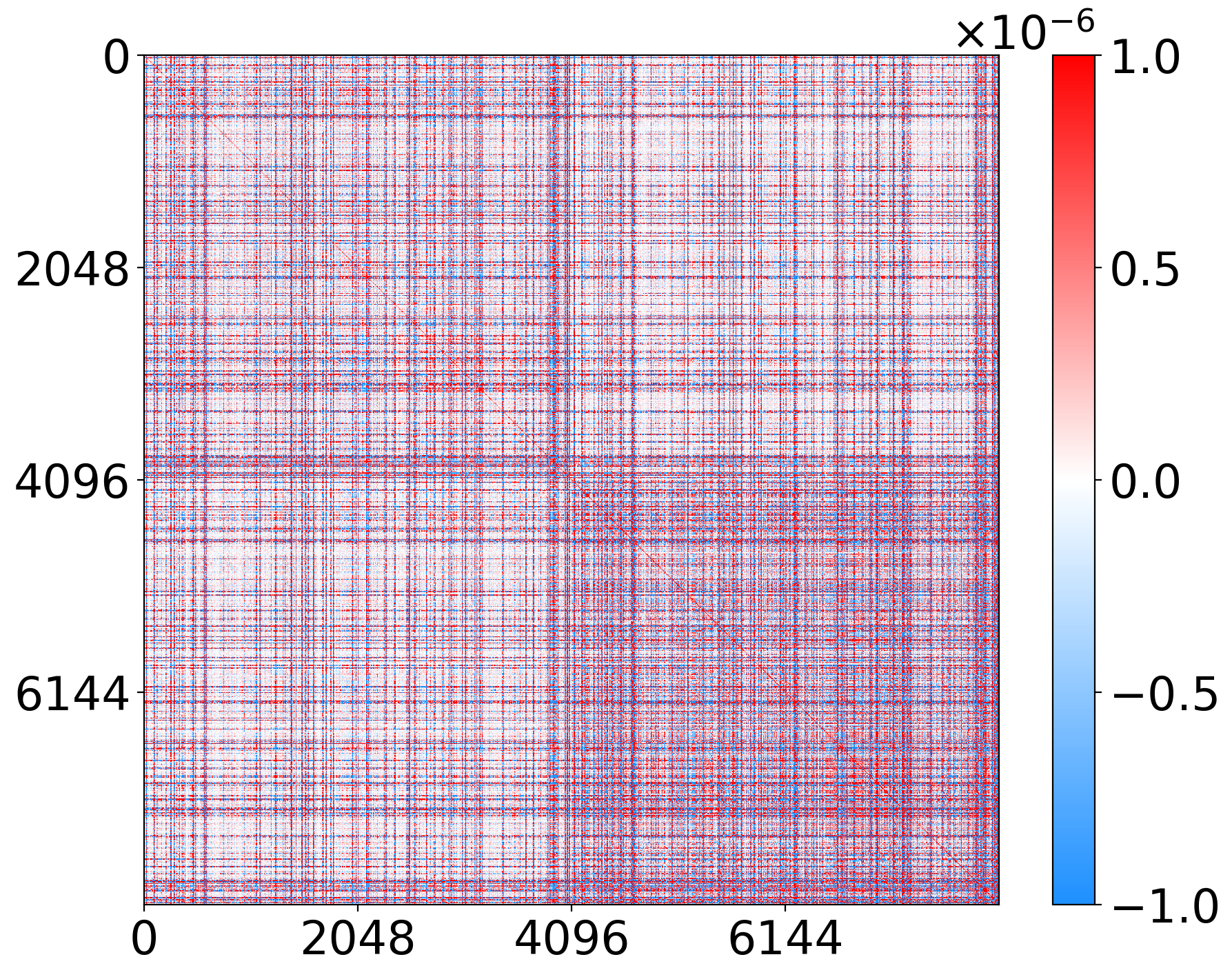}
    \end{minipage}
    \hfill
    \begin{minipage}[t]{0.3\textwidth}
        \centering
        \includegraphics[width=\textwidth]{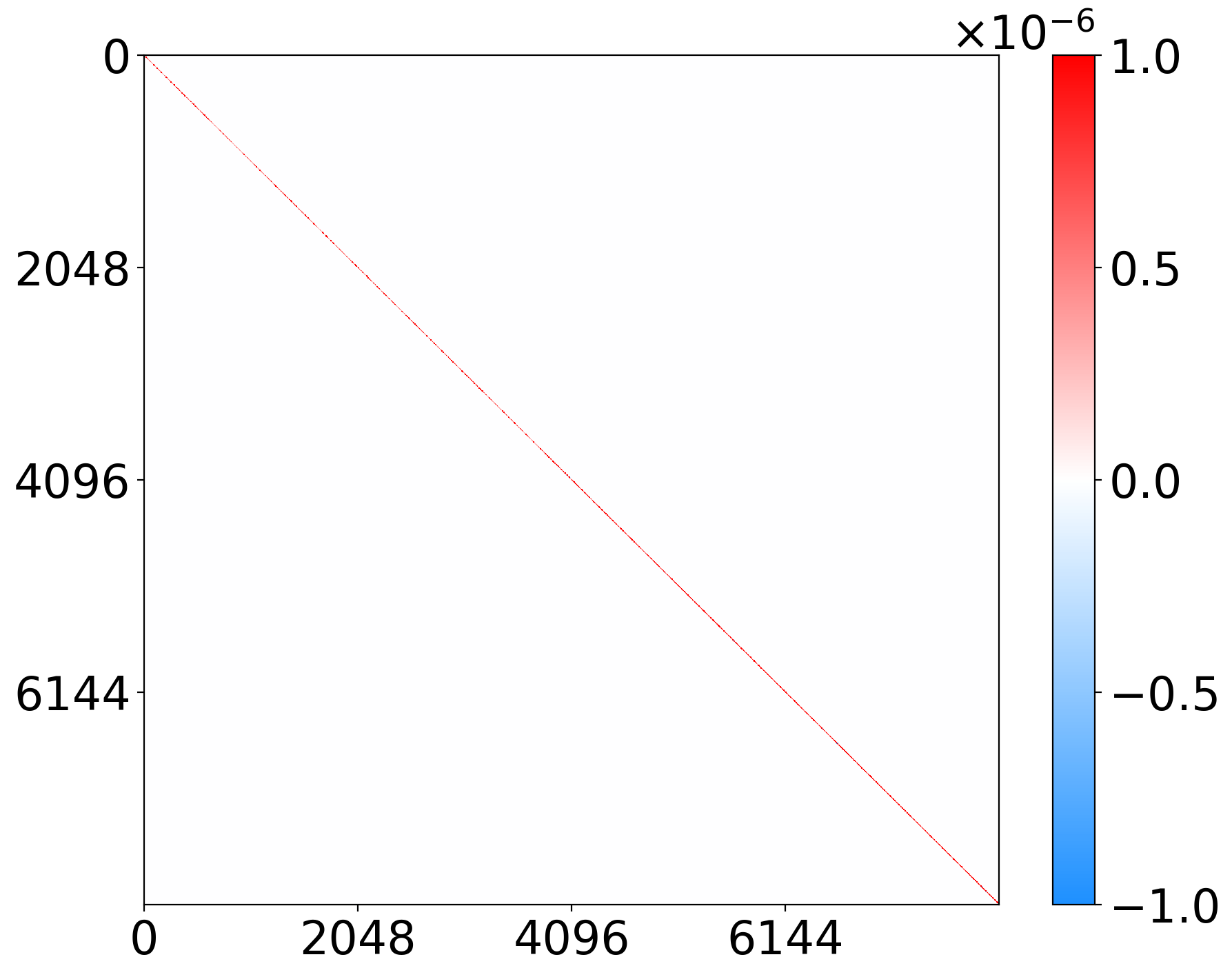}
    \end{minipage}
    \hfill
    \begin{minipage}[t]{0.3\textwidth}
        \centering
        \includegraphics[width=\textwidth]{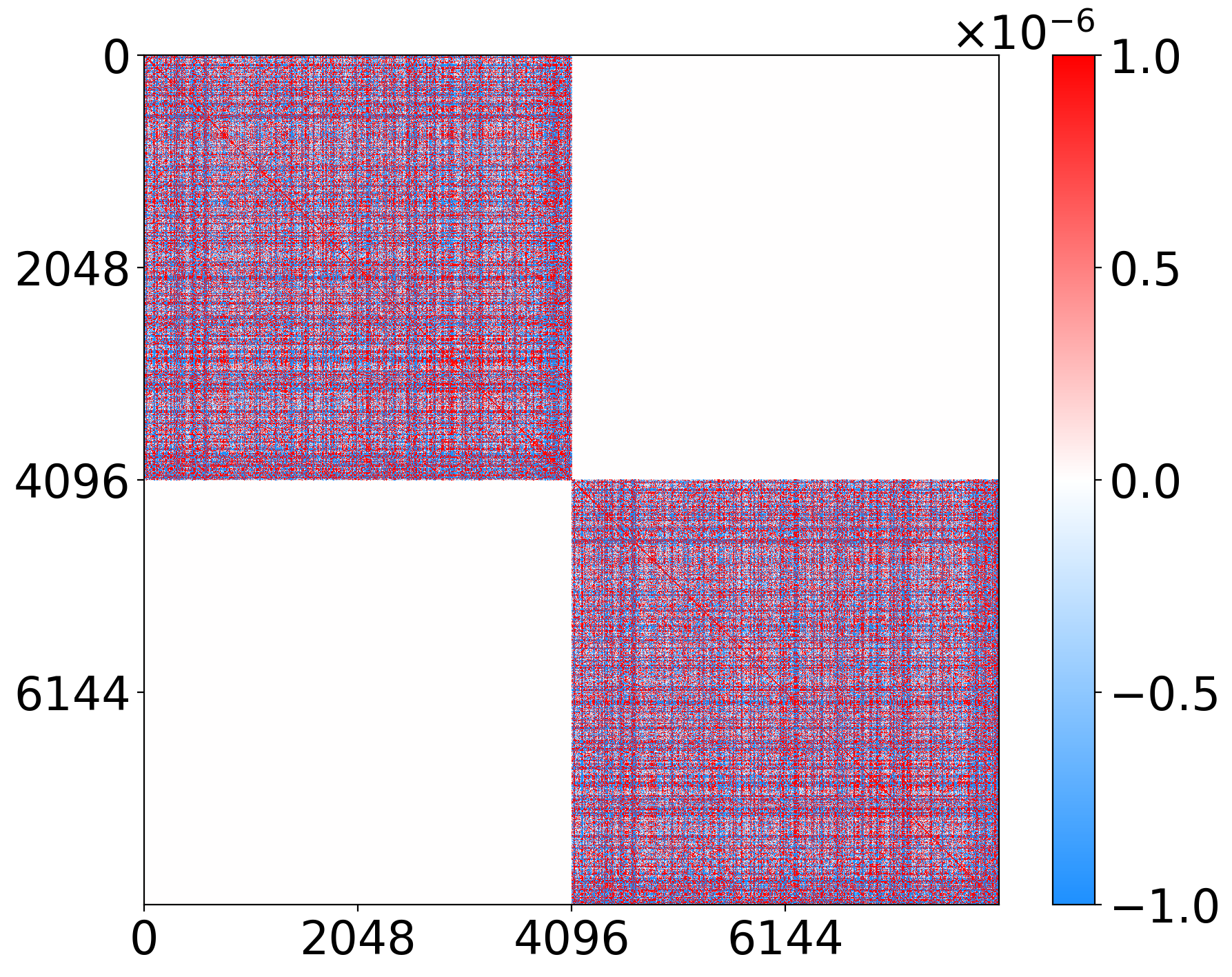}
    \end{minipage}

    \vspace{1em} 

    \makebox[0.06\textwidth][c]{\raisebox{0.5\height}{\rotatebox{90}{\shortstack{\texttt{self\_attn}\\\texttt{v\_proj}}}}} 
    \begin{minipage}[t]{0.3\textwidth}
        \centering
        \includegraphics[width=\textwidth]{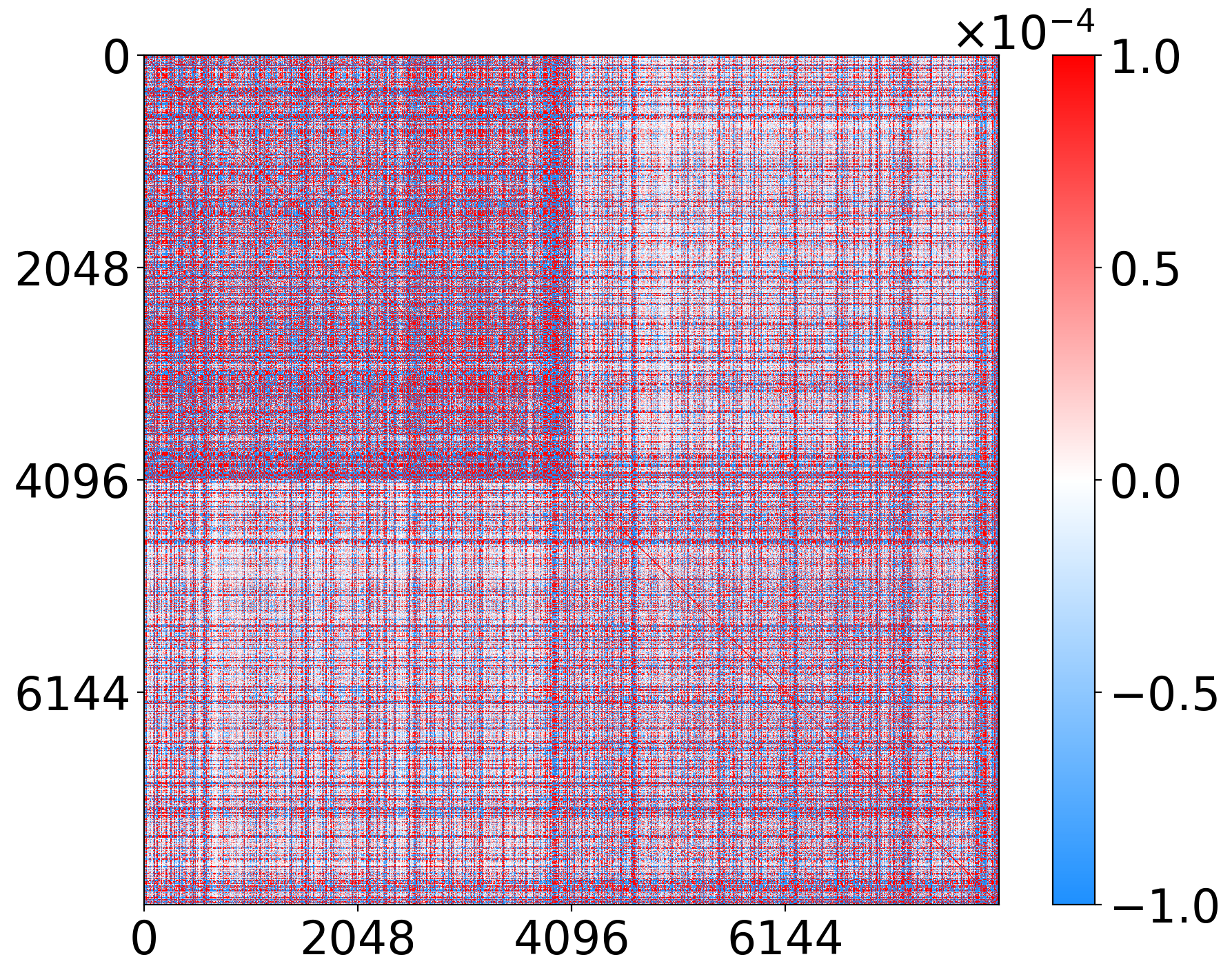}
    \end{minipage}
    \hfill
    \begin{minipage}[t]{0.3\textwidth}
        \centering
        \includegraphics[width=\textwidth]{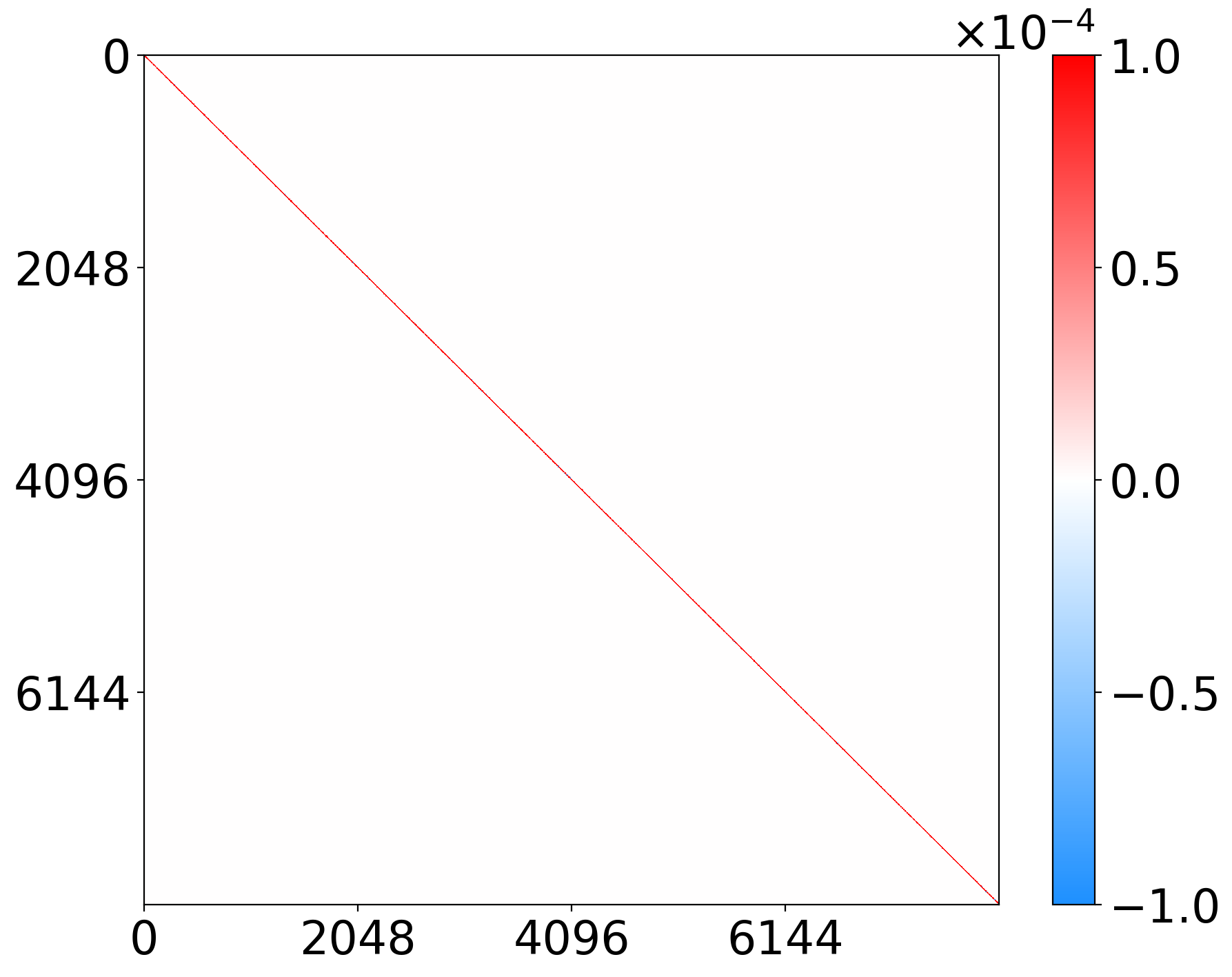}
    \end{minipage}
    \hfill
    \begin{minipage}[t]{0.3\textwidth}
        \centering
        \includegraphics[width=\textwidth]{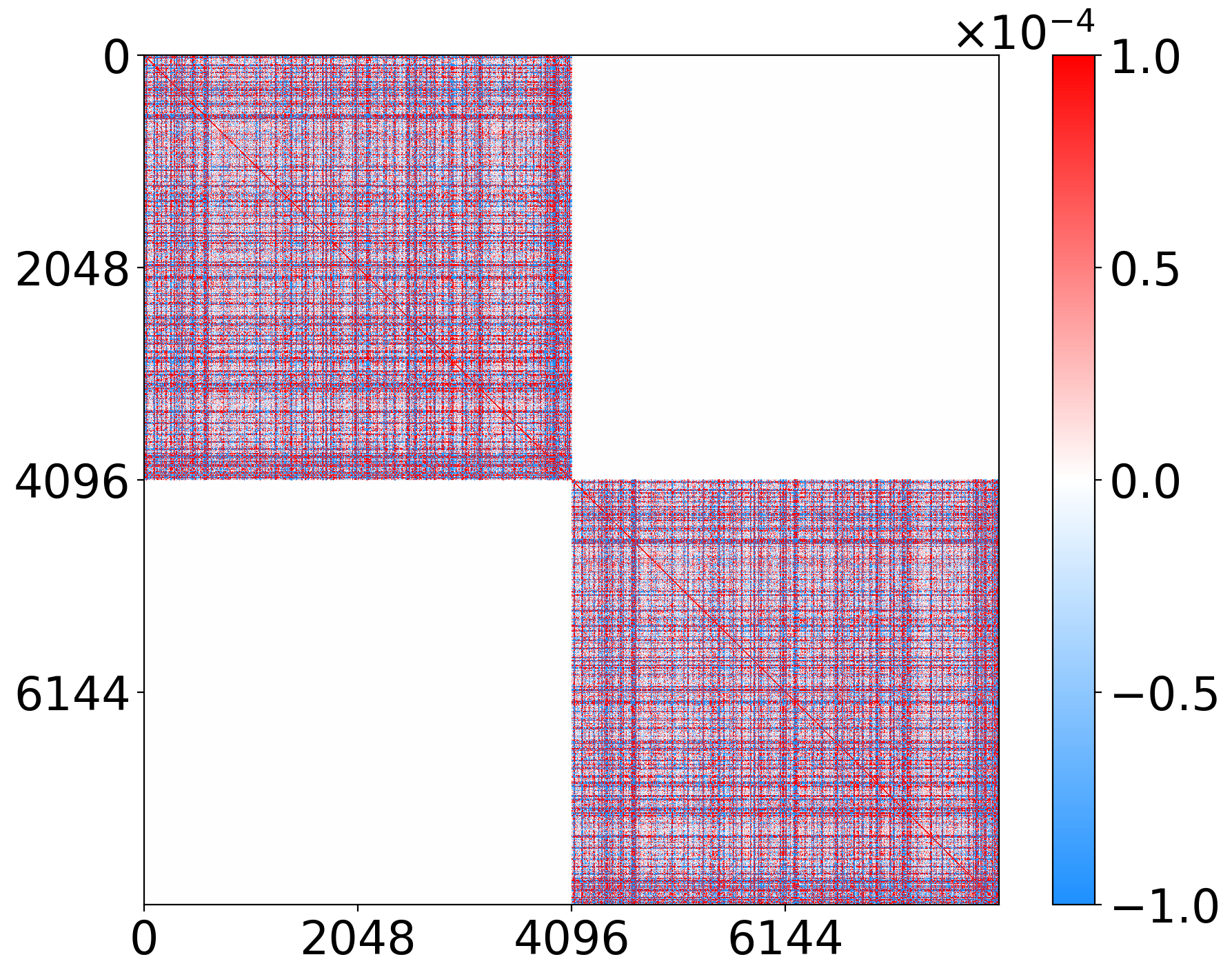}
    \end{minipage}

    \vspace{1em} 
    
    \makebox[0.06\textwidth][c]{\raisebox{0.5\height}{\rotatebox{90}{\shortstack{\texttt{self\_attn}\\\texttt{o\_proj}}}}} 
    \begin{minipage}[t]{0.3\textwidth}
        \centering
        \includegraphics[width=\textwidth]{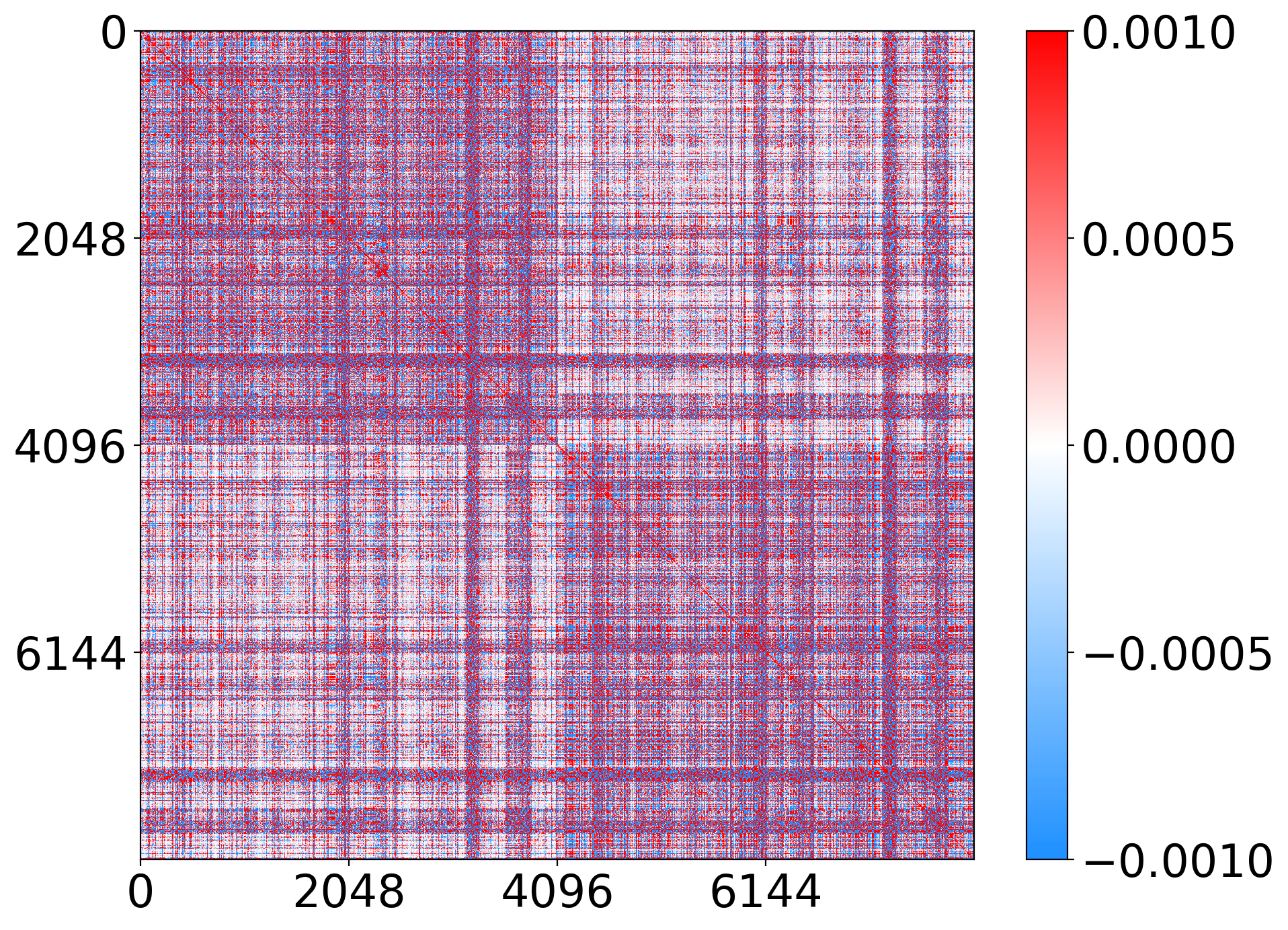}
    \end{minipage}
    \hfill
    \begin{minipage}[t]{0.3\textwidth}
        \centering
        \includegraphics[width=\textwidth]{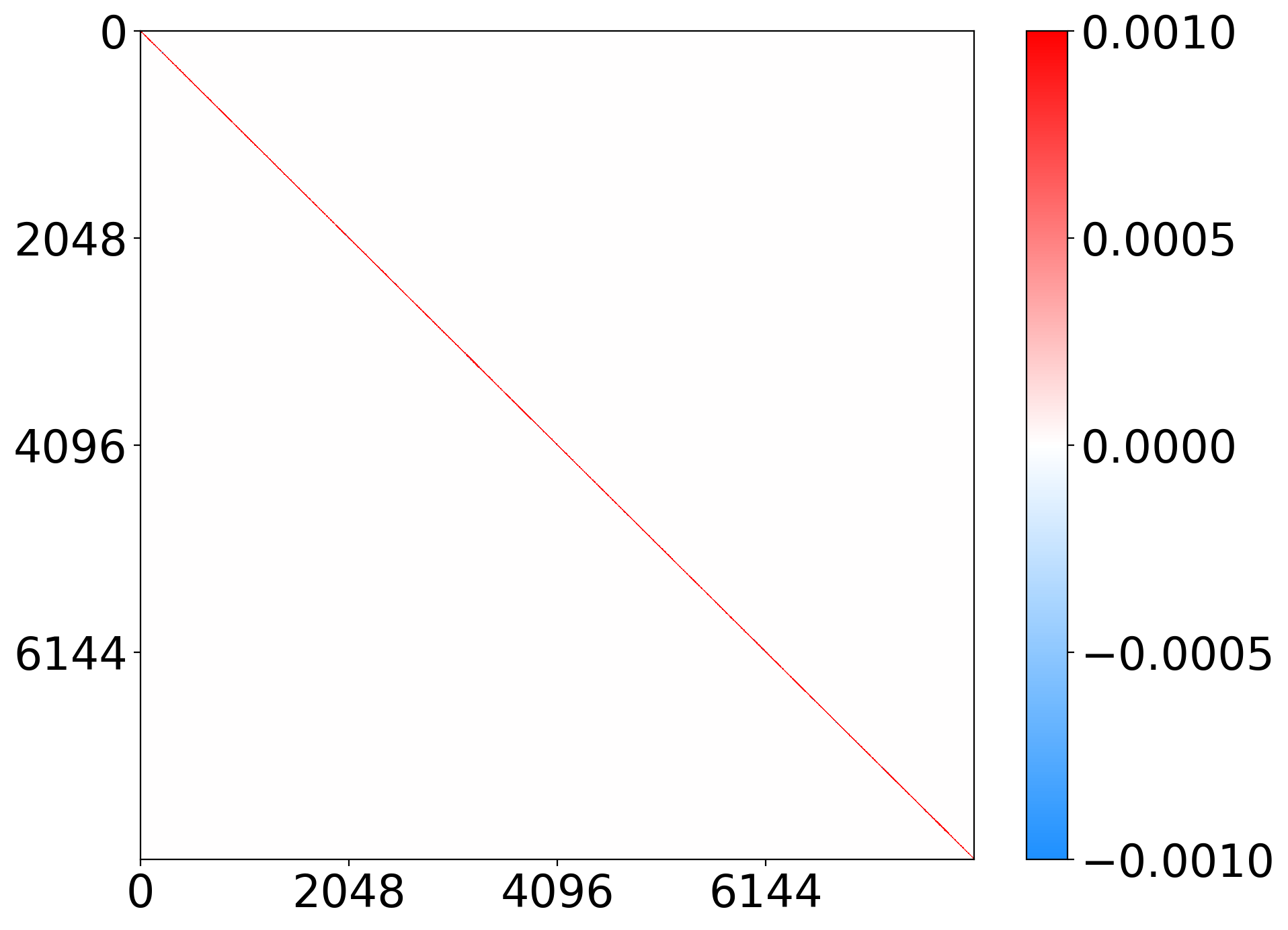}
    \end{minipage}
    \hfill
    \begin{minipage}[t]{0.3\textwidth}
        \centering
        \includegraphics[width=\textwidth]{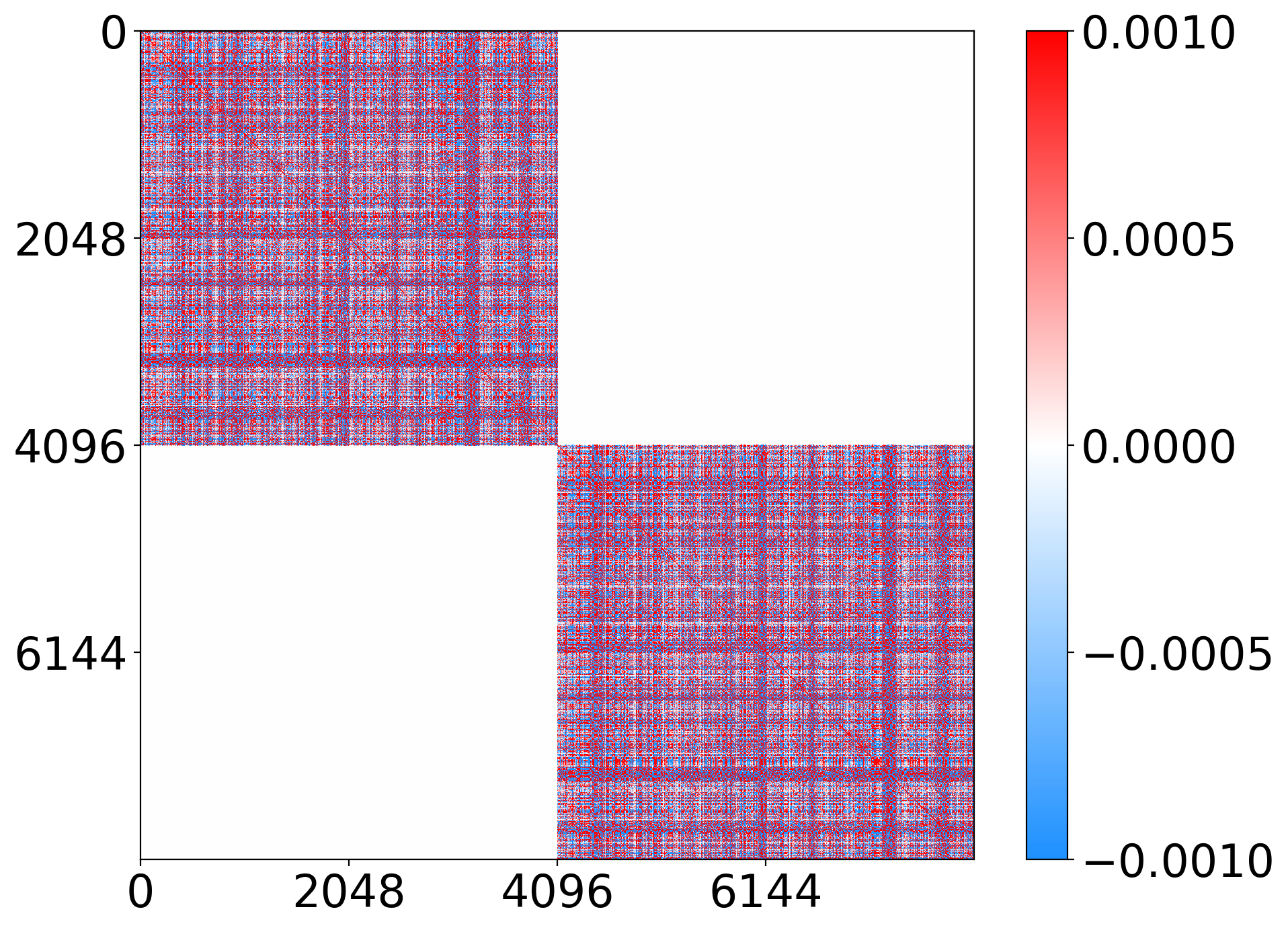}
    \end{minipage}

    \caption{Visualization of the scaled Fisher information matrix, $n\bF_j^{(l)} \times 10^6$, for the first two output channels in the \texttt{self\_attn.q\_proj}, \texttt{self\_attn.k\_proj}, \texttt{self\_attn.v\_proj}, and \texttt{self\_attn.o\_proj} layer of the first Transformer block in Llama-2-7B model. Left: the original Fisher matrices; Middle: the WoodFisher style block-diagonal approximation (block size $B=4$ for all of the layers); Right: the GuidedQuant approximation (the number of groups $g=4$). Both approximations are compared under an equal storage budget.}
    \label{fig:fisher1}
\end{figure}

\begin{figure}[t]
    \centering
    \makebox[0.06\textwidth]{} 
    \makebox[0.3\textwidth]{{Fisher (Original)}} \hfill
    \makebox[0.3\textwidth]{{WoodFisher}} \hfill
    \makebox[0.3\textwidth]{{GuidedQuant (Ours)}} 


    \makebox[0.06\textwidth][c]{\raisebox{0.5\height}{\rotatebox{90}{\shortstack{\texttt{mlp}\\\texttt{gate\_proj}}}}}
    \hfill
    \begin{minipage}[t]{0.3\textwidth}
        \centering
        \includegraphics[width=\textwidth]{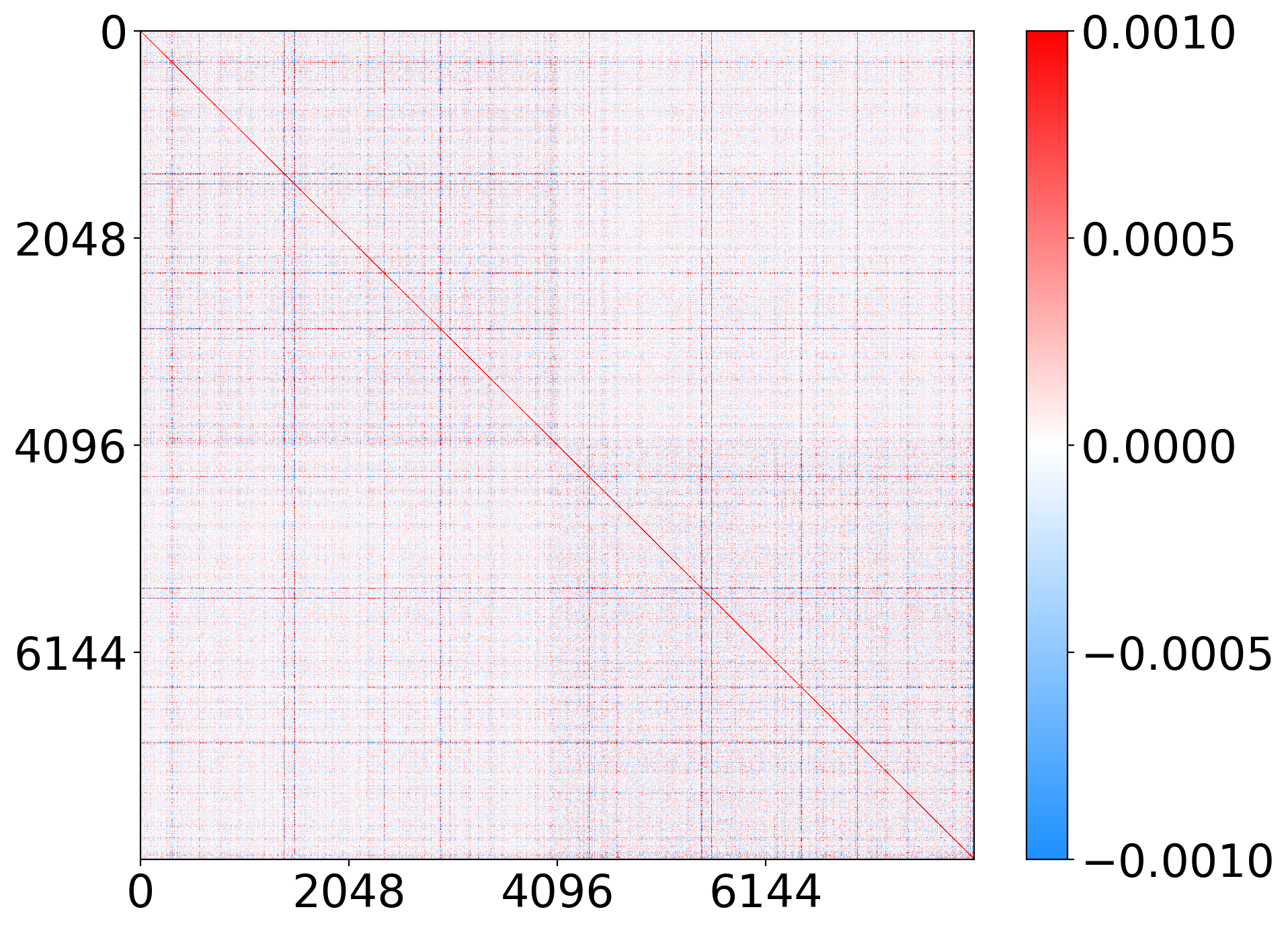}
    \end{minipage}
    \hfill
    \begin{minipage}[t]{0.3\textwidth}
        \centering
        \includegraphics[width=\textwidth]{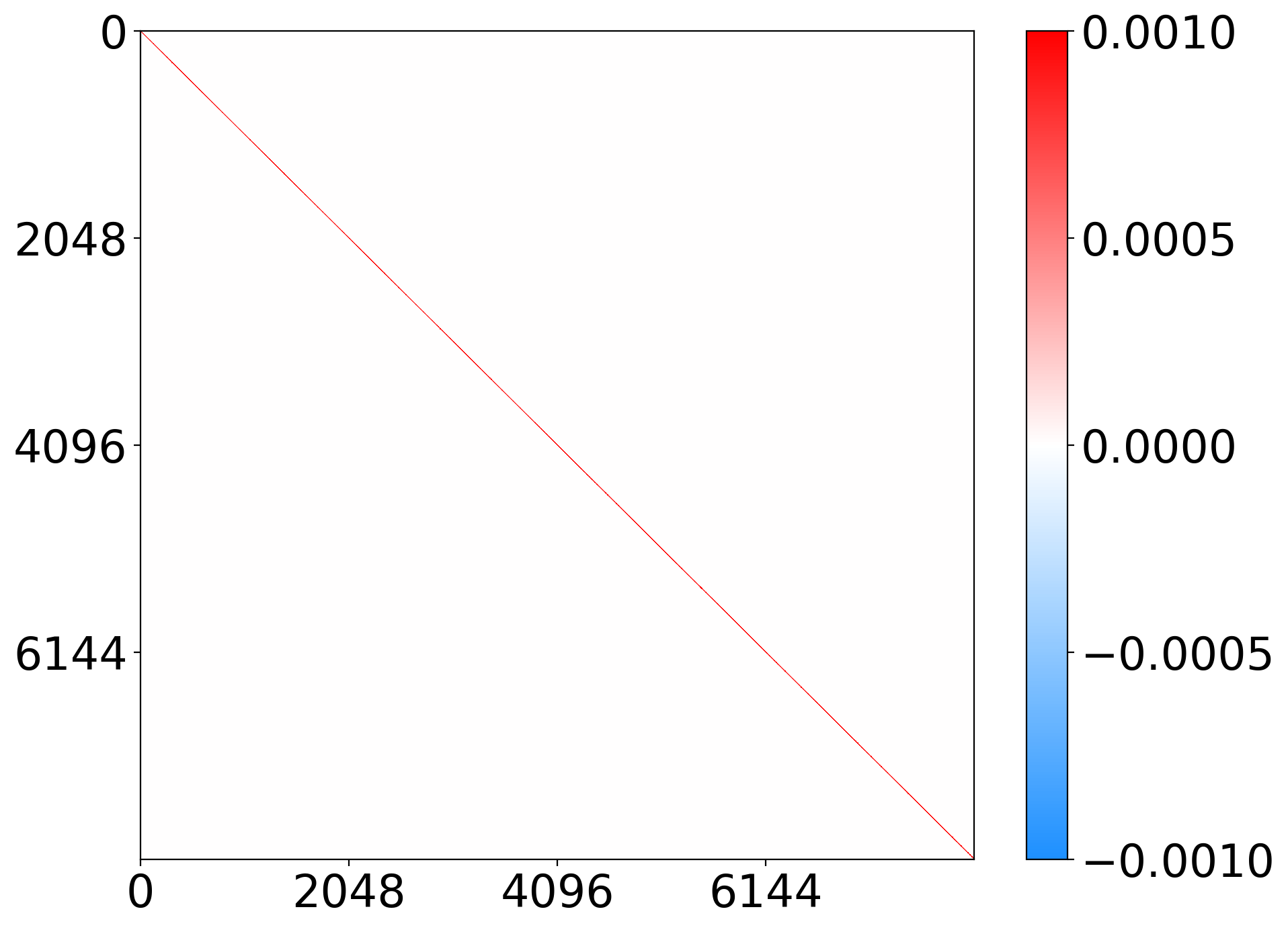}
    \end{minipage}
    \hfill
    \begin{minipage}[t]{0.3\textwidth}
        \centering
        \includegraphics[width=\textwidth]{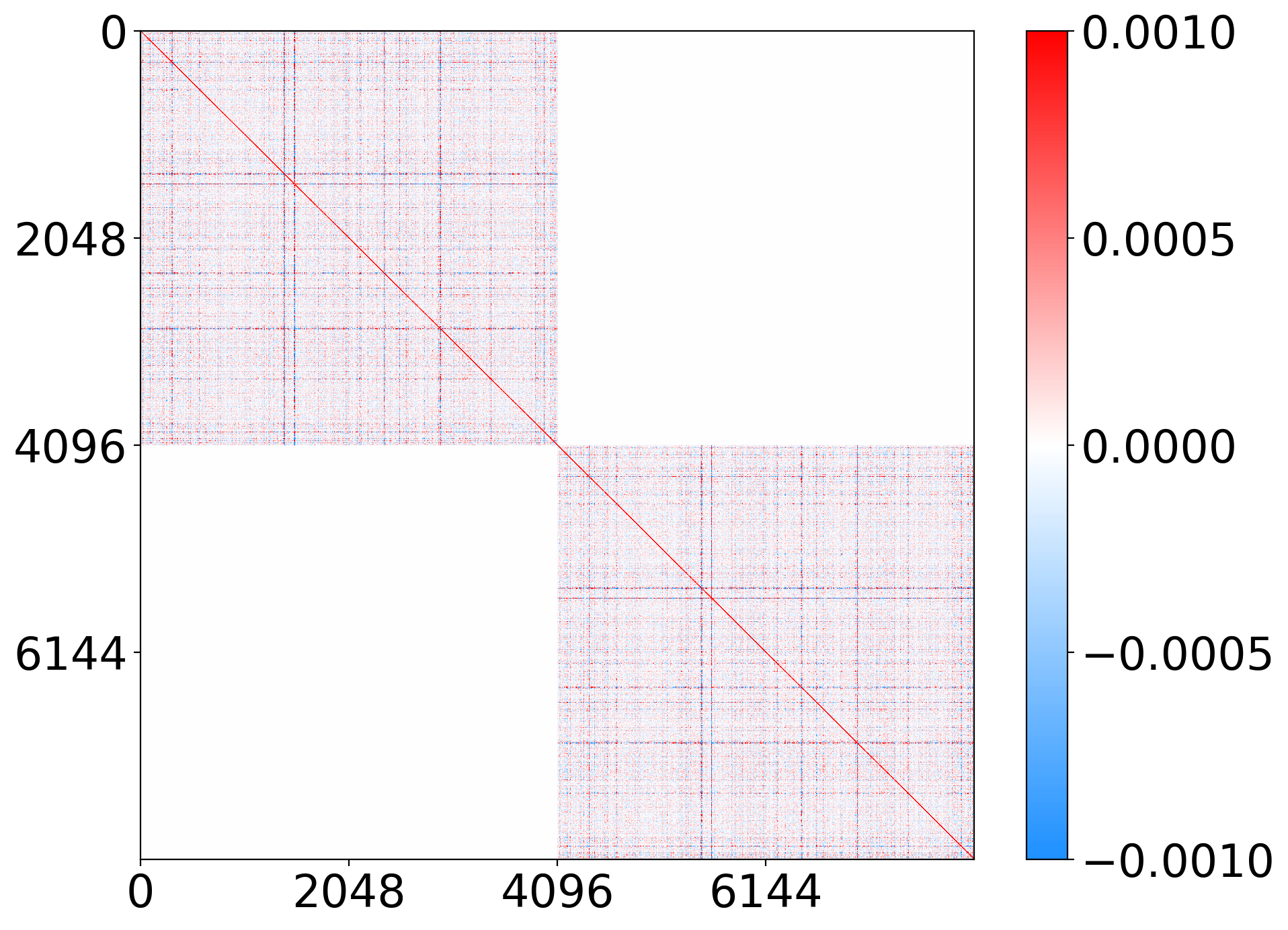}
    \end{minipage}

    \vspace{1em} 

    \makebox[0.06\textwidth][c]{\raisebox{0.5\height}{\rotatebox{90}{\shortstack{\texttt{mlp}\\\texttt{up\_proj}}}}} 
    \hfill
    \begin{minipage}[t]{0.3\textwidth}
        \centering
        \includegraphics[width=\textwidth]{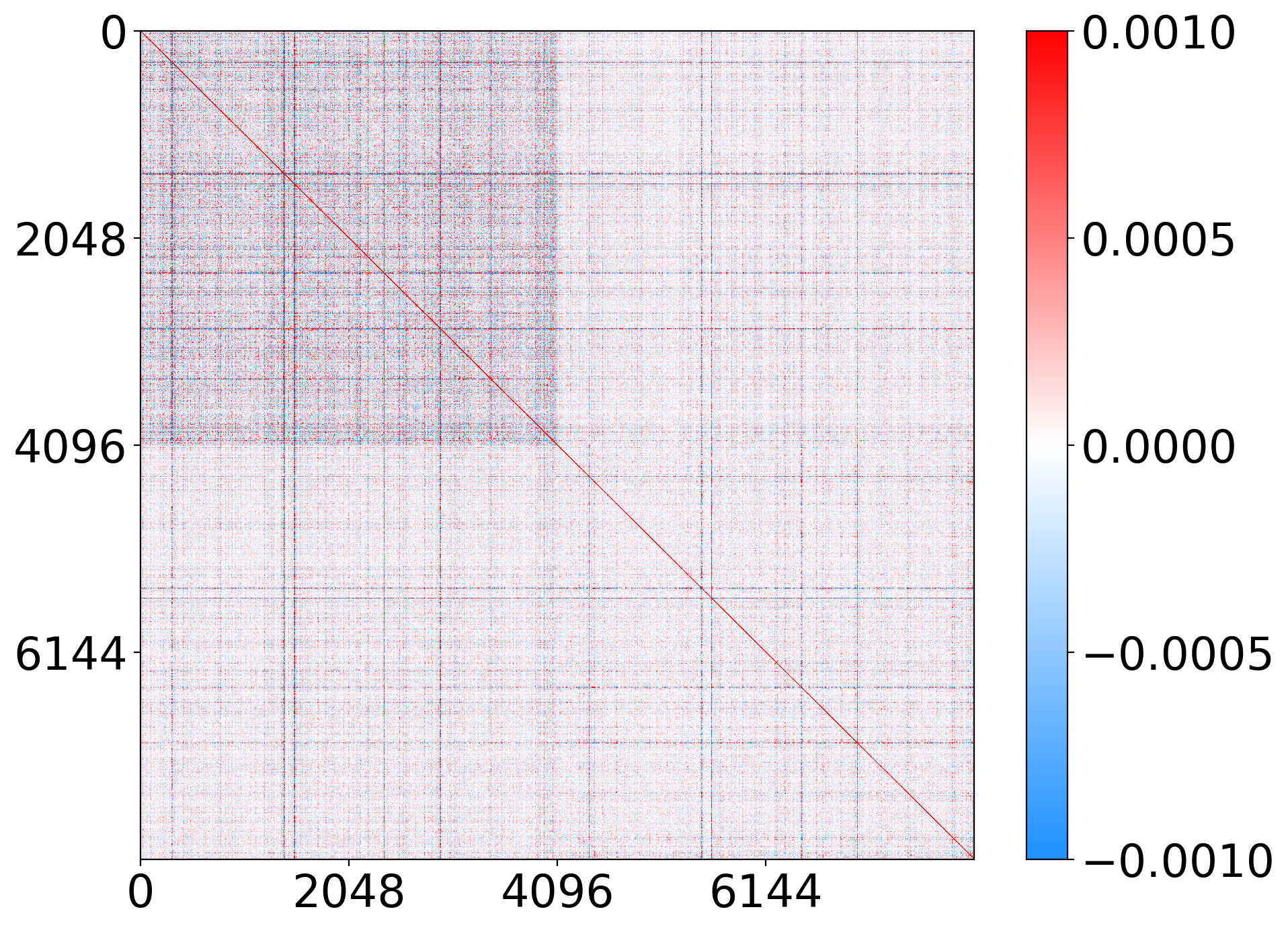}
    \end{minipage}
    \hfill
    \begin{minipage}[t]{0.3\textwidth}
        \centering
        \includegraphics[width=\textwidth]{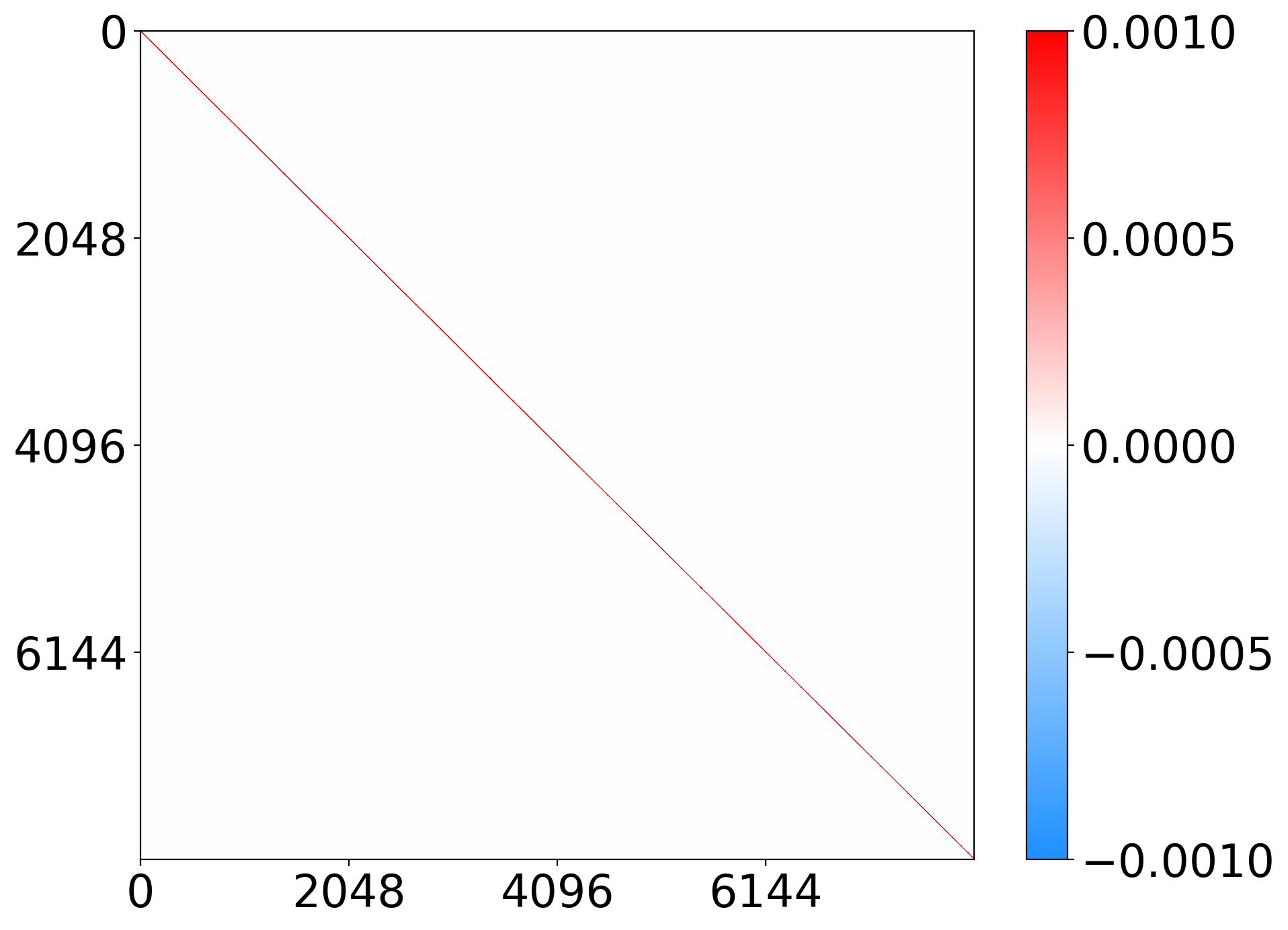}
    \end{minipage}
    \hfill
    \begin{minipage}[t]{0.3\textwidth}
        \centering
        \includegraphics[width=\textwidth]{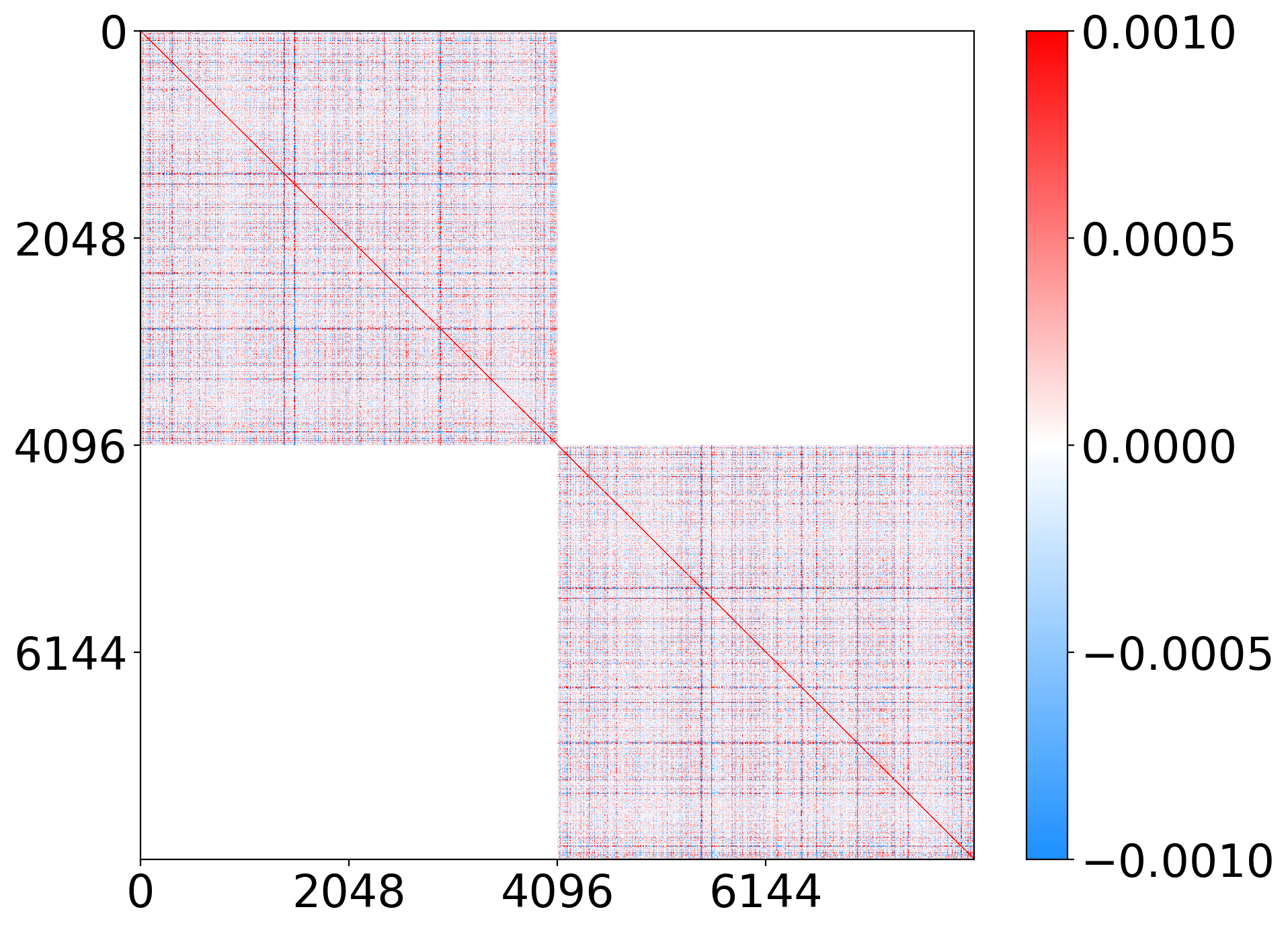}
    \end{minipage}

    \vspace{1em} 

    \makebox[0.06\textwidth][c]{\raisebox{0.5\height}{\rotatebox{90}{\shortstack{\texttt{mlp}\\\texttt{down\_proj}}}}}  
    \hfill
    \begin{minipage}[t]{0.3\textwidth}
        \centering
        \includegraphics[width=\textwidth]{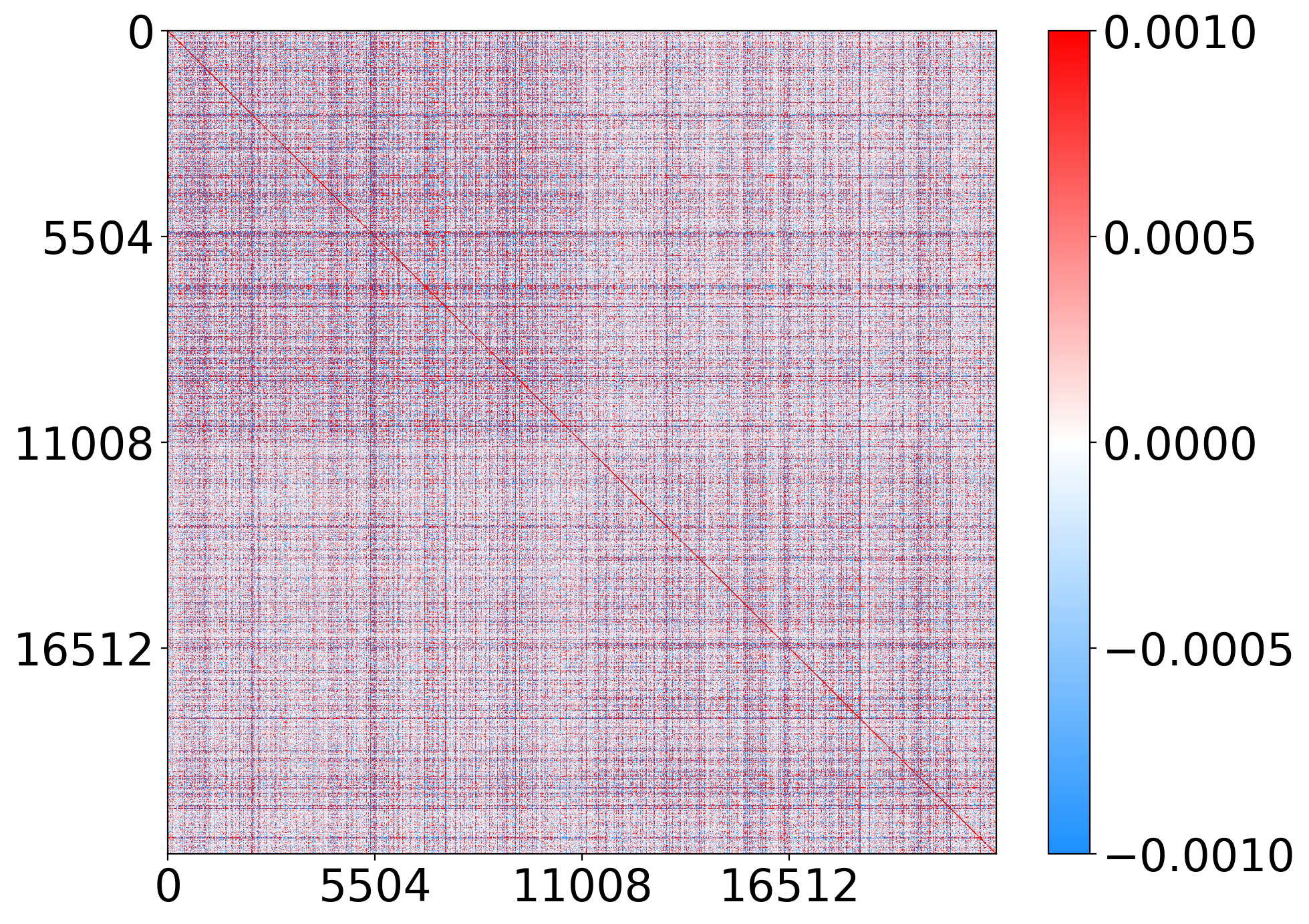}
    \end{minipage}
    \hfill
    \begin{minipage}[t]{0.3\textwidth}
        \centering
        \includegraphics[width=\textwidth]{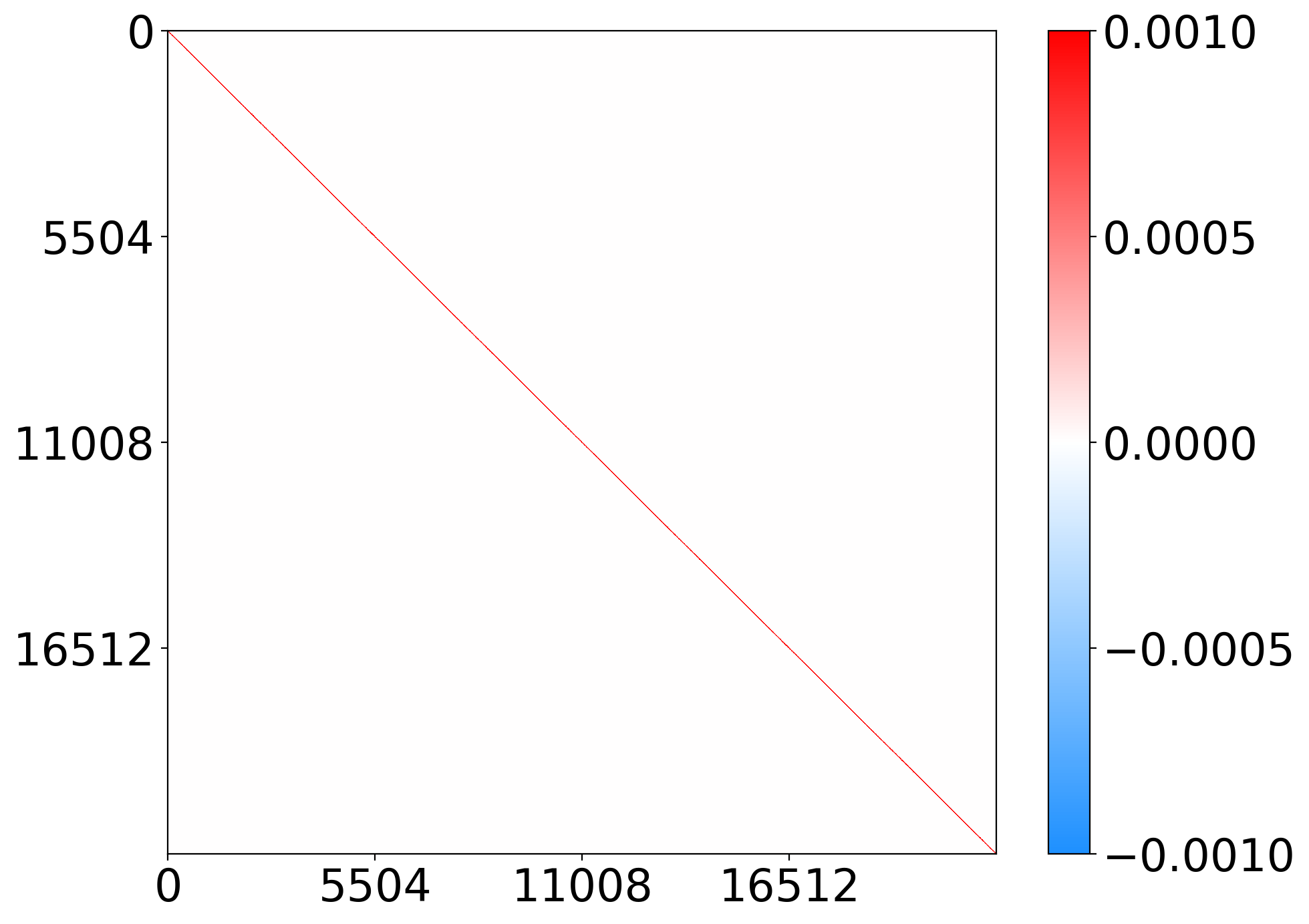}
    \end{minipage}
    \hfill
    \begin{minipage}[t]{0.3\textwidth}
        \centering
        \includegraphics[width=\textwidth]{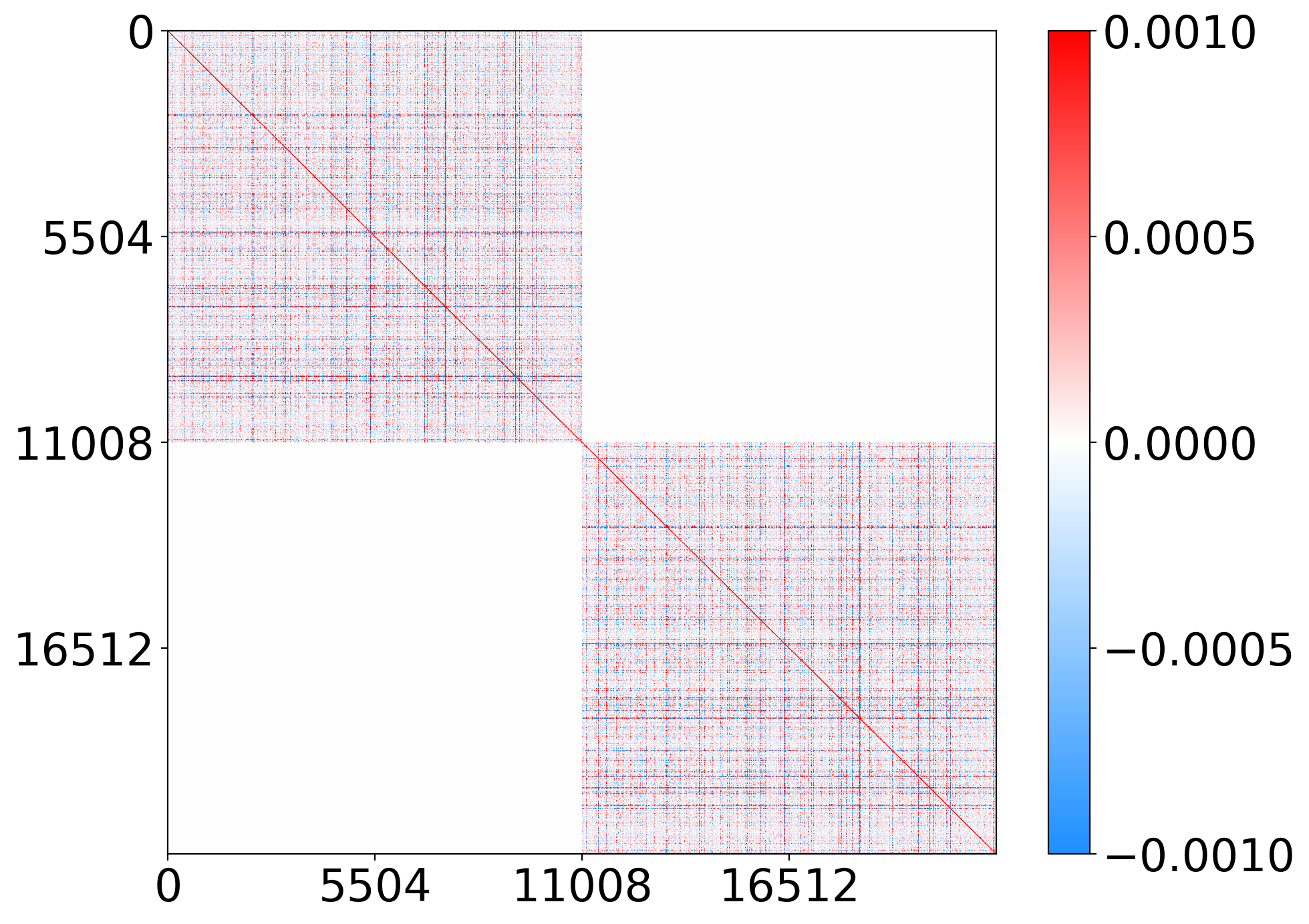}
    \end{minipage}

    \caption{Visualization of the scaled Fisher information matrix, $n\bF_j^{(l)} \times 10^6$, for the first two output channels in the \texttt{mlp.gate\_proj}, \texttt{mlp.up\_proj}, and \texttt{mlp.down\_proj} layer of the first Transformer block in Llama-2-7B model. Left: the original Fisher matrices; Middle: the WoodFisher style block-diagonal approximation (block size $B=2$, $B=2$, and $B=11$, respectively); Right: the GuidedQuant approximation (the number of groups $g=4$). Both approximations are compared under an equal storage budget.}
    \label{fig:fisher2}
\end{figure}


In this section, we review existing neural network compression methods that use a block-diagonal Fisher matrix approximation of the Hessian and highlight their differences from GuidedQuant.
In particular, we discuss WoodFisher \citep{singh2020woodfisher} for pruning CNNs, Optimal BERT Surgeon \citep{kurtic2022optimal} for pruning BERT models, and BRECQ \citep{li2021brecq} for quantizing CNNs.


WoodFisher and Optimal BERT Surgeon use blocks of arbitrary size $B \times B$ along the diagonal to reduce the storage cost. 
WoodFisher explores $B$ size of $\{20, 100, 1000, 5000, 12288, 37000\}$ in ResNet-20 \citep{he2015deep}, while Optimal BERT Surgeon uses $B=50$, since the larger block size does not fit in the memory.
BRECQ leaves the blocks that correspond to the parameters within each residual block in CNNs, and further uses a first-order Taylor approximation on the residual block's outputs to estimate the second-order error for each block to avoid the need to handle prohibitively large matrices.

The proposed GuidedQuant maintains the blocks corresponding to each output channel, resulting the $B$ size to be $4096$ to $11008$ for Llama-2-7B model. Directly computing these block-diagonal matrices would be infeasible, requiring over $110$ TB for and more than $13,000$ GPU hours on RTX 6000 Ada GPU for Llama-2-7B. To address this, GuidedQuant averages the Fisher diagonal blocks within each group, approximately preserving dependencies within each output channel at the scale of modern LLMs. We present the theoretical complexity of GuidedQuant in \cref{sec:challenge}, report its practical cost in \cref{tab:cost_hess}, and report the performance of approximating more (opting for smaller number of groups) in \cref{sec:group}.

In \cref{fig:fisher1,fig:fisher2}, we illustrate submatrices of the scaled Fisher information matrix, $n\bF_j^{(l)} \times 10^6$, for the linear layers in the first Transformer block of the Llama-2-7B model, alongside corresponding approximation results. 
Here, $n$ denotes the number of calibration data, and the results are computed using calibration data from the RedPajama dataset, which consists of 1024 sentences with 4096 tokens each.
Since each linear layer in the model contains $d_{\mathrm{in}} \times d_{\mathrm{out}}$ weights, fully visualizing its Fisher information matrix would yield a matrix of size $d_{\mathrm{in}} d_{\mathrm{out}} \times d_{\mathrm{in}} d_{\mathrm{out}}$, which is computationally prohibitive. Therefore, we restrict our visualization to the submatrix corresponding to the first two output channels of each layer. Since each output channel has $d_{\mathrm{in}}$ weights, this results in visualizing a $2 d_{\mathrm{in}} \times 2 d_{\mathrm{in}}$ matrix.
Within the Transformer block of the Llama-2-7B model, there are seven linear layers: \texttt{self\_attn.q\_proj}, \texttt{self\_attn.k\_proj}, \texttt{self\_attn.v\_proj}, \texttt{self\_attn.o\_proj}, \texttt{mlp.gate\_proj}, \texttt{mlp.up\_proj}, and \texttt{mlp.down\_proj}.
For the first six layers, $d_{\mathrm{in}} = 4096$, so we visualize an $8192 \times 8192$ matrix, while for the final layer (\texttt{mlp.down\_proj}) with $d_{\mathrm{in}} = 11008$, an $22016 \times 22016$ matrix is visualized.

We compare two approximation strategies:
\vspace{-1em}
\begin{itemize} \itemsep=0pt
    \item WoodFisher: This approach retains the blocks size of $B \times B$ along the diagonal. The storage requirement for this method is $B \, d_{\mathrm{in}} \, d_{\mathrm{out}}$.
    \item GuidedQuant: Here, the block size is set to $d_{\mathrm{in}} \times d_{\mathrm{in}}$ and blocks are averaged within groups. This strategy requires $g \, d_{\mathrm{in}}^2$ storage, where $g$ is the number of groups.
\end{itemize}
To ensure a fair comparison, we choose the WoodFisher block size as $B = \lceil g\, d_{\mathrm{out}} / d_{\mathrm{in}} \rceil$. Specifically, we choose $g = 4$ for the GuidedQuant, which results in $B = 4$ for the self-attention projection layers, $B = 2$ for the \texttt{mlp.gate\_proj} and \texttt{mlp.up\_proj} layers, and $B = 11$ for the \texttt{mlp.down\_proj} layer.

The visualizations reveal that the original Fisher information matrix exhibits strong off-diagonal values and a prominent block-diagonal structure with blocks of size $d_{\mathrm{in}} \times d_{\mathrm{in}}$. This indicates stronger interactions among weights within the same output channel compared to those across different channels. Overall, the GuidedQuant approximation captures significantly more of this structural detail than the WoodFisher-style block-diagonal approximation, which retains only arbitrarily sized diagonal blocks.


\end{document}